\documentclass[accepted]{uai2026}              

\usepackage[american]{babel}
\usepackage{graphicx}
\usepackage{subcaption}

\usepackage{natbib} 
\usepackage{mathtools} 
\usepackage{booktabs} 
\usepackage{tikz} 
\usepackage{enumitem,amsmath,amsthm,amssymb}
\usepackage{algorithm,algorithmic}

\usepackage{booktabs}
\usepackage{multirow}

\newtheorem{theorem}{Theorem}
\newtheorem{lemma}{Lemma}
\newtheorem{proposition}{Proposition}
\newtheorem{assumption}{Assumption}
\newtheorem{remark}{Remark}
\newtheorem{corollary}{Corollary}

\title{Mixture-Greedy for Online Generative Model Selection:\\ Is UCB Necessary in Diversity-Aware Multi-Armed Bandits?}

\author[1]{Bahar Dibaei Nia}
\author[1]{Farzan Farnia}

\affil[1]{%
Department of Computer Science and Engineering, The Chinese University of Hong Kong, Hong Kong SAR
}

\begin{document}
\maketitle

\begin{abstract}
Efficient selection among multiple generative models is increasingly important in modern generative AI, where sampling from suboptimal models is costly. This problem can be viewed as a \emph{multi-armed bandit (MAB)} task. Under diversity-aware evaluation scores, a non-degenerate mixture of generators can outperform any individual model, distinguishing this MAB setting from classical best-arm identification. Prior approaches incorporate an Upper Confidence Bound (UCB) exploration bonus into the mixture objective. However, across multiple datasets and evaluation metrics, we observe that the UCB term consistently slows convergence and reduces sample efficiency. In contrast, a simple \emph{Mixture-Greedy} strategy without explicit UCB-type optimism converges faster and achieves even better performance, particularly for widely used metrics such as FID and Vendi where tight confidence bounds are difficult to construct. We provide theoretical insight explaining this behavior: under structural conditions, diversity-aware objectives induce \emph{implicit exploration} by favoring interior mixtures, leading to sampling of all arms and sublinear regret guarantees for diversity-based objectives. These results suggest that in diversity-aware multi-armed bandits, e.g., for generative model selection, exploration can arise intrinsically from the objective's geometry. 
The project code is available at \url{https://github.com/bhrdbn/Mixture-Greedy}.
\end{abstract}

\section{Introduction}\label{sec:intro}
 \begin{figure*}[t]
        \centering
        \includegraphics[width=\linewidth]{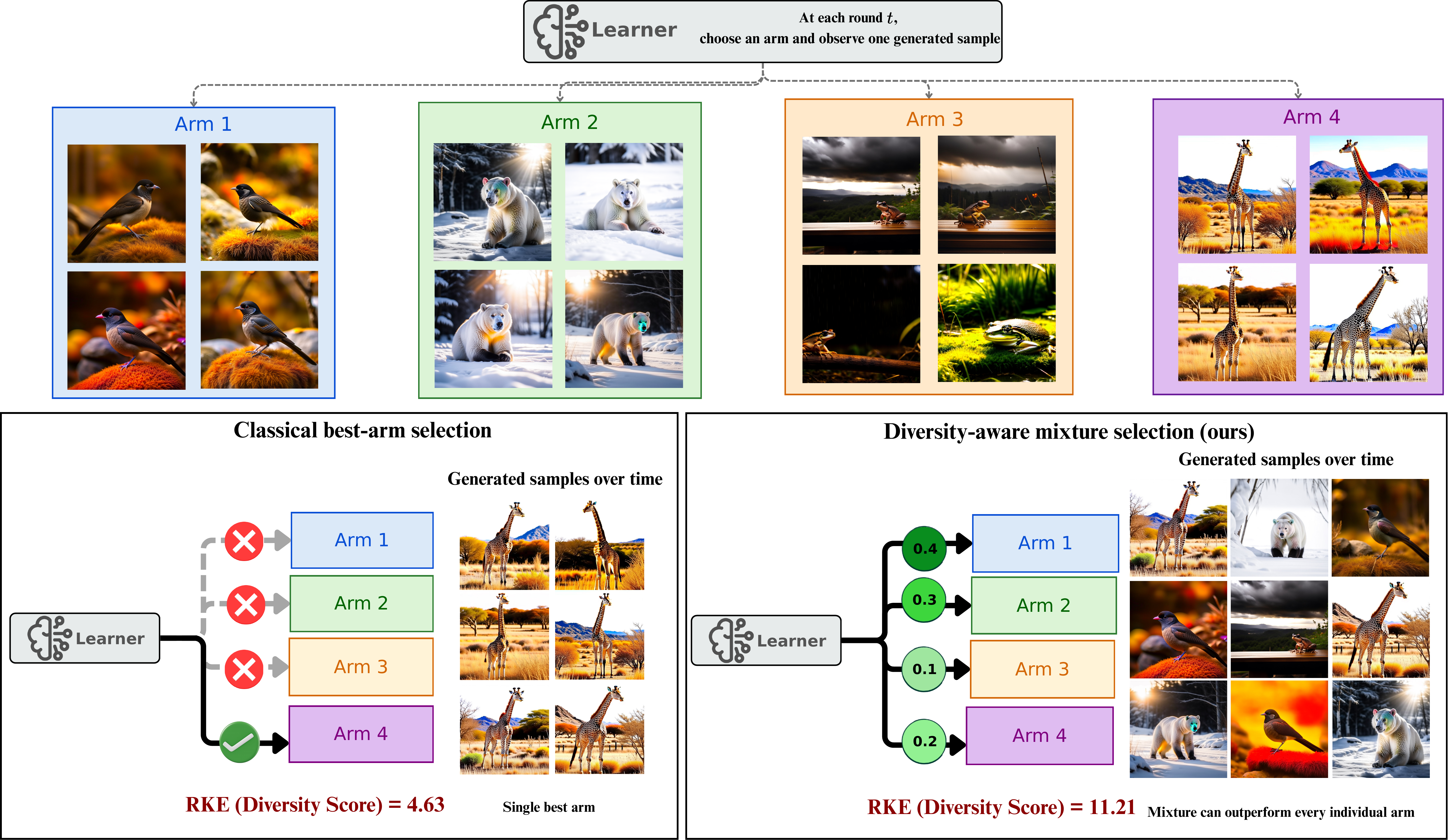}
        \label{overview}
    \caption{
Classical multi-armed bandits seek the best single arm with the highest expected reward. However, for diversity-aware objectives, the score can be strictly higher for a non-degenerate mixture of the arms. For example, in the generative model bandit setting,  the optimal mixture solution is usually an interior point of the probability simplex, assigning non-zero probability to multiple generators (arms). We demonstrate that this geometric property can naturally encourage exploration and study whether explicit UCB exploration is necessary in diversity-aware multi-armed bandit settings.}
     \label{fig:vendidogs}
\end{figure*}

\begin{figure*}[t]
        \centering
        \includegraphics[width=\linewidth]{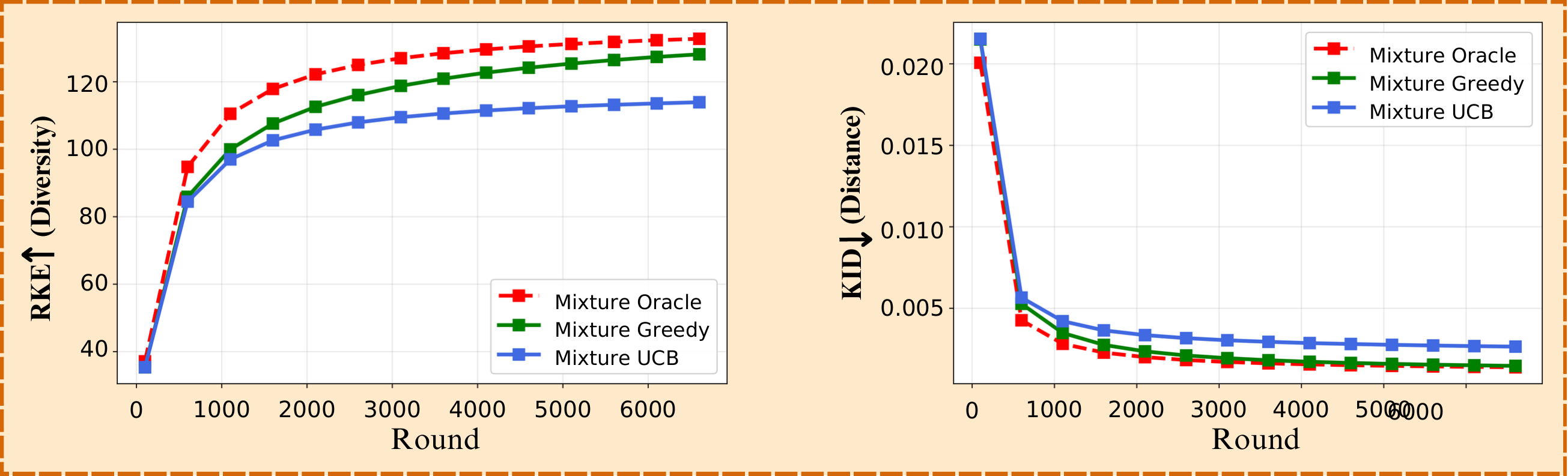}
        \label{FID}
    \caption{Comparison of mixture selection strategies over five ImageNet pretrained generative models from \citet{stein2023exposing}'s repository in terms of RKE$\uparrow$ (left) and KID$\downarrow$ (right) of generated samples from step $1$ to $t$ (plots' x-axis). Mixture-Greedy performed better than Mixture-UCB \citep{rezaei2025more} (with the UCB exploration term) in both the experiments with RKE and KID objectives. Mixture-Oracle always generates samples from the population-optimal mixture of models.}
     \label{fig:vendidogs}
\end{figure*}

The past decade has witnessed remarkable advances in generative modeling, leading to a wide availability of high-quality pretrained generators across diverse architectures and training paradigms. Early approaches such as variational autoencoders (VAEs) \citep{Kingma2013VAE} and generative adversarial networks (GANs) \citep{Goodfellow2014GAN} established scalable frameworks for learning complex data distributions. Subsequent architectural and training advances \citep{karras2019style,brock2018large,vahdat2020nvae}, together with the emergence of diffusion-based generative models \citep{Ho2020DDPM,song2021denoising,Rombach2022LDM}, have substantially improved both sample fidelity and diversity. As a result, practitioners often have access to multiple well-trained generative models that differ in architecture, training data, and inference cost.

In such settings, a fundamental practical question arises: how can one select or combine these generators in a sample-efficient manner? Generating samples from suboptimal models can be computationally expensive, and evaluation budgets are typically limited. Suppose we are given $m$ candidate generators, where the $i$th generator $\mathcal{G}_i$ induces a distribution $P_{\mathcal{G}_i}$. At each round $t = 1, 2, \dots, T$, we select a generator index $I_t \in [m]$, observe a sample $X_t \sim P_{\mathcal{G}_{I_t}}$, and update an estimate of a target evaluation score. The objective is to allocate samples so as to approach the best achievable evaluation performance while minimizing the number of samples drawn from suboptimal generators.

A natural formalization of this task is as a stochastic multi-armed bandit problem \citep{russo2018tutorial,slivkins2019introduction,lattimore2020bandit}. Standard bandit algorithms aim to balance exploration and exploitation in the online selection process. To do this, a well-established approach is the Upper Confidence Bound (UCB) principle \citep{Auer2002UCB}, which augments empirical estimates with optimism bonuses to encourage exploration of uncertain arms. In the context of generative model selection, however, the evaluation criteria are typically non-linear functionals of the underlying distribution rather than simple expected rewards. For example, the Fr\'echet Inception Distance (FID) \citep{heusel2017ttur} compares mean and covariance statistics of two distributions, and other widely used metrics such as Kernel Inception Distance (KID) \citep{binkowski2018demystifying} and precision-recall based measures \citep{sajjadi2018assessing,kynkaanniemi2019improvedpr} also depend on higher-order distributional properties. To address this, \citet{hu2025banditgen} proposed a bandit framework tailored to generative model evaluation and derived UCB-type confidence bounds for the FID metric, enabling online best-arm identification despite the non-linearity of the objective.

A notable finding in generative model selection is the recognition that the optimal solution need not be a single generator (see Figure~\ref{fig:vendidogs}). As shown by \citet{rezaei2025more}, under several diversity-aware evaluation metrics, a non-degenerate mixture of generators can strictly outperform every individual model. Let $\alpha \in \Delta_m := \{\alpha \in \mathbb{R}^m : \alpha \succeq 0,\ \mathbf{1}^\top \alpha = 1\}$ denote mixture weights and define the mixture distribution $
P_\alpha = \sum_{i=1}^m \alpha_i P_{\mathcal{G}_i}$. The mixture-model selection problem can then be formulated as
\begin{align}
\alpha^\star 
\; \in\; \underset{\alpha \in \Delta_m}{\arg\!\max}\;\; 
\mathrm{Score}\bigl(P_\alpha\bigr)
\label{eq:intro_mixture_problem}
\end{align}
Unlike classical best-arm identification, the optimizer of \eqref{eq:intro_mixture_problem} may lie in the interior of the simplex, implying that combining generators can strictly improve performance. To address this, \citet{rezaei2025more} proposed \emph{Mixture-UCB} algorithms that operate over the simplex and incorporate \emph{Upper~Confidence~Bound~(UCB)} bonuses. For quadratic kernel-based objectives such as KID \citep{binkowski2018demystifying} and RKE \citep{jalali2023information}, they showed that uniform high-probability guarantees can be obtained by estimating a finite number of pairwise expectations, making UCB-style exploration analytically tractable in that regime. Extending such guarantees to non-quadratic metrics, including FID and entropy-based diversity scores such as the Vendi score \citep{friedman2023vendi}, remains significantly more challenging, as the resulting UCB bound would be loose due to the non-quadratic structure, leading to over-exploration and substantially slow convergence to the optimal solution.

In this work, we revisit a foundational design choice in bandit-based generative model selection: Is an explicit UCB exploration bonus actually necessary? In classical best-arm bandits with linear rewards, purely greedy strategies are known to be brittle: early estimation noise can cause the algorithm to over-commit to an apparently good arm, stop collecting data from others, and never recover. UCB-style optimism was introduced precisely to prevent this collapse by forcing continued exploration \citep{Auer2002UCB}. 

Surprisingly, we empirically find that this classical intuition does not necessarily carry over to diversity-aware \emph{mixture} selection. Across extensive experiments on multiple datasets and evaluation metrics, adding the UCB bonus consistently slows convergence and often reduces sample efficiency. In contrast, a simple \emph{Mixture-Greedy} strategy that repeatedly optimizes the empirical objective over the simplex $\Delta_m$ (without any explicit exploration bonus) converges faster and achieves comparable or superior performance, especially for widely used non-quadratic metrics such as FID and Vendi, where constructing UCB bounds is challenging.

Our explanation is that diversity-aware mixture objectives can induce \emph{implicit exploration}. The main distinction from best-arm identification is that the decision space here is not the discrete set of arms but the continuous simplex of mixtures. For many diversity-aware metrics, the optimum is attained by a non-degenerate mixture, and the objective itself tends to disfavor degenerate solutions that place nearly all mass on a single generator. As a result, optimizing the empirical objective can naturally produce interior mixture weights, which in turn ensures that \emph{every} generator continues to receive samples. In other words, the mixture geometry plus the diversity-aware objective can create an exploration effect that is internal to the optimization problem, rather than an external confidence bonus appended to it.

We formalize this phenomenon by providing theoretical results showing that, under explicit structural conditions, Mixture-Greedy stays uniformly away from the boundary of $\Delta_m$, which yields linear sampling of all arms and sublinear regret without UCB. Our analyses cover representative families of objectives that matter in practice and were difficult to handle within prior UCB-based mixture frameworks: entropy-based objectives related to log-Vendi, FID-type Fr\'echet objectives under suitable interiority margin conditions, and kernel-based quadratic objectives including inverse-RKE-type formulations. Taken together with the empirical evidence, these results highlight a structural gap between diversity-aware mixture selection and classical linear-reward bandits: in the former, exploration may arise from the objective and the mixture structure themselves, suggesting that bandit methods for generative model selection should be designed around objective-induced geometry rather than relying solely on generic optimism bonuses. Here is a summary of this work's main contributions:
\begin{itemize}[leftmargin=*]
\item Demonstrating empirically that removing the UCB exploration bonus accelerates convergence and improves sample efficiency in mixture-based generative model multi-armed bandits.

\item Studying the theoretical nature of exploration in diversity-aware generative model selection and characterizing when explicit confidence bonuses are unnecessary.

\item Developing sublinear regret guarantees for Mixture-Greedy under entropy-based, kernel-based, and FID-type objectives via implicit exploration induced by objective.
\end{itemize}
\section{Related Work}\label{sec:relatedwork}
    \textbf{Evaluation metrics: fidelity, diversity, and memorization/novelty.}
Generative models are routinely evaluated with embedding-based sample metrics, including Inception Score (IS) \citep{salimans2016improved}, Fr\'echet Inception Distance (FID) \citep{ heusel2017ttur} as well as precision--recall style decompositions that separate quality from coverage \citep{sajjadi2018assessing,kynkaanniemi2019improvedpr,naeem2020reliable}.
Beyond distribution-level scores, recent work emphasizes sample-level diagnostics for memorization and novelty, including authenticity-style auditing \citep{alaa2022faithful}, feature-likelihood based generalization measures \citep{jiralerspong2023fld}, and rarity score \citep{han2023rarity}.


Another line of work highlights that these metrics inherit biases from the embedding space.
\citet{kynkaanniemi2023role} show that FID is tightly coupled to ImageNet class geometry; \citet{stein2023exposing} demonstrate systematic mismatches between embedding-based scores and human evaluation.
In our numerical analysis, we use CLIP \citep{clip} and DINOv2 \citep{dinov2} backbone embeddings following the results and suggestions of these references.

\textbf{Online selection of generators, mixtures, and hyperparameters.}
Online allocation of sampling/evaluation budgets across pretrained generators can be cast as a stochastic bandit problem.
\citet{hu2025banditgen} derive UCB-style guarantees for online selection under FID-type objectives, and \citet{rezaei2025more} show that diversity-aware objectives may be optimized by \emph{interior} mixtures, proposing Mixture-UCB methods with guarantees for quadratic kernel objectives.
\citet{chen2024fast} view generative model hyperparameter search as adaptive resource allocation, proposing successive-Halving-style procedures guided by distributional tests.
Complementarily, \citet{hosseini2026efficient} study the robustness of offline bandit-based model evaluation, showing that small adversarial perturbations to high-dimensional reward models can significantly alter the resulting bandit behavior.
Our work focuses on a different design question: in diversity-aware \emph{mixture} selection, is explicit optimism necessary, or can the objective geometry enforce sampling of all arms?

\textbf{Greedy bandits: failures and implicit exploration under structure.}
In classical stochastic bandits with linear rewards, purely greedy policies can incur linear regret due to early mis-ordering, motivating explicit exploration principles such as UCB \citep{Auer2002UCB,lattimore2020bandit}.
However, positive results show that greedy (or exploration-light) algorithms can achieve sublinear regret when the problem structure induces implicit exploration, including many-armed regimes \citep{bayati2020unreasonable} and contextual/linear bandits under diversity or smoothing assumptions \citep{bastani2021mostly,kannan2018smoothed}. Furthermore, recent work highlights that UCB-style optimism can be inefficient for non-linear bandit objectives. For instance, \citep{rajaraman2024statisticalcomplexityoptimalalgorithms} prove that UCB is statistically suboptimal for non-linear ridge bandits. 
Our analysis establishes an analogous phenomenon for diversity-aware \emph{mixture} objectives in generative model selection.

\textbf{Kernel-based diversity objectives and mixture selection.}
Kernel methods have become a central tool for evaluating generative models because they compare distributions through the pairwise similarity geometry of generated samples. Classical examples include MMD-based two-sample comparison~\citep{gretton2012kernel} and the Kernel Inception Distance~\citep{binkowski2018demystifying,wang2023distributedKID}, while more recent entropy- and spectrum-based criteria such as RKE~\citep{jalali2023information}, Vendi~\citep{friedman2023vendi,ospanov2024towards}, KEN~\citep{pmlr-v235-zhang24bh,Zhang_2025_CVPR}, and conditional kernel entropy~\citep{jalali2025sparke,Ospanov_2025_ICCV,jalali2026conditional} use kernel eigenstructure to quantify diversity, novelty, and conditional diversity. Quality-weighted variants of Vendi score incorporate per-item utility to trade off quality and diversity within a unified similarity-based objective \citep{nguyen2024qualityweightedvendi}. Additionally, \cite{ospanov2025vendiconvergence} study statistical convergence of RKE/Vendi and proposes truncated variants with finite-sample convergence guarantees.

Beyond evaluation, kernel matrices have also been used to compare, align, and fuse learned representations~\citep{pmlr-v267-jalali25a,gong2025kernel,wu2025fusing,wu2026maximum_isit,wu2026koda}, suggesting that embeddings can be studied through their induced similarity geometry rather than only through downstream prediction accuracy. This perspective is particularly relevant for generative-model selection: recent work has shown that embedding-based scores can expose systematic differences between generators~\citep{stein2023exposing}, support online model evaluation and selection~\citep{hu2025banditgen,hu2025pak,lewandowski2025matchchoosemodel,jafari2026dak,jules2026nemos}, and, importantly, favor non-degenerate mixtures of generators over any single model~\citep{rezaei2025more}. Our work builds on this kernel-geometric viewpoint, but focuses on a different algorithmic consequence: when the evaluation objective itself rewards diverse mixtures, the resulting simplex geometry can naturally keep multiple generators active, thereby reducing the need for a UCB bonus.

\section{Preliminaries}\label{sec:pre}
 \subsection{Generative model evaluation}

A generative model $\mathcal{G}$ induces a probability distribution $P_{\mathcal{G}}$ over a sample space $\mathcal{X}$. 
We write $x \sim P_{\mathcal{G}}$ for a generated sample. In practice, evaluation is performed in a fixed embedding space. Throughout our analysis, we assume $x \in \mathbb{R}^d$ denotes the \emph{embedded representation} of a sample. Mathematically, given i.i.d.\ generated samples $x_1,\dots,x_n \sim P_{\mathcal{G}}$ and real samples $y_1,\dots,y_{n_{\mathrm{real}}} \sim P_{\mathrm{data}}$ in the same embedding space, an evaluation metric computes a scalar score $\mathrm{Score}(x_1,\dots,x_n)$ reflecting fidelity and/or diversity.

\textbf{Fr\'echet Distance (FD).}
Let $\widehat\mu_x, \widehat\Sigma_x$ denote the empirical mean and covariance matrix of $\{x_i\}_{i=1}^n$, and similarly $\widehat\mu_y, \widehat\Sigma_y$ for real data. The empirical Fr\'echet distance is
\begin{align}\label{eq:fd_def}
&\mathrm{FD}(\widehat{P}_G,\widehat{P}_{\mathrm{data}}) \\
=\:
&\|\widehat\mu_x-\widehat\mu_y\|_2^2
+
\,\mathrm{Tr}\Bigl(\widehat\Sigma_x+\widehat\Sigma_y-2\bigl(\widehat\Sigma_x^{1/2}\widehat\Sigma_y\widehat\Sigma_x^{1/2}\bigr)^{1/2}\Bigr) \nonumber
\end{align}

\subsection{Kernel-based scores}

A kernel is a symmetric positive semidefinite function 
$k:\mathcal{X}\times\mathcal{X}\to\mathbb{R}$.
Equivalently, there exists a Hilbert space $\mathcal{H}$ and feature map $\phi:\mathcal{X}\to\mathcal{H}$ such that
\[
k(x,y) = \bigl\langle \phi(x), \phi(y) \bigr\rangle_{\mathcal{H}}.
\]
Throughout, we assume the kernel is normalized: $k(x,x)=1$ for every $x\in\mathcal{X}$.
Note that even if a kernel function $k$ is not normalized, we can replace $k$ with the normalized $\widetilde{k}(x,y)=k(x,y)/\sqrt{k(x,x)k(y,y)}$ that remains a valid kernel function. Given samples $x_1,\dots,x_n$, let $K\in\mathbb{R}^{n\times n}$ denote the kernel (Gram) matrix with entries defined as $K_{ij}=k(x_i,x_j)$ for every $1\le i,j\le n$.

\textbf{Kernel Distance (KD).}
The squared maximum mean discrepancy (MMD) between distributions $P$ and $Q$ is
\begin{align}
\mathrm{KD}(P,Q)
=
\mathbb{E}\bigl[k(X,X')\bigr]
+
\mathbb{E}\bigl[k(Y,Y')\bigr]
-
2\mathbb{E}\bigl[k(X,Y)\bigr]
\label{eq:mmd_def}
\end{align}
where $X,X'\sim P$ and $Y,Y'\sim Q$ are independent. Empirical estimators replace expectations with sample averages.

\textbf{Von Neumann entropy and Vendi.}
Define the normalized Gram matrix $\rho = \tfrac{1}{n}K$, that is PSD and unit-trace given a normalized kerk function.
Let $\lambda_1,\dots,\lambda_n$ be its eigenvalues, which form a probability model as they are non-negative and add to $1$. The von Neumann entropy (VNE) of this density matrix is defined as
\[
\mathrm{VNE}\bigl(\rho\bigr)\,
=\, - \mathrm{Tr}(\rho \log \rho)\, =\, 
\sum_{i=1}^n \lambda_i \log \frac{1}{\lambda_i}
\]
The Vendi score \citep{friedman2023vendi} is defined as the exponential of the VNE of $\tfrac{1}{n}K$
\begin{align}
\mathrm{Vendi}(x_1,\dots,x_n)
=
\exp\Bigl(\mathrm{VNE}\bigl(\tfrac{1}{n}K\bigr)\Bigr).
\label{eq:vendi_def}
\end{align}

\textbf{RKE and inverse-RKE.}
For a distribution $P$, \cite{jalali2023information} define the RKE mode-count as follows:
\[
\mathrm{RKE}(P)
=
\frac{1}{\mathbb{E}_{X,X'\stackrel{\text{iid}}{\sim} P}\bigl[k^2(X,X')\bigr]},
\]
which gives the Inverse-RKE definition as $\mathrm{InvRKE}(P)
=
\mathbb{E}_{X,X'\sim P}[k(X,X')^2]$. For the empirical distribution $\widehat P_n$ supported on $\{x_i\}_{i=1}^n$, the above definitions reduce to $
\mathrm{InvRKE}(\widehat P_n)
=
\frac{1}{n^2}\|{K}\|_F^2$, and
$\mathrm{RKE}(\widehat P_n)
=
{n^2}/{\|{K}\|_F^2}$.

\subsection{Multi-armed bandits for online generative model selection}

We consider $m$ pretrained generators $\mathcal{G}_1,\dots,\mathcal{G}_m$, with corresponding distributions $P_{\mathcal{G}_1},\dots,P_{\mathcal{G}_m}$. 
At each round $t=1,\dots,T$, the learner selects $I_t\in[m]$ and observes $X_t\sim P_{\mathcal{G}_{I_t}}$. 
The objective is to maximize an evaluation score computed from the accumulated generated samples. This task can be naturally viewed as a stochastic multi-armed bandit \citep{slivkins2019introduction,lattimore2020bandit,hu2025banditgen}, where the online learner aims to identify the best model by successive queries to the models.

Beyond selecting a single generator, we consider identifying a mixture of the models. Let $\Delta_m := \{\alpha\in\mathbb{R}^m:\alpha_i\ge0,\ \sum_{i=1}^m \alpha_i=1\}$ be the probability simplex of weights assigned to the $m$ models,
and define the mixture distribution
$ P_\alpha = \sum_{i=1}^m \alpha_i P_{\mathcal{G}_i}$.
The optimization problem is therefore \eqref{eq:intro_mixture_problem},
where standard choices for $\mathrm{Score}$ can be negative-FD, negative-KD, Vendi, RKE, or a composite score combining multiple scores, especially combining the pure-diversity scores with quality precision or density scores.

To address the task, \citet{rezaei2025more} propose the Mixture-UCB algorithm that adapts the established upper-confidence-bound (UCB) approach to the mixture setting. Note that the proposed Mixture-UCB in \citep{rezaei2025more} is designed exclusively for quadratic evaluation scores, including KD, RKE, and their linear combinations with Precision and Density fidelity scores. Furthermore, another natural algorithm that we will analyze in our work is \emph{Mixture-Greedy}, which assigns no UCB bonus term and at each round optimizes the empirical objective over $\Delta_m$ and samples $I_t$ from the resulting mixture $\alpha_t$. We highlight that the Mixture-Greedy approach is directly applicable to all the evaluation scores discussed in this section.

\section{Methodology: Mixture-Greedy \& score-based simplex program}
\label{sec:method_mg}

 We consider $m$ pretrained generators $\mathcal{G}_1,\dots,\mathcal{G}_m$ with induced distributions
$P_{\mathcal{G}_i}$ over embedded samples in $\mathbb{R}^d$.
At round $t$, Mixture-Greedy forms an empirical loss $\widehat{\mathcal{L}}_{t-1}(\alpha)$ over the simplex
$\Delta_m:=\{\alpha\in\mathbb{R}^m:\alpha\succeq 0,\ \mathbf{1}^\top\alpha=1\}$ and selects
\begin{equation}
\alpha_t \in \arg\min_{\alpha\in\Delta_m}\ \widehat{\mathcal{L}}_{t-1}(\alpha).
\label{eq:mg_update_method}
\end{equation}
It then samples an arm $I_t\sim \alpha_t$ and draws $X_t\sim P_{\mathcal{G}_{I_t}}$.

We focus on three objective families used throughout the paper: (i) Fr\'echet distance (FD) via Gaussian moment matching,
(ii) negative log-Vendi (equivalently, negative von Neumann entropy), and (iii) convex quadratic scores (KD/MMD, inverse-RKE),
possibly combined with a linear fidelity term (Precision/Density-type).

\textbf{(i) FD via Gaussian moment matching.}
Fix the empirical real-data moments $(\widehat\mu_{\rm data},\widehat\Sigma_{\rm data})$ in the embedding space.
Let $\widehat\mu_i(t)\in\mathbb{R}^d$ and $\widehat S_i(t)\in\mathbb{R}^{d\times d}$ denote, respectively, the empirical mean and
uncentered second moment of samples from arm $i$ up to time $t$:
$\widehat S_i(t)=\frac{1}{n_i(t)}\sum_{r=1}^{n_i(t)}x_{i,r}x_{i,r}^\top\succeq 0$.
For $\alpha\in\Delta_m$, define the mixture mean and covariance via the exact mixture identities
\begin{align}
\widehat\mu_t(\alpha) &:= \sum_{i=1}^m \alpha_i \widehat\mu_i(t),
\label{eq:fd_mu_mix_method}\\
\widehat\Sigma_t(\alpha) &:= \sum_{i=1}^m \alpha_i \widehat S_i(t)-\widehat\mu_t(\alpha)\widehat\mu_t(\alpha)^\top .
\label{eq:fd_Sigma_mix_method}
\end{align}
The FD-type loss is then the Gaussian Fr\'echet functional
\begin{align}
\widehat{\mathcal{L}}^{\rm FD}_t(\alpha)
:=\ &\|\widehat\mu_t(\alpha)-\widehat\mu_{\rm data}\|_2^2
+\mathrm{Tr}\big(\widehat\Sigma_t(\alpha)+\widehat\Sigma_{\rm data}\big) \nonumber\\
&-2\,\mathrm{Tr}\Big(\big(\widehat\Sigma_{\rm data}^{1/2}\widehat\Sigma_t(\alpha)\widehat\Sigma_{\rm data}^{1/2}\big)^{1/2}\Big).
\label{eq:fd_loss_method}
\end{align}

\textbf{(ii) Negative log-Vendi from pooled samples.}
Let $\{x_j\}_{j=1}^N$ be the pooled generated samples available at time $t$,
where each $x_j$ came from arm $I_j\in[m]$, and let $K\in\mathbb{R}^{N\times N}$ be the normalized kernel Gram matrix
$K_{rs}=k(x_r,x_s)$. For $\alpha\in\Delta_m$, define per-sample weights
\begin{equation}
q_j(\alpha):=\frac{\alpha_{I_j}}{n_{I_j}(t)},\qquad j\in[N],
\label{eq:q_def_method}
\end{equation}
so $q(\alpha)\succeq 0$ and $\sum_{j=1}^N q_j(\alpha)=1$.
Consider the following unit-trace PSD matrix
\begin{equation}
\rho_t(\alpha):=\mathrm{diag}(q(\alpha))^{1/2}\,K\,\mathrm{diag}(q(\alpha))^{1/2},
\label{eq:rho_def_method}
\end{equation}
where we note that $\mathrm{Tr}(\rho_t(\alpha))=1$ holds and we set the negative log-Vendi loss as
\begin{equation}
\widehat{\mathcal{L}}^{\rm Vendi}_t(\alpha):=\mathrm{Tr}\big(\rho_t(\alpha)\log \rho_t(\alpha)\big).
\label{eq:neglogvendi_method}
\end{equation}
Minimizing \eqref{eq:neglogvendi_method} is equivalent to maximizing $\mathrm{VNE}(\rho_t(\alpha))$ and hence the Vendi score of the generated samples.

\textbf{(iii) Convex quadratics and linear fidelity.}
Let $\widehat K(t)\in\mathbb{R}^{m\times m}$ be symmetric PSD and $\widehat b(t)\in\mathbb{R}^m$ estimate a linear functional
$b_i=\mathbb{E}[\psi_i(X)]$ with $\psi_i\in[0,1]$. We consider
\begin{equation}
\widehat{\mathcal{L}}^{\rm quad}_t(\alpha):=\alpha^\top \widehat K(t)\alpha+w\,\alpha^\top \widehat b(t),\qquad w\ge 0.
\label{eq:quad_loss_method}
\end{equation}
The following proposition shows that the above loss functions yield convex functions of the mixture weights, which can be optimized by convex optimization algorithms.
\begin{proposition}
\label{thm:mg_convex_programs}
For each fixed $t$, the update \eqref{eq:mg_update_method} is a convex optimization problem over $\Delta_m$ for:
(i) $\widehat{\mathcal{L}}^{\rm FD}_t(\alpha)$ in \eqref{eq:fd_loss_method},
(ii) $\widehat{\mathcal{L}}^{\rm Vendi}_t(\alpha)$ in \eqref{eq:neglogvendi_method},
and (iii) $\widehat{\mathcal{L}}^{\rm quad}_t(\alpha)$ in \eqref{eq:quad_loss_method} whenever $\widehat K(t)\succeq 0$.
\end{proposition}
\begin{proof}
    We provide the proof in the Appendix.
\end{proof}

\textbf{Exponentiated-gradient (EG) solver and practical instantiation.}
To implement the Mixture-Greedy update \eqref{eq:mg_update_method} in an efficient manner on the simplex,
we apply the exponentiated gradient (EG) descent, a mirror-descent method with KL geometry that maintains $\alpha\in\Delta_m$ by construction.
Starting from a feasible iterate $\alpha^{(0)}$ (we use the uniform distribution $\tfrac{1}{m}\mathbf{1}$ in our experiments), every EG step updates
\begin{equation}
\widetilde\alpha^{(s+1)}_i = \alpha^{(s)}_i\exp\bigl(-\eta g^{(s)}_i\bigr),
\;\;
\alpha^{(s+1)} = \frac{1}{\Vert \widetilde\alpha^{(s+1)}\Vert_1}\widetilde\alpha^{(s+1)},
\label{eq:eg_update_maintext}
\end{equation}
where $g^{(s)}=\nabla_\alpha \widehat{\mathcal{L}}_{t-1}(\alpha^{(s)})$ and $\eta>0$ is a stepsize.
In our setting, the objective families in this section lead to smooth convex losses over $\Delta_m$
(Proposition~\ref{thm:mg_convex_programs}); therefore, a fixed number of EG steps per round would be sufficient.
Algorithm~\ref{alg:mg_eg} summarizes the main steps.

\textbf{Score-specific gradients.}
For quadratic objectives \eqref{eq:quad_loss_method}, the gradient is explicit:
$\nabla_\alpha \widehat{\mathcal{L}}^{\rm quad}_t(\alpha)=2\widehat K(t)\alpha+w\,\widehat b(t)$.
For FD \eqref{eq:fd_loss_method} and log-Vendi \eqref{eq:neglogvendi_method}, gradients require matrix-function evaluations
(e.g., a matrix square-root/inverse-square-root for FD, and a matrix logarithm for log-Vendi).
As we discuss, one can compute these via eigendecomposition.
To reduce cost for log-Vendi when the pooled sample size is significantly large, we also use a finite-dimensional feature representation
(e.g., random Fourier features for shift-invariant kernels), replacing the $N\times N$ kernel-matrix logarithm by a
$D\times D$ feature-covariance logarithm when $D\ll N$.
The full derivations are provided in Appendix~\ref{app:method_proofs_grads}.

\begin{algorithm}[t]
\caption{Mixture-Greedy via Exponentiated Gradient}
\label{alg:mg_eg}
\small
\begin{algorithmic}[1]
\REQUIRE Models $\{\mathcal{G}_i\}_{i=1}^m$, horizon $T$, warm start $M\ge 1$, EG stepsize $\eta>0$, EG steps $S\ge 1$
\STATE Initialize counts $n_i(0)=M$ for all $i\in[m]$ and draw $M$ warm-start samples from each model to initialize statistics
\FOR{$t=1,2,\dots,T$}
    \STATE Form the empirical loss $\widehat{\mathcal{L}}_{t-1}(\alpha)$ from samples up to round $t-1$
    \STATE Initialize $\alpha^{(0)} \gets \alpha_{t-1}$ (warm start) or $\alpha^{(0)}\gets \frac{1}{m}\mathbf{1}$
    \FOR{$s=0,1,\dots,S-1$}
        \STATE Compute $g^{(s)} \gets \nabla_\alpha \widehat{\mathcal{L}}_{t-1}\big(\alpha^{(s)}\big)$
        \STATE $\widetilde{\alpha}^{(s+1)}_i \gets \alpha^{(s)}_i \exp(-\eta\, g^{(s)}_i)$ for all $i\in[m]$
        \STATE $\alpha^{(s+1)} \gets \widetilde{\alpha}^{(s+1)}/\sum_{j=1}^m \widetilde{\alpha}^{(s+1)}_j$
    \ENDFOR
    \STATE Set $\alpha_t \gets \alpha^{(S)}$, sample $I_t \sim \alpha_t$, draw $X_t \sim P_{\mathcal{G}_{I_t}}$
    \STATE Update count $n_{I_t}(t)\gets n_{I_t}(t-1)+1$, and update the empirical statistics needed for $\widehat{\mathcal{L}}_t$
\ENDFOR
\end{algorithmic}
\end{algorithm}

\section{Implicit exploration and regret for Diversity Scores}
\label{sec:mg_regret}

We analyze Mixture-Greedy for two diversity-aware objectives:
(i) negative log-Vendi (negative von Neumann entropy of the kernel matrix), and
(ii) the Fr\'echet distance functional under mean and covariance matching.
Our goal is to characterize structural conditions under which Mixture-Greedy, without any explicit UCB bonus, achieves sublinear regret.

Regarding the bandit protocol and regret definition, we consider a warm start of $M\ge 1$ samples per arm (i.e., $n_i(0)=M$).
At each round $t=1,\dots,T$, Mixture-Greedy computes
\begin{equation}\label{eq:mg_update_main}
\alpha_t \in \arg\min_{\alpha\in\Delta_m}\widehat F_{t-1}(\alpha),
\qquad I_t\sim \alpha_t,
\end{equation}
where $\widehat F_{t-1}$ is the plug-in empirical objective formed from samples up to time $t-1$.
We measure regret against the best fixed mixture:
\begin{equation}\label{eq:mg_regret_main}
\mathrm{Reg}_T := \sum_{t=1}^T\Bigl[F(\alpha_t)-\min_{\alpha\in\Delta_m}F(\alpha)\Bigr].
\end{equation}

\subsection{Negative log-Vendi induces implicit exploration}
\label{subsec:mg_regret_vendi_main}

Let $\phi:\mathcal{X}\to\mathbb{R}^d$ be a finite-dimensional feature map.
Define $S_i:=\mathbb{E}[\phi(X)\phi(X)^\top]$ for $X\sim P_{\mathcal{G}_i}$ and
$S_\alpha:=\sum_{i=1}^m\alpha_i S_i$.
The negative log-Vendi objective is
\[
F_{\mathrm{NLV}}(\alpha):=\mathrm{Tr}(S_\alpha\log S_\alpha).
\]

\begin{assumption}[Normalized kernel]\label{ass:vendi_norm_main}
$\|\phi(x)\|_2=1$ for all $x$.
\end{assumption}

\begin{assumption}[Population innovation for log-Vendi]\label{ass:vendi_innov_main}
There exist constants $\nu_0\in(0,1]$ and $\varepsilon_0\in[0,\nu_0/8)$ such that for every $i\in[m]$
there exists a unit vector $v_i$ satisfying
$v_i^\top S_i v_i \ge \nu_0$ and $\sum_{j\neq i} v_i^\top S_j v_i \le \varepsilon_0$.
\end{assumption}

Under these structural conditions, the entropy geometry enforces a uniform interiority property for empirical minimizers.
This yields linear sampling of all arms and enables time-uniform concentration.

\begin{theorem}[Mixture-Greedy regret for negative log-Vendi]\label{thm:nlv_main}
Fix $T\ge 1$ and $\delta\in(0,1)$.
Under Assumptions~\ref{ass:vendi_norm_main}--\ref{ass:vendi_innov_main}, there exists a warm-start size
$M=M(d,m,\nu_0,\varepsilon_0,T,\delta)$ such that, with probability at least $1-\delta$ the following holds, where $C_{\mathrm{NLV}}$ depends only on $(d,m,\nu_0,\varepsilon_0)$:
\[
\mathrm{Reg}_T
\ \le\
C_{\mathrm{NLV}}\Bigl(1+\sqrt{\log\frac{m(T+1)}{\delta}}\Bigr)\sqrt{T}\,(1+\log T),
\]
\end{theorem}
\begin{proof}
We provide the proof in the Appendix.
\end{proof}

\subsection{Fr\'echet Distance: regret under a population interiority margin}
\label{subsec:mg_regret_fid_main}

The Fr\'echet Distance functional is smooth at the simplex boundary and may admit boundary minimizers.
We therefore state a population interiority condition ensuring that optimal mixtures lie in the interior.

Let arm $i$ produce i.i.d.\ embeddings $Z\in\mathbb{R}^d$ with mean $\mu_i$ and covariance $\Sigma_i$.
For $\alpha\in\Delta_m$ define
\[
\mu_\alpha=\sum_{i=1}^m\alpha_i\mu_i,\;
\Sigma_\alpha=\sum_{i=1}^m\alpha_i\Bigl(\Sigma_i+(\mu_i-\mu_\alpha)(\mu_i-\mu_\alpha)^\top\Bigr)
\]
Fix $(\mu_0,\Sigma_0)$ with $\Sigma_0\succeq \lambda_0 I$ and define
\[
F_{\mathrm{FD}}(\alpha)
:=\|\mu_\alpha-\mu_0\|_2^2+\mathrm{Tr}\Bigl(\Sigma_\alpha+\Sigma_0
-2(\Sigma_0^{1/2}\Sigma_\alpha\Sigma_0^{1/2})^{1/2}\Bigr)
\]

\begin{assumption}[Bounded embeddings]\label{ass:fid_bdd_main}
There exists $B<\infty$ such that $\|Z\|_2\le B$ almost surely for every arm.
\end{assumption}

\begin{assumption}[Uniform positive definiteness]\label{ass:fid_pd_main}
There exists $\nu>0$ such that $\Sigma_i\succeq \nu I$ for all $i\in[m]$.
\end{assumption}

\begin{assumption}[Population interiority margin]\label{ass:fid_margin_main}
There exist $\gamma_0\in(0,1/m]$ and $\Delta_0>0$ such that
\[
\inf_{\alpha\in\Delta_m:\ \min_i\alpha_i\le \gamma_0}F_{\mathrm{FD}}(\alpha)
\ \ge\ \min_{\alpha\in\Delta_m}F_{\mathrm{FD}}(\alpha)+\Delta_0.
\]
\end{assumption}

\begin{theorem}[Mixture-Greedy regret for Fr\'echet Distance]\label{thm:fid_main}
Fix $T\ge 1$ and $\delta\in(0,1)$.
Under Assumptions~\ref{ass:fid_bdd_main}--\ref{ass:fid_margin_main}, there exists
$M=M(B,\lambda_0,\nu,m,d,\gamma_0,\Delta_0,T,\delta)$ such that, with probability at least $1-\delta$, where $C_{\mathrm{FD}}$ depends only on $(B,\lambda_0,\nu,m,d,\gamma_0)$:
\[
\mathrm{Reg}_T
\ \le\
C_{\mathrm{FD}}\Bigl(1+\sqrt{\log\frac{m(T+1)}{\delta}}\Bigr)\sqrt{T}
\]
\end{theorem}
\begin{proof}
We provide the proof in the Appendix.
\end{proof}

\begin{figure*}[t]
        \centering
        \includegraphics[width=\linewidth]{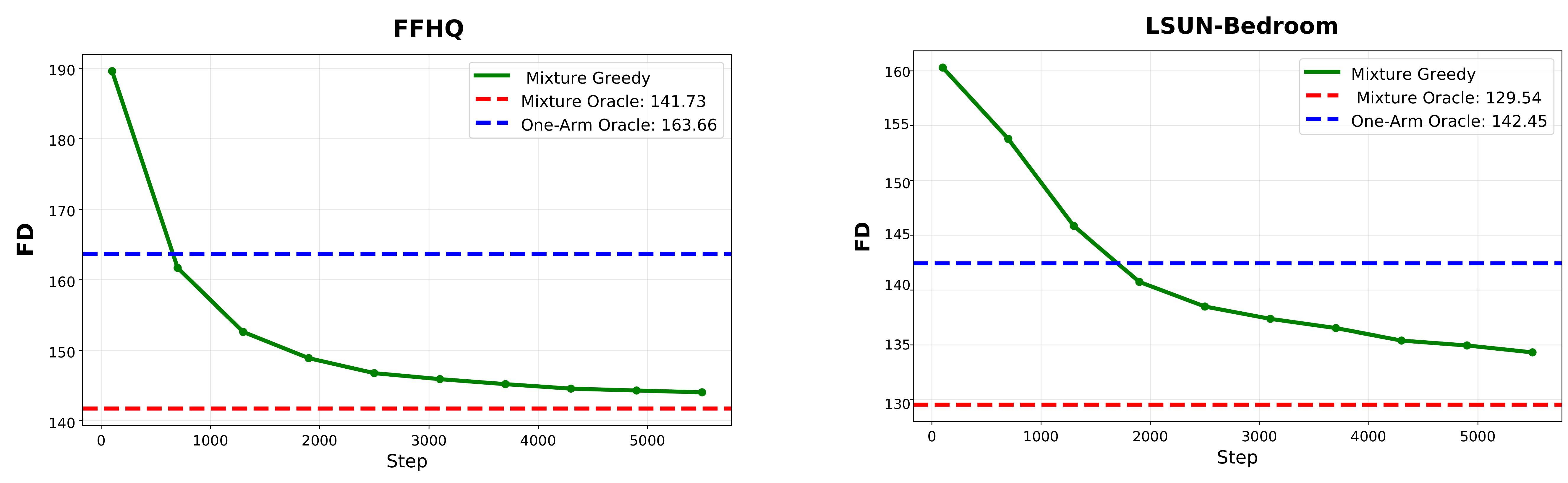}
       
    \caption{Convergence of Mixture Greedy to the Mixture Oracle in terms of Fréchet Distance (FD) across image selection steps on FFHQ (left) and LSUN-Bedroom (right).}
     \label{fig:FID_ffhq_lsun}

\end{figure*}

\begin{remark}
Our analysis extends to the RKE diversity objective, $\mathrm{RKE}(\alpha)=1/\mathrm{Tr}(S_\alpha^2)$,
optimized via the equivalent convex functional $F_{\mathrm{RKE}}(\alpha)=\mathrm{Tr}(S_\alpha^2)$.
Under a standard population margin condition analogous to Assumption~\ref{ass:fid_margin_main},
Mixture-Greedy achieves an $\widetilde{\mathcal{O}}(\sqrt{T})$ regret bound.
The same rate holds when the objective is augmented with linear fidelity terms.
The statements and proofs are in Appendix~\ref{app:rke_mg}.
\end{remark}

\subsection{Including a linear fidelity term}
\label{subsec:mg_linear_main}

The fidelity scores for generative models, including Precision~\citep{kynkaanniemi2019improvedpr} and Density~\citep{naeem2020reliable} are linear functionals of the generator distribution.
We model this using bounded functions $\psi_i:\mathcal{X}\to[0,1]$:
\[
G(\alpha):=\sum_{i=1}^m \alpha_i \theta_i,
\quad
\theta_i:=\mathbb{E}_{X\sim P_{\mathcal{G}_i}}[\psi_i(X)]\in[0,1],
\]
and for a base diversity-aware objective $F(\alpha)$, we consider
$H(\alpha):=F(\alpha)+w\,G(\alpha)$ for given coefficient $w>0$.
Given this objective, Mixture-Greedy minimizes the plug-in empirical objective
$\widehat H_{t-1}(\alpha)=\widehat F_{t-1}(\alpha)+w\,\widehat G_{t-1}(\alpha)$.

\textbf{Negative log-Vendi.}
The linear increment along any simplex direction is uniformly bounded,
thus the entropy-induced interiority mechanism remains intact,
with the quantitative floor adjusted by a controlled $w$-dependent term.

\begin{corollary}[Mixture-Greedy regret for negative log-Vendi + linear term]
\label{cor:nlv_linear_main}
Assume the conditions of Theorem~\ref{thm:nlv_main} and $0\le \psi_i\le 1$.
For $M$ sufficiently large (see Appendix~\ref{app:mg_linear}),
with probability at least $1-\delta$,
\[
\sum_{t=1}^T\Bigl(H(\alpha_t)-\min_{\alpha\in\Delta_m}H(\alpha)\Bigr)
 \le
C_{\mathrm{NLV},w}\sqrt{T\log\frac{m(T+1)}{\delta}},
\]
where $C_{\mathrm{NLV},w}$ depends only on $(d,m,\nu_0,\varepsilon_0,w)$.
\end{corollary}

\textbf{Fr\'echet Distance.}
Since $G(\alpha)\in[0,1]$, adding $wG$ perturbs any interiority margin by at most $w$,
hence the Fr\'echet Distance guarantee extends under the corresponding margin condition for $H$.

\begin{corollary}[Mixture-Greedy regret for Fr\'echet Distance + linear term]
\label{cor:fid_linear_main}
Assume the conditions of Theorem~\ref{thm:fid_main} and $0\le \psi_i\le 1$.
If $H(\alpha)=F_{\mathrm{FD}}(\alpha)+wG(\alpha)$ satisfies the population interiority margin
(Assumption~\ref{ass:fid_margin_main} with $F_{\mathrm{FD}}$ replaced by $H$),
then with probability at least $1-\delta$,
\[
\sum_{t=1}^T\Bigl(H(\alpha_t)-\min_{\alpha\in\Delta_m}H(\alpha)\Bigr)
 \le
C_{\mathrm{FD},w}\sqrt{T\log\Bigl(\frac{m(T+1)}{\delta}\Bigr)},
\]
where $C_{\mathrm{FD},w}$ depends only on $(B,\lambda_0,\nu,m,d,\gamma_0,w)$.
\end{corollary}


\section{Numerical Results}\label{numerical}
    \begin{figure*}[t]

        \centering
        \includegraphics[width=\linewidth]{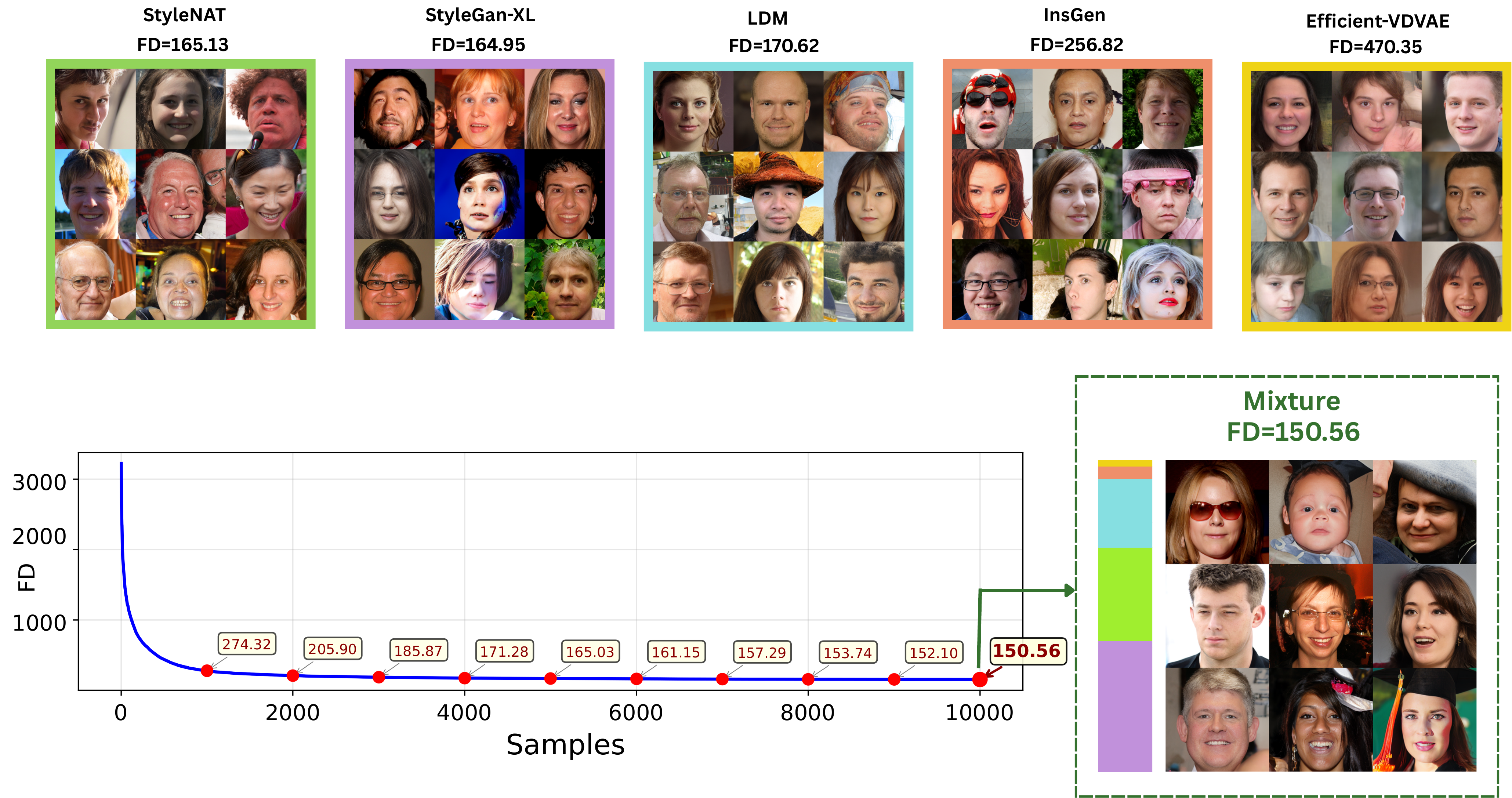}

    \caption{Mixture Greedy, minimizing Frechet Distance using exponentiated gradient descent on FFHQ generated samples.}
    \label{fig:FID}
\end{figure*}

We assess the performance of the proposed approach in both real-data and 
synthetic environments. For the real-data experiments, we adopt benchmark 
datasets standard in ~\citep{stein2023exposing}. 
Complementarily, we constructed synthetic scenarios with controlled data-generating 
mechanisms. Please refer to the Appendix \ref{app:syn} for details and results. For image feature extraction, we use DINOv2-ViT-L/14 \citep{dinov2} following the study by \citet{stein2023exposing}. For text encoding, we use SBERT \citep{sbert}. 

\subsection{Real-World Settings}
In the real-data setting, we evaluate our method on samples from pretrained generative 
models provided by the evaluation benchmark in ~\citep{stein2023exposing}. 
Each generative model is treated as an arm in our experimental framework. 
Here, we consider three datasets: 
FFHQ, LSUN-Bedroom, and ImageNet. 
For each dataset, we treat pretrained generative models as arms in our framework. All models are pre-trained on the dataset provided by~\citet{stein2023exposing}.

\textbf{FFHQ.}
For the FFHQ dataset~\citep{karras2019style}, 
we considered five distinct generative models:
LDM~\citep{Rombach2022LDM}, 
StyleGAN-XL~\citep{sauer2022styleganxl}, 
Efficient-VDVAE~\citep{hazami2022efficientvdvae}, 
InsGen~\citep{yang2021insgen}, 
and StyleNAT~\citep{walton2022stylenat}. 

\textbf{LSUN-Bedroom.}
For the LSUN-Bedroom dataset, we used generated samples from four models:
StyleGAN~\citep{karras2019ffhq}, 
Projected GAN~\citep{sauer2021projectedgan}, 
iDDPM~\citep{nichol2021iddpm}, 
and Unleashing Transformers~\citep{bondtaylor2022unleashing}. 

\textbf{ImageNet.}
For ImageNet, we evaluate four pretrained generative models:
DiT-XL-2-guided~\citep{peebles2023dit}, 
LDM~\citep{Rombach2022LDM}, 
RQ-Transformer~\citep{lee2022rqtransformer}, 
and StyleGAN-XL~\citep{sauer2022styleganxl}. 

\textbf{Mixture Improves Fréchet Distance.}
Figure \ref{fig:FID_ffhq_lsun} shows the evolution of FD across selection rounds on two datasets. Alongside the online Mixture Greedy algorithm, we report two oracle baselines: the One-Arm Oracle, which samples exclusively from the single generator with the lowest standalone FD, and the Mixture Oracle, which assumes knowledge of the optimal mixture weights $\alpha^\star$ and provides the corresponding lower bound on FD.

Across both datasets, Mixture Greedy rapidly decreases FD and converges to the Mixture Oracle within a few rounds, consistently outperforming the One-Arm Oracle. This highlights the advantage of optimizing over generator mixtures rather than selecting a single best generator. See Figure \ref{fig:in_fd} in Appendix for ImageNet results.
Figure \ref{fig:FID} illustrates the FD values obtained from individual generative 
models (arms) trained on FFHQ, as well as the FD achieved by their 
optimized mixture. While the best single model attains an FD of 164.13 
(StyleGAN-XL), the samples from the mixture achieves a substantially lower 
FD of 155.13, outperforming every individual arm. The optimization is performed 
using exponentiated gradient descent.

\section{Conclusion and Limitations}\label{conclusion}
    We studied online selection among pretrained generative models under diversity-aware evaluation objectives. While classical bandit algorithms rely on explicit optimism to avoid premature commitment to a single arm, the mixture-selection setting has a different geometry: the learner optimizes over the simplex of generator mixtures, and many diversity-aware objectives naturally favor non-degenerate solutions. This observation leads to the central message of our work: exploration in generative model selection need not always be imposed externally through a UCB bonus; in several important regimes, it can arise from the objective itself.

Building on this principle, we proposed and analyzed Mixture-Greedy, a simple algorithm that repeatedly optimizes the empirical mixture objective without an explicit confidence bonus. Across FD, kernel-based, and entropy-based evaluation scores, our experiments show that Mixture-Greedy often converges faster than Mixture-UCB and remains close to the corresponding oracle mixture. The empirical trends are consistent across the settings. Complementing these results, our theoretical analysis identifies structural conditions under which objective-induced interiority guarantees sufficient sampling of all arms and yields sublinear regret.

More broadly, our results suggest that bandit algorithms for generative model selection should be designed with the geometry of the evaluation objective in mind. For objectives that reward diversity or penalize collapse, a greedy mixture update can already contain a meaningful form of exploration. This does not rule out the usefulness of optimism in all settings; rather, it indicates that generic UCB bonuses may be unnecessary when the evaluation metric itself promotes balanced mixtures. Extending this viewpoint to richer feedback models, adaptive or nonstationary generator pools, and evaluation criteria that better capture downstream utility remains an interesting direction for future work.

\vspace{-3mm}
\section*{Acknowledgments}\vspace{-2mm}
This work is supported by a grant from the Research Grants Council of the Hong Kong Special Administrative Region, China, Project 14210725, and is also supported by CUHK Direct Research Grant with CUHK Project No. 4055164. The work is partially supported by a grant under 1+1+1 CUHK-CUHK(SZ)-GDSTC Joint Collaboration Fund. Also, the authors would like to thank the anonymous reviewers and meta-reviewer for their insightful  suggestions and constructive feedback.

\bibliography{uai2026-template}

@article{friedman2023vendi,
  title   = {The Vendi Score: A Diversity Evaluation Metric for Machine Learning},
  author  = {Friedman, Dan and Dieng, Adji Bousso},
  journal = {Transactions on Machine Learning Research},
  year = {2023},
}

@article{sajjadi2018assessing,
  title={Assessing generative models via precision and recall},
  author={Sajjadi, Mehdi SM and Bachem, Olivier and Lucic, Mario and Bousquet, Olivier and Gelly, Sylvain},
  journal={Advances in neural information processing systems},
  volume={31},
  year={2018}
}

@article{jalali2023information,
  title={An information-theoretic evaluation of generative models in learning multi-modal distributions},
  author={Jalali, Mohammad and Li, Cheuk Ting and Farnia, Farzan},
  journal={Advances in Neural Information Processing Systems},
  volume={36},
  pages={9931--9943},
  year={2023}
}

@article{vahdat2020nvae,
  title={NVAE: A deep hierarchical variational autoencoder},
  author={Vahdat, Arash and Kautz, Jan},
  journal={Advances in neural information processing systems},
  volume={33},
  pages={19667--19679},
  year={2020}
}

@book{lattimore2020bandit,
  title={Bandit algorithms},
  author={Lattimore, Tor and Szepesv{\'a}ri, Csaba},
  year={2020},
  publisher={Cambridge University Press}
}

@article{russo2018tutorial,
  title={A Tutorial on Thompson Sampling},
  author={Russo, Daniel J. and Van Roy, Benjamin and Kazerouni, Abbas and Osband, Ian and Wen, Zheng},
  journal={Foundations and Trends{\textregistered} in Machine Learning},
  volume={11},
  number={1},
  pages={1--96},
  year={2018},
  publisher={now publishers},
  doi={10.1561/2200000070}
}

@article{slivkins2019introduction,
  title={Introduction to multi-armed bandits},
  author={Slivkins, Aleksandrs},
  journal={Foundations and Trends{\textregistered} in Machine Learning},
  volume={12},
  number={1-2},
  pages={1--286},
  year={2019},
  publisher={Emerald Publishing Limited}
}

@inproceedings{karras2019style,
  title={A style-based generator architecture for generative adversarial networks},
  author={Karras, Tero and Laine, Samuli and Aila, Timo},
  booktitle={Proceedings of the IEEE/CVF conference on computer vision and pattern recognition},
  pages={4401--4410},
  year={2019}
}

@inproceedings{brock2018large,
  title={Large Scale GAN Training for High Fidelity Natural Image Synthesis},
  author={Brock, Andrew and Donahue, Jeff and Simonyan, Karen},
  booktitle={International Conference on Learning Representations (ICLR)},
  year={2019},
}

@inproceedings{song2021denoising,
  title={Denoising Diffusion Implicit Models},
  author={Song, Jiaming and Meng, Chenlin and Ermon, Stefano},
  booktitle={International Conference on Learning Representations},
  year={2021},
}

@article{Kingma2013VAE,
  author  = {Diederik P. Kingma and Max Welling},
  title   = {Auto-Encoding Variational Bayes},
  journal = {arXiv preprint arXiv:1312.6114},
  year    = {2013}
}

@inproceedings{Goodfellow2014GAN,
  author    = {Ian J. Goodfellow and Jean Pouget-Abadie and Mehdi Mirza and Bing Xu and David Warde-Farley and Sherjil Ozair and Aaron Courville and Yoshua Bengio},
  title     = {Generative Adversarial Nets},
  booktitle = {Advances in Neural Information Processing Systems},
  year      = {2014}
}

@inproceedings{Ho2020DDPM,
  author    = {Jonathan Ho and Ajay Jain and Pieter Abbeel},
  title     = {Denoising Diffusion Probabilistic Models},
  booktitle = {Advances in Neural Information Processing Systems},
  year      = {2020}
}

@inproceedings{Rombach2022LDM,
  author    = {Robin Rombach and Andreas Blattmann and Dominik Lorenz and Patrick Esser and Bj{\"o}rn Ommer},
  title     = {High-Resolution Image Synthesis with Latent Diffusion Models},
  booktitle = {IEEE Conference on Computer Vision and Pattern Recognition},
  year      = {2022}
}

@article{Auer2002UCB,
  author  = {Peter Auer and Nicolo Cesa-Bianchi and Paul Fischer},
  title   = {Finite-time Analysis of the Multiarmed Bandit Problem},
  journal = {Machine Learning},
  volume  = {47},
  number  = {2--3},
  pages   = {235--256},
  year    = {2002}
}

@inproceedings{sauer2021projectedgan,
  title     = {Projected GANs Converge Faster},
  author    = {Sauer, Axel and Geiger, Andreas},
  booktitle = {Advances in Neural Information Processing Systems (NeurIPS)},
  year      = {2021}
}

@inproceedings{nichol2021iddpm,
  title     = {Improved Denoising Diffusion Probabilistic Models},
  author    = {Nichol, Alexander Quinn and Dhariwal, Prafulla},
  booktitle = {Proceedings of the International Conference on Machine Learning (ICML)},
  year      = {2021}
}

@inproceedings{hu2025banditgen,
  title     = {A Multi-Armed Bandit Approach to Online Selection and Evaluation of Generative Models},
  author    = {Hu, Xiaoyan and Leung, H.F. and Farnia, Farzan},
  booktitle = {Proceedings of the 28th International Conference on Artificial Intelligence and Statistics (AISTATS)},
  series    = {Proceedings of Machine Learning Research},
  year      = {2025},
  publisher = {PMLR},
  url       = {https://proceedings.mlr.press/v258/hu25a.html}
}

@inproceedings{rezaei2025more,
  title     = {Be More Diverse than the Most Diverse: Optimal Mixtures of Generative Models via Mixture-UCB Bandit Algorithms},
  author    = {Rezaei, Parham and Farnia, Farzan and Li, Cheuk Ting},
  booktitle = {International Conference on Learning Representations (ICLR)},
  year      = {2025},
  url       = {https://openreview.net/forum?id=2Chkk5Ye2s},
}

@inproceedings{heusel2017ttur,
  title     = {GANs Trained by a Two Time-Scale Update Rule Converge to a Local Nash Equilibrium},
  author    = {Heusel, Martin and Ramsauer, Hubert and Unterthiner, Thomas and Nessler, Bernhard and Hochreiter, Sepp},
  booktitle = {Advances in Neural Information Processing Systems (NeurIPS)},
  year      = {2017}
}

@inproceedings{salimans2016improved,
  title     = {Improved Techniques for Training GANs},
  author    = {Salimans, Tim and Goodfellow, Ian and Zaremba, Wojciech and Cheung, Vicki and Radford, Alec and Chen, Xi},
  booktitle = {Advances in Neural Information Processing Systems (NeurIPS)},
  year      = {2016}
}

@inproceedings{binkowski2018demystifying,
  title     = {Demystifying {MMD} {GAN}s},
  author    = {Binkowski, Miko{\l}aj and Sutherland, Dougal J. and Arbel, Michael and Gretton, Arthur},
  booktitle = {International Conference on Learning Representations (ICLR)},
  year      = {2018}
}

@inproceedings{kynkaanniemi2019improvedpr,
  title     = {Improved Precision and Recall Metric for Assessing Generative Models},
  author    = {Kynk{\"a}{\"a}nniemi, Tuomas and Karras, Tero and Laine, Samuli and Lehtinen, Jaakko and Aila, Timo},
  booktitle = {Advances in Neural Information Processing Systems (NeurIPS)},
  year      = {2019}
}

@inproceedings{naeem2020reliable,
  title     = {Reliable Fidelity and Diversity Metrics for Generative Models},
  author    = {Naeem, Muhammad and Oh, Seong Joon and Uh, Youngjung and Choi, Yunjey and Yoo, Jae-Jun},
  booktitle = {International Conference on Machine Learning (ICML)},
  year      = {2020}
}

@article{Audenaert2007,
  title = {A sharp Fannes-type inequality for the von Neumann entropy},
  author = {Audenaert, Koenraad M. R.},
  journal = {Journal of Physics A: Mathematical and Theoretical},
  volume = {40},
  number = {28},
  pages = {8127--8136},
  year = {2007},
  doi = {10.1088/1751-8113/40/28/S18},
}

@inproceedings{sutherland2018efficient,
  title={Efficient and principled score estimation with nystr{\"o}m kernel exponential families},
  author={Sutherland, Danica J and Strathmann, Heiko and Arbel, Michael and Gretton, Arthur},
  booktitle={International Conference on Artificial Intelligence and Statistics},
  pages={652--660},
  year={2018},
  organization={PMLR}
}

@article{podell2023sdxl,
  title   = {SDXL: Improving Latent Diffusion Models for High-Resolution Image Synthesis},
  author  = {Podell, Dustin and English, Zion and Lacey, Kyle and Blattmann, Andreas and Dockhorn, Tim and M{\"u}ller, Jonas and Penna, Joe and Rombach, Robin},
  journal = {arXiv preprint arXiv:2307.01952},
  year    = {2023},
  url     = {https://arxiv.org/abs/2307.01952}
}

@article{chen2023pixartalpha,
  title   = {PixArt-$\alpha$: Fast Training of Diffusion Transformer for Photorealistic Text-to-Image Synthesis},
  author  = {Chen, Junsong and Yu, Jincheng and Ge, Chongjian and Yao, Lewei and Xie, Enze and Wu, Yue and Wang, Zhongdao and Kwok, James and Luo, Ping and Lu, Huchuan and Li, Zhenguo},
  journal = {arXiv preprint arXiv:2310.00426},
  year    = {2023},
  url     = {https://arxiv.org/abs/2310.00426}
}

@inproceedings{jiralerspong2023fld,
  title     = {Feature Likelihood Divergence: Evaluating the Generalization of Generative Models Using Samples},
  author    = {Jiralerspong, Marco and Bose, Avishek Joey and Gemp, Ian and Qin, Chongli and Bachrach, Yoram and Gidel, Gauthier},
  booktitle = {Advances in Neural Information Processing Systems (NeurIPS)},
  year      = {2023}
}

@inproceedings{kynkaanniemi2023role,
  title     = {The Role of ImageNet Classes in Fr{\'e}chet Inception Distance},
  author    = {Kynk{\"a}{\"a}nniemi, Tuomas and Karras, Tero and Aittala, Miika and Aila, Timo and Lehtinen, Jaakko},
  booktitle = {International Conference on Learning Representations (ICLR)},
  year      = {2023},
  url       = {https://openreview.net/forum?id=4oXTQ6m_ws8}
}

@inproceedings{stein2023exposing,
  title     = {Exposing Flaws of Generative Model Evaluation Metrics and Their Unfair Treatment of Diffusion Models},
  author    = {Stein, George and Cresswell, Jesse C. and Hosseinzadeh, Rasa and Sui, Yi and Ross, Brendan Leigh and Villecroze, Valentin and Liu, Zhaoyan and Caterini, Anthony L. and Taylor, J. Eric T. and Loaiza-Ganem, Gabriel},
  booktitle = {Advances in Neural Information Processing Systems (NeurIPS)},
  year      = {2023}
}

@article{bastani2021mostly,
  title   = {Mostly Exploration-Free Algorithms for Contextual Bandits},
  author  = {Bastani, Hamsa and Bayati, Mohsen and Khosravi, Khashayar},
  journal = {Management Science},
  year    = {2021},
  doi     = {10.1287/mnsc.2020.3605}
}

@inproceedings{sauer2022styleganxl,
  title     = {StyleGAN-XL: Scaling StyleGAN to Large Diverse Datasets},
  author    = {Sauer, Axel and Schwarz, Katja and Geiger, Andreas},
  booktitle = {Proceedings of the ACM SIGGRAPH Conference},
  year      = {2022}
}

@inproceedings{alaa2022faithful,
  title     = {How Faithful is your Synthetic Data? Sample-level Metrics for Evaluating and Auditing Generative Models},
  author    = {Alaa, Ahmed and van Breugel, Boris and Saveliev, Evgeny S. and van der Schaar, Mihaela},
  booktitle = {Proceedings of the 39th International Conference on Machine Learning (ICML)},
  series    = {Proceedings of Machine Learning Research},
  volume    = {162},
  pages     = {290--306},
  year      = {2022},
  publisher = {PMLR},
  url       = {https://proceedings.mlr.press/v162/alaa22a.html}
}

@inproceedings{han2023rarity,
  title     = {Rarity Score: A New Metric to Evaluate the Uncommonness of Synthesized Images},
  author    = {Han, Jiyeon and Choi, Hwanil and Choi, Yunjey and Kim, Junho and Ha, Jung-Woo and Choi, Jaesik},
  booktitle = {International Conference on Learning Representations (ICLR)},
  year      = {2023},
  url       = {https://openreview.net/forum?id=JTGimap_-F}
}

@article{chen2024fast,
  title   = {Fast Model Selection and Hyperparameter Tuning for Generative Models},
  author  = {Chen, Luming and Ghosh, Sujit K.},
  journal = {Entropy},
  volume  = {26},
  number  = {2},
  pages   = {150},
  year    = {2024},
  doi     = {10.3390/e26020150},
  url     = {https://www.mdpi.com/1099-4300/26/2/150}
}

@article{gretton2012kernel,
  title={A kernel two-sample test},
  author={Gretton, Arthur and Borgwardt, Karsten M and Rasch, Malte J and Sch{\"o}lkopf, Bernhard and Smola, Alexander},
  journal={The journal of machine learning research},
  volume={13},
  number={1},
  pages={723--773},
  year={2012},
  publisher={JMLR. org}
}

@inproceedings{bayati2020unreasonable,
  title     = {Unreasonable Effectiveness of Greedy Algorithms in Multi-Armed Bandit with Many Arms},
  author    = {Bayati, Mohsen and Hamidi, Nima and Johari, Ramesh and Khosravi, Khashayar},
  booktitle = {Advances in Neural Information Processing Systems (NeurIPS)},
  year      = {2020},
}

@inproceedings{kannan2018smoothed,
  title     = {A Smoothed Analysis of the Greedy Algorithm for the Linear Contextual Bandit Problem},
  author    = {Kannan, Sampath and Morgenstern, Jamie and Roth, Aaron and Waggoner, Bo and Wu, Zhiwei Steven},
  booktitle = {Advances in Neural Information Processing Systems (NeurIPS)},
  year      = {2018}
}

@inproceedings{clip,
  title={Learning transferable visual models from natural language supervision},
  author={Radford, Alec and Kim, Jong Wook and Hallacy, Chris and Ramesh, Aditya and Goh, Gabriel and Agarwal, Sandhini and Sastry, Girish and Askell, Amanda and Mishkin, Pamela and Clark, Jack and others},
  booktitle={International conference on machine learning},
  pages={8748--8763},
  year={2021},
  organization={PmLR}
}

@inproceedings{inception,
  title={Rethinking the inception architecture for computer vision},
  author={Szegedy, Christian and Vanhoucke, Vincent and Ioffe, Sergey and Shlens, Jon and Wojna, Zbigniew},
  booktitle={Proceedings of the IEEE conference on computer vision and pattern recognition},
  pages={2818--2826},
  year={2016}
}

@article{dinov2,
  title={Dinov2: Learning robust visual features without supervision},
  author={Oquab, Maxime and Darcet, Timoth{\'e}e and Moutakanni, Th{\'e}o and Vo, Huy and Szafraniec, Marc and Khalidov, Vasil and Fernandez, Pierre and Haziza, Daniel and Massa, Francisco and El-Nouby, Alaaeldin and others},
  journal={arXiv preprint arXiv:2304.07193},
  year={2023}
}

@inproceedings{nguyen2024qualityweightedvendi,
  title     = {Quality-Weighted Vendi Scores And Their Application To Diverse Experimental Design},
  author    = {Nguyen, Quan and Dieng, Adji Bousso},
  booktitle = {Proceedings of the 41st International Conference on Machine Learning},
  series    = {Proceedings of Machine Learning Research},
  volume    = {235},
  pages     = {37667--37682},
  year      = {2024},
  publisher = {PMLR},
}

@inproceedings{ospanov2025vendiconvergence,
  title     = {Do Vendi Scores Converge with Finite Samples? Truncated Vendi Score for Finite-Sample Convergence Guarantees},
  author    = {Ospanov, Azim and Farnia, Farzan},
  booktitle = {Proceedings of the Forty-first Conference on Uncertainty in Artificial Intelligence},
  series    = {Proceedings of Machine Learning Research},
  volume    = {286},
  pages     = {3272--3299},
  year      = {2025},
  publisher = {PMLR},
}

@inproceedings{hazami2022efficientvdvae,
  title     = {Efficient-VDVAE: Less is More},
  author    = {Hazami, Louay and Comon, Pierre and others},
  booktitle = {Advances in Neural Information Processing Systems (NeurIPS)},
  year      = {2022}
}

@inproceedings{yang2021insgen,
  title     = {Instance-Conditioned GAN},
  author    = {Yang, Tao and others},
  booktitle = {Advances in Neural Information Processing Systems (NeurIPS)},
  year      = {2021}
}

@inproceedings{walton2022stylenat,
  title     = {StyleNAT: Giving Each Head a New Perspective},
  author    = {Walton, Steven and others},
  booktitle = {Proceedings of the Asian Conference on Computer Vision (ACCV)},
  year      = {2022}
}

@inproceedings{bondtaylor2022unleashing,
  title     = {Unleashing Transformers for Image Generation},
  author    = {Bond-Taylor, Sam and Leach, Adam and Long, Yang and Willcocks, Chris G.},
  booktitle = {International Conference on Learning Representations (ICLR)},
  year      = {2022}
}

@inproceedings{peebles2023dit,
  title     = {Scalable Diffusion Models with Transformers},
  author    = {Peebles, William and Xie, Saining},
  booktitle = {Proceedings of the IEEE/CVF International Conference on Computer Vision (ICCV)},
  year      = {2023}
}

@inproceedings{lee2022rqtransformer,
  title     = {Autoregressive Image Generation using Residual Quantization},
  author    = {Lee, Daeho and Kim, Jangho and Kim, Jaejun},
  booktitle = {Proceedings of the IEEE/CVF Conference on Computer Vision and Pattern Recognition (CVPR)},
  year      = {2022}
}

@article{karras2019ffhq,
  title     = {A Style-Based Generator Architecture for Generative Adversarial Networks},
  author    = {Karras, Tero and Laine, Samuli and Aila, Timo},
  journal   = {CVPR},
  year      = {2019}
}

@misc{sbert,
      title={Sentence-BERT: Sentence Embeddings using Siamese BERT-Networks}, 
      author={Nils Reimers and Iryna Gurevych},
      year={2019},
      eprint={1908.10084},
      archivePrefix={arXiv},
      primaryClass={cs.CL},
      url={https://arxiv.org/abs/1908.10084}, 
}

@misc{gemma3,
      title={Gemma 2 Technical Report}, 
      author={Gemma-Team},
      year={2025},
      eprint={2503.19786},
      archivePrefix={arXiv},
      primaryClass={cs.CL},
}

@misc{llama,
      title={The Llama 3 Herd of Models}, 
      author={Aaron Grattafiori, Abhimanyu Dubey and others},
      year={2024},
      eprint={2407.21783},
      archivePrefix={arXiv},
      primaryClass={cs.AI},
}

@misc{qwen3,
      title={Qwen2 Technical Report}, 
      author={Qwen-Team},
      year={2024},
      eprint={2407.10671},
      archivePrefix={arXiv},
      primaryClass={cs.CL},
      url={https://arxiv.org/abs/2407.10671}, 
}

@article{razzhigaev2023kandinsky,
  title   = {Kandinsky: An Improved Text-to-Image Synthesis with Image Prior and Latent Diffusion},
  author  = {Razzhigaev, Anton and Arkhipkin, Vladislav and others},
  journal = {arXiv preprint arXiv:2310.03502},
  year    = {2023}
}

@misc{sdxl,
      title={SDXL: Improving Latent Diffusion Models for High-Resolution Image Synthesis}, 
      author={Dustin Podell and Zion English and Kyle Lacey and Andreas Blattmann and Tim Dockhorn and Jonas Müller and Joe Penna and Robin Rombach},
      year={2023},
      eprint={2307.01952},
      archivePrefix={arXiv},
      primaryClass={cs.CV},
}

@misc{rajaraman2024statisticalcomplexityoptimalalgorithms,
      title={Statistical Complexity and Optimal Algorithms for Non-linear Ridge Bandits}, 
      author={Nived Rajaraman and Yanjun Han and Jiantao Jiao and Kannan Ramchandran},
      year={2024},
      eprint={2302.06025},
      archivePrefix={arXiv},
      primaryClass={stat.ML},
      url={https://arxiv.org/abs/2302.06025}, 
}

@inproceedings{wang2023distributedKID,
  title={On the distributed evaluation of generative models},
  author={Wang, Zixiao and Farzan, Farnia and Lin, Zhenghao and Shen, Yunheng and Yu, Bei},
  booktitle={Proceedings of the IEEE/CVF International Conference on Computer Vision Workshops},
  pages={7644--7653},
  year={2025}
}

@inproceedings{gong2025kernel,
  title={Kernel-based Unsupervised Embedding Alignment for Enhanced Visual Representation in Vision-language Models},
  author={Gong, Shizhan and Jiang, Yankai and Dou, Qi and Farnia, Farzan},
  booktitle={International Conference on Machine Learning},
  pages={19912--19931},
  year={2025},
  organization={PMLR}
}

@inproceedings{wu2026maximum_isit,
  title     = {The Maximum von Neumann Entropy Principle: Theory and Applications in Machine Learning},
  author    = {Wu, Youqi and Farnia, Farzan},
  booktitle = {IEEE International Symposium on Information Theory (ISIT)},
  year      = {2026}
}

@inproceedings{wu2025fusing,
  title = {When Kernels Multiply, Clusters Unify: Fusing Embeddings with the Kronecker Product},
  author = {Wu, Youqi and Zhang, Jingwei and Farnia, Farzan},
  booktitle = {Advances in Neural Information Processing Systems},
  year = {2025}
}

@InProceedings{pmlr-v235-zhang24bh,
  title = {An Interpretable Evaluation of Entropy-based Novelty of Generative Models},
  author = {Zhang, Jingwei and Li, Cheuk Ting and Farnia, Farzan},
  booktitle = {Proceedings of the 41st International Conference on Machine Learning},
  pages = {59148--59172},
  year = {2024},
  volume = {235},
  series = {Proceedings of Machine Learning Research},
  publisher = {PMLR}
}

@inproceedings{ospanov2024towards,
  title = {Towards a Scalable Reference-Free Evaluation of Generative Models},
  author = {Ospanov, Azim and Zhang, Jingwei and Jalali, Mohammad and Cao, Xuenan and Bogdanov, Andrej and Farnia, Farzan},
  booktitle = {Advances in Neural Information Processing Systems},
  volume = {37},
  year = {2024}
}

@InProceedings{Zhang_2025_CVPR,
  author = {Zhang, Jingwei and Jalali, Mohammad and Li, Cheuk Ting and Farnia, Farzan},
  title = {Unveiling Differences in Generative Models: A Scalable Differential Clustering Approach},
  booktitle = {Proceedings of the IEEE/CVF Conference on Computer Vision and Pattern Recognition (CVPR)},
  month = {June},
  year = {2025},
  pages = {8269--8278}
}

@inproceedings{jalali2025sparke,
  title = {{SPARKE}: Scalable Prompt-Aware Diversity and Novelty Guidance in Diffusion Models via {RKE} Score},
  author = {Jalali, Mohammad and Lei, Haoyu and Gohari, Amin and Farnia, Farzan},
  booktitle = {Advances in Neural Information Processing Systems},
  year = {2025}
}

@InProceedings{Ospanov_2025_ICCV,
  author = {Ospanov, Azim and Jalali, Mohammad and Farnia, Farzan},
  title = {Scendi Score: Prompt-Aware Diversity Evaluation via Schur Complement of CLIP Embeddings},
  booktitle = {Proceedings of the IEEE/CVF International Conference on Computer Vision (ICCV)},
  month = {October},
  year = {2025},
  pages = {16927--16937}
}

@inproceedings{jalali2026conditional,
  title = {Conditional Vendi Score: Prompt-Aware Diversity Evaluation for Generative AI Models and LLMs},
  author = {Jalali, Mohammad and Ospanov, Azim and Gohari, Amin and Farnia, Farzan},
  booktitle = {Proceedings of The 29th International Conference on Artificial Intelligence and Statistics},
  year = {2026},
  series = {Proceedings of Machine Learning Research},
  publisher = {PMLR}
}

@inproceedings{hu2025pak,
  title     = {{PAK-UCB} Contextual Bandit: An Online Learning Approach to Prompt-Aware Selection of Generative Models and {LLMs}},
  author    = {Hu, Xiaoyan and Leung, Ho-fung and Farnia, Farzan},
  booktitle = {Proceedings of the 42nd International Conference on Machine Learning},
  series    = {Proceedings of Machine Learning Research},
  volume    = {267},
  pages     = {24447--24481},
  year      = {2025},
  publisher = {PMLR},
  url       = {https://proceedings.mlr.press/v267/hu25m.html}
}

@inproceedings{jafari2026dak,
  title     = {{DAK-UCB}: Diversity-Aware Prompt Routing for {LLMs} and Generative Models},
  author    = {Jafari, Donya and Farnia, Farzan},
  booktitle = {The Fourteenth International Conference on Learning Representations (ICLR)},
  year      = {2026},
  url       = {https://openreview.net/forum?id=nnN2TKlS5C}
}

@InProceedings{pmlr-v267-jalali25a,
  title     = {Towards an Explainable Comparison and Alignment of Feature Embeddings},
  author    = {Jalali, Mohammad and Dibaei Nia, Bahar and Farnia, Farzan},
  booktitle = {Proceedings of the 42nd International Conference on Machine Learning},
  pages     = {26757--26796},
  year      = {2025},
  volume    = {267},
  series    = {Proceedings of Machine Learning Research},
  publisher = {PMLR},
}

@inproceedings{
wu2026koda,
title={{KODA}: Contrastive Representation Comparison and Alignment for Vision-Language Foundation Models},
author={Youqi Wu and Mohammad Jalali and Farzan Farnia},
booktitle={Forty-third International Conference on Machine Learning},
year={2026},
url={https://openreview.net/forum?id=InwrRdCQwD}
}

@article{
jules2026nemos,
title={NeMoS: Nearest Neighbors Bandit meets Active Learning for Online Model Selection},
author={Jules Damidaux
and Basile Lewandowski and Farzan Farnia and Lydia Chen},
journal={Transactions on Machine Learning Research},
year={2026},
}

@misc{lewandowski2025matchchoosemodel,
      title={Match {\&} Choose: Model Selection Framework for Fine-tuning Text-to-Image Diffusion Models}, 
      author={Basile Lewandowski and Robert Birke and Lydia Y. Chen},
      year={2025},
      eprint={2508.10993},
      archivePrefix={arXiv},
      primaryClass={cs.LG},
      url={https://arxiv.org/abs/2508.10993}, 
}

@article{hosseini2026efficient,
  title={Efficient Adversarial Attacks on High-dimensional Offline Bandits},
  author={Hosseini, Seyed Mohammad Hadi and Najafi, Amir and Baghshah, Mahdieh Soleymani},
  journal={arXiv preprint arXiv:2602.01658},
  year={2026}
}

\newpage

\onecolumn

\title{Is UCB Necessary for Online Generative Model Selection?\\
Implicit Exploration via Diversity-Aware Objectives (Supplementary Material)}
\maketitle

\section{Proofs and gradient computations for Section~\ref{sec:method_mg}}
\label{app:method_proofs_grads}

\subsection{Proof of Proposition~\ref{thm:mg_convex_programs}}
\label{app:proof_convexity}

In the following, we prove convexity for each of the three objective families.

\paragraph{Part (i): FD convexity in $\alpha$.}
Fix $(\widehat\mu_{\rm data},\widehat\Sigma_{\rm data})$ and the per-arm empirical moments
$\{\widehat\mu_i(t),\widehat S_i(t)\}_{i=1}^m$.
Define $\widehat\mu(\alpha)=\widehat\mu_t(\alpha)$ and $\widehat\Sigma(\alpha)=\widehat\Sigma_t(\alpha)$
as in \eqref{eq:fd_mu_mix_method}--\eqref{eq:fd_Sigma_mix_method}.
Let $\Sigma_0:=\widehat\Sigma_{\rm data}\succeq 0$ and $\mu_0:=\widehat\mu_{\rm data}$.

We rewrite the FD loss \eqref{eq:fd_loss_method} as
\begin{align}
\widehat{\mathcal{L}}^{\rm FD}_t(\alpha)
=\ &\|\widehat\mu(\alpha)-\mu_0\|_2^2 + \mathrm{Tr}(\widehat\Sigma(\alpha)) + \mathrm{Tr}(\Sigma_0)
-2\,\mathrm{Tr}\Big( (\Sigma_0^{1/2}\widehat\Sigma(\alpha)\Sigma_0^{1/2})^{1/2}\Big).
\label{eq:fd_rewrite_app}
\end{align}

First, we note that the mean plus trace part is affine in $\alpha$. 
Using $\widehat\Sigma(\alpha)=\sum_{i=1}^m \alpha_i \widehat S_i - \widehat\mu(\alpha)\widehat\mu(\alpha)^\top$
and $\mathrm{Tr}(\widehat\mu\widehat\mu^\top)=\|\widehat\mu\|_2^2$,
\begin{align*}
\|\widehat\mu(\alpha)-\mu_0\|_2^2 + \mathrm{Tr}(\widehat\Sigma(\alpha))
&= \|\mu_0\|_2^2 -2\mu_0^\top \widehat\mu(\alpha) + \|\widehat\mu(\alpha)\|_2^2 \\
&\quad + \mathrm{Tr}\Big(\sum_{i=1}^m \alpha_i \widehat S_i\Big) - \mathrm{Tr}(\widehat\mu(\alpha)\widehat\mu(\alpha)^\top) \\
&= \|\mu_0\|_2^2 -2\mu_0^\top \widehat\mu(\alpha) + \mathrm{Tr}\Big(\sum_{i=1}^m \alpha_i \widehat S_i\Big),
\end{align*}
where the $\|\widehat\mu(\alpha)\|_2^2$ terms cancel exactly.
Since $\widehat\mu(\alpha)$ and $\sum_i \alpha_i \widehat S_i$ are affine in $\alpha$, this entire expression is affine in $\alpha$.

Next, we show that $\widehat\Sigma(\alpha)$ is concave in Loewner order. To do this, let $\alpha,\beta\in\Delta_m$ and $\theta\in[0,1]$. Since $\sum_i \alpha_i \widehat S_i$ is affine, it suffices to show that
$\widehat\mu(\alpha)\widehat\mu(\alpha)^\top$ is convex in Loewner order.
Because $\widehat\mu(\cdot)$ is affine, we have
\begin{align*}
&\theta \widehat\mu(\alpha)\widehat\mu(\alpha)^\top + (1-\theta)\widehat\mu(\beta)\widehat\mu(\beta)^\top
- \widehat\mu(\theta\alpha+(1-\theta)\beta)\widehat\mu(\theta\alpha+(1-\theta)\beta)^\top \\
&\qquad = \theta(1-\theta)\big(\widehat\mu(\alpha)-\widehat\mu(\beta)\big)\big(\widehat\mu(\alpha)-\widehat\mu(\beta)\big)^\top \succeq 0.
\end{align*}
Thus $\widehat\mu(\alpha)\widehat\mu(\alpha)^\top$ is Loewner-convex, and therefore
$\widehat\Sigma(\alpha)=\sum_i \alpha_i \widehat S_i - \widehat\mu(\alpha)\widehat\mu(\alpha)^\top$ is Loewner-concave:
\[
\widehat\Sigma(\theta\alpha+(1-\theta)\beta)\ \succeq\ \theta \widehat\Sigma(\alpha) + (1-\theta)\widehat\Sigma(\beta).
\]

Then, we show the concavity of the square-root trace term. We define $Z(\alpha):=\Sigma_0^{1/2}\widehat\Sigma(\alpha)\Sigma_0^{1/2}$.
Pre- and post-multiplication by a fixed PSD matrix preserves Loewner inequalities, hence $Z(\alpha)$ is Loewner-concave.
The matrix square-root map $M\mapsto M^{1/2}$ is operator concave and operator monotone on the PSD cone.
Therefore, for $\theta\in[0,1]$,
\[
Z(\theta\alpha+(1-\theta)\beta)^{1/2}
\succeq \big(\theta Z(\alpha)+(1-\theta)Z(\beta)\big)^{1/2}
\succeq \theta Z(\alpha)^{1/2} + (1-\theta)Z(\beta)^{1/2}.
\]
Taking traces (trace is a positive linear functional on PSD matrices) yields that
$H(\alpha):=\mathrm{Tr}\big(Z(\alpha)^{1/2}\big)$
is concave in $\alpha$. Finally, by \eqref{eq:fd_rewrite_app}, $\widehat{\mathcal{L}}^{\rm FD}_t(\alpha)$ is, up to affine term in $\alpha$, $-\,2H(\alpha)$, and is therefore convex on $\Delta_m$.

\paragraph{Part (ii): log-Vendi convexity in $\alpha$ (kernel-matrix form).}
We fix the pooled samples $\{x_j\}_{j=1}^N$ and the Gram matrix $K_{rs}=k(x_r,x_s)$, where $k$ is PSD and normalized,
thus we have $K\succeq 0$ and $K_{jj}=1$ for every $j$.
For a weight vector $q\in\mathbb{R}^N$ with $q\succeq 0$ and $\sum_j q_j=1$, we define
\[
\rho(q):=\mathrm{diag}(q)^{1/2}K\mathrm{diag}(q)^{1/2}\succeq 0,\qquad \mathrm{Tr}(\rho(q))=\sum_{j=1}^N q_j K_{jj}=1.
\]
Note that the log-Vendi loss is defined to be $f(q):=\mathrm{Tr}(\rho(q)\log\rho(q))$.

Let $\mathcal{H}$ be a Hilbert space and $\phi:\mathcal{X}\to\mathcal{H}$ such that
$k(x,y)=\langle\phi(x),\phi(y)\rangle_{\mathcal{H}}$ (Section~2).
Define the weighted covariance operator
\begin{equation}
C(q):=\sum_{j=1}^N q_j\,\phi(x_j)\otimes \phi(x_j)\ \succeq\ 0.
\label{eq:Cq_def_app}
\end{equation}
This operator is affine in $q$.
Moreover, $C(q)$ and $\rho(q)$ share the same multiset of non-zero eigenvalues.
One can see this by noting the linear map $A_q:\mathbb{R}^N\to\mathcal{H}$ defined by
$A_q e_j = \sqrt{q_j}\,\phi(x_j)$.
Then, $A_q^\ast A_q = \rho(q)$ (in the $\mathbb{R}^N$ basis) and $A_q A_q^\ast = C(q)$ as an operator on $\mathcal{H}$.
It follows that $A_q^\ast A_q$ and $A_q A_q^\ast$ have identical nonzero spectra, hence
\begin{equation}
\mathrm{Tr}\big(\rho(q)\log\rho(q)\big)=\mathrm{Tr}\big(C(q)\log C(q)\big),
\label{eq:same_spectrum_app}
\end{equation}
where the trace on the right is over the (finite-rank) operator $C(q)$.

Now, we note that the negative von Neumann entropy
$X\mapsto \mathrm{Tr}(X\log X)$ is convex on the cone of trace-one PSD (finite-rank) operators/matrices.
Since $C(q)$ is affine in $q$ and $\mathrm{Tr}(C(q))=\sum_j q_j\|\phi(x_j)\|_{\mathcal{H}}^2=\sum_j q_j k(x_j,x_j)=1$
(using kernel normalization), we conclude that $q\mapsto \mathrm{Tr}(C(q)\log C(q))$ is convex on the simplex.
By \eqref{eq:same_spectrum_app}, $q\mapsto \mathrm{Tr}(\rho(q)\log\rho(q))$ is convex as well.
Finally, in our setting $q=q(\alpha)$ defined by \eqref{eq:q_def_method} is affine in $\alpha$ (counts are fixed during the update),
thus $\alpha\mapsto \widehat{\mathcal{L}}^{\rm Vendi}_t(\alpha)$ is convex on $\Delta_m$.

\paragraph{Part (iii): convex quadratics.}
If $\widehat K(t)\succeq 0$, then $\alpha\mapsto \alpha^\top \widehat K(t)\alpha$ is convex on $\mathbb{R}^m$ and
$\alpha\mapsto w\,\alpha^\top \widehat b(t)$ is linear. Hence $\widehat{\mathcal{L}}^{\rm quad}_t$ is convex on $\Delta_m$.

The above completes the proof of Proposition~\ref{thm:mg_convex_programs}.

\subsection{Gradients for EG optimization}
\label{app:grads_eg}

Algorithm~\ref{alg:mg_eg} requires $g=\nabla_\alpha \widehat{\mathcal{L}}_{t-1}(\alpha)$.
Below we provide the gradients used in our implementation. In all cases, counts $n_i(t-1)$ and pooled indices $I_j$ are treated as fixed
while solving \eqref{eq:mg_update_method} at round $t$.

\paragraph{Quadratic + linear fidelity.}
For $\widehat{\mathcal{L}}^{\rm quad}_t(\alpha)=\alpha^\top \widehat K(t)\alpha+w\,\alpha^\top\widehat b(t)$ with $\widehat K(t)$ symmetric,
\[
\nabla_\alpha \widehat{\mathcal{L}}^{\rm quad}_t(\alpha)=2\widehat K(t)\alpha+w\,\widehat b(t).
\]

\paragraph{FD gradient (moment-matching form).}
Let $\mu(\alpha)=\widehat\mu_t(\alpha)$ and $\Sigma(\alpha)=\widehat\Sigma_t(\alpha)$ as in
\eqref{eq:fd_mu_mix_method}--\eqref{eq:fd_Sigma_mix_method}. Then
\[
\frac{\partial \mu(\alpha)}{\partial \alpha_i}=\widehat\mu_i(t),\qquad
\frac{\partial \Sigma(\alpha)}{\partial \alpha_i}=\widehat S_i(t)-\widehat\mu_i(t)\mu(\alpha)^\top-\mu(\alpha)\widehat\mu_i(t)^\top.
\]
We write $\widehat{\mathcal{L}}^{\rm FD}_t(\alpha)=\mathcal{F}(\mu(\alpha),\Sigma(\alpha))$, where
\[
\mathcal{F}(\mu,\Sigma)=\|\mu-\mu_0\|_2^2+\mathrm{Tr}(\Sigma+\Sigma_0)-2\,\mathrm{Tr}\big((\Sigma_0^{1/2}\Sigma\Sigma_0^{1/2})^{1/2}\big),
\]
with $(\mu_0,\Sigma_0)=(\widehat\mu_{\rm data},\widehat\Sigma_{\rm data})$.
Then by the chain rule,
\begin{align}
\frac{\partial \widehat{\mathcal{L}}^{\rm FD}_t(\alpha)}{\partial \alpha_i}
&= 2(\mu(\alpha)-\mu_0)^\top \widehat\mu_i(t)
+ \Big\langle G(\Sigma(\alpha)),\ \widehat S_i(t)-\widehat\mu_i(t)\mu(\alpha)^\top-\mu(\alpha)\widehat\mu_i(t)^\top\Big\rangle,
\label{eq:fd_grad_chain_app}
\end{align}
where $\langle A,B\rangle:=\mathrm{Tr}(A^\top B)$ and
\begin{equation}
G(\Sigma):=\nabla_\Sigma \mathcal{F}(\mu,\Sigma)
= I - \Sigma_0^{1/2}\,(\Sigma_0^{1/2}\Sigma\Sigma_0^{1/2})^{-1/2}\,\Sigma_0^{1/2}.
\label{eq:fd_grad_sigma_app}
\end{equation}
Note that we can compute the square root $(\cdot)^{-1/2}$ via eigendecomposition.

\paragraph{Negative log-Vendi gradient (kernel matrix form).}
At time $t$, let pooled samples $\{x_j\}_{j=1}^N$ with indices $I_j$, Gram matrix $K$, and
$q_j(\alpha)=\alpha_{I_j}/n_{I_j}(t)$ as in \eqref{eq:q_def_method}.
Define $\rho(\alpha)=\mathrm{diag}(q(\alpha))^{1/2}K\mathrm{diag}(q(\alpha))^{1/2}$.
For $\rho\succ 0$, the differential identity
\[
d\,\mathrm{Tr}(\rho\log\rho)=\langle \log\rho+I,\ d\rho\rangle
\]
implies $\nabla_\rho \widehat{\mathcal{L}}^{\rm Vendi}_t(\alpha)=\log\rho(\alpha)+I$.
Moreover,
\[
\frac{\partial q_j(\alpha)}{\partial \alpha_i}=\frac{1}{n_i(t)}\,\mathbf{1}\{I_j=i\}.
\]
Let $D:=\mathrm{diag}(q)^{1/2}$. For $q_j>0$,
\[
\frac{\partial D}{\partial q_j}=\frac{1}{2\sqrt{q_j}}\,e_je_j^\top,
\qquad
\frac{\partial \rho}{\partial q_j}=\frac{\partial D}{\partial q_j}KD + DK\frac{\partial D}{\partial q_j}.
\]
Therefore,
\begin{equation}
\frac{\partial \widehat{\mathcal{L}}^{\rm Vendi}_t(\alpha)}{\partial \alpha_i}
=
\sum_{j:I_j=i}\frac{1}{n_i(t)}\Big\langle \log\rho(\alpha)+I,\ \frac{\partial \rho}{\partial q_j}\Big\rangle.
\label{eq:vendi_grad_kernel_app}
\end{equation}
We compute $\log\rho$ via eigendecomposition $\rho=U\Lambda U^\top$ and $\log\rho=U(\log\Lambda)U^\top$.

\paragraph{Negative log-Vendi gradient (finite-dimensional feature / random-feature form).}
If we use a finite-dimensional feature map $\varphi(x)\in\mathbb{R}^D$ (e.g.\ random Fourier features) and define
\[
C(\alpha):=\sum_{j=1}^N q_j(\alpha)\,\varphi(x_j)\varphi(x_j)^\top\in\mathbb{R}^{D\times D},
\]
then the objective can be equivalently computed as $\mathrm{Tr}(C(\alpha)\log C(\alpha))$ (finite-rank case),
and $\nabla_C \mathrm{Tr}(C\log C)=\log C + I$.
Since $\partial C/\partial q_j=\varphi(x_j)\varphi(x_j)^\top$,
\begin{align}
\frac{\partial \widehat{\mathcal{L}}^{\rm Vendi}_t(\alpha)}{\partial \alpha_i}
&=
\sum_{j:I_j=i}\frac{1}{n_i(t)}\Big\langle \log C(\alpha)+I,\ \varphi(x_j)\varphi(x_j)^\top\Big\rangle \nonumber\\
&=
\sum_{j:I_j=i}\frac{1}{n_i(t)}\,\varphi(x_j)^\top(\log C(\alpha)+I)\varphi(x_j).
\label{eq:vendi_grad_rff_app}
\end{align}
This replaces the $O(N^3)$ cost of eigendecomposing $\rho$ by an $O(D^3)$ eigendecomposition of $C$ when $D< N$.

\section{Proofs for Section~\ref{sec:mg_regret}: log-Vendi and Fr\'echet Distance regret}
\label{app:mg_regret}

\subsection{Bandit protocol and ERM-to-regret reduction}
\label{app:mg_common}

Let $\Delta_m=\{\alpha\in\mathbb{R}^m:\alpha_i\ge 0,\ \sum_{i=1}^m\alpha_i=1\}$.
For the warm start, we suppose $n_i(0)=M$ for all $i$.
At round $t=1,\dots,T$, Mixture-Greedy selects
\[
\alpha_t\in\arg\min_{\alpha\in\Delta_m}\widehat F_{t-1}(\alpha),
\qquad I_t\sim \alpha_t.
\]
The regret at iteration $T$ is defined as $\mathrm{Reg}_T=\sum_{t=1}^T(F(\alpha_t)-\min_{\alpha}F(\alpha))$.

\begin{lemma}\label{lem:erm_compare_app}
For every $t\ge 1$, let $\alpha_t\in\arg\min_{\alpha\in\Delta_m}\widehat F_{t-1}(\alpha)$.
Then, the following holds
\[
F(\alpha_t)-\min_{\alpha\in\Delta_m}F(\alpha)
\le
2\sup_{\alpha\in\Delta_m}\big|\widehat F_{t-1}(\alpha)-F(\alpha)\big|.
\]
\end{lemma}
\begin{proof}
Let $\alpha^\star\in\arg\min_{\alpha}F(\alpha)$. Since $\alpha_t$ minimizes $\widehat F_{t-1}$,
\begin{align*}
F(\alpha_t)-F(\alpha^\star)
&=
\big(F(\alpha_t)-\widehat F_{t-1}(\alpha_t)\big)
+\big(\widehat F_{t-1}(\alpha_t)-\widehat F_{t-1}(\alpha^\star)\big)
+\big(\widehat F_{t-1}(\alpha^\star)-F(\alpha^\star)\big) \\
&\le
|F(\alpha_t)-\widehat F_{t-1}(\alpha_t)|
+|\widehat F_{t-1}(\alpha^\star)-F(\alpha^\star)|
\le
2\sup_{\alpha}|\widehat F_{t-1}(\alpha)-F(\alpha)|.
\end{align*}
The lemma's proof is hence complete.
\end{proof}

\begin{lemma}\label{lem:counts_floor_app}
Assume there exists $\gamma\in(0,1]$ such that $\alpha_{t,i}\ge \gamma$ for all $i$ and all $t\le T$.
Then, for every $\delta>0$, with probability at least $1-\delta$, the following holds simultaneously for all $i\in[m]$ and all $t\le T$,
\[
n_i(t)\ge M + \gamma t - \sqrt{2t\log\frac{mT}{\delta}}.
\]
\end{lemma}
\begin{proof}
Consider integer $i$ and let $N_i(t)=\sum_{s=1}^t\mathbf{1}\{I_s=i\}$, hence $n_i(t)=M+N_i(t)$.
Define $D_s:=\mathbf{1}\{I_s=i\}-\mathbb{E}[\mathbf{1}\{I_s=i\}\mid\mathcal{F}_{s-1}]$, therefore $|D_s|\le 1$ and
$\sum_{s=1}^t D_s=N_i(t)-\sum_{s=1}^t\alpha_{s,i}$.
Azuma--Hoeffding and union bound over $i,t$ yields
$N_i(t)\ge \sum_{s=1}^t\alpha_{s,i}-\sqrt{2t\log\frac{mT}{\delta}}\ge \gamma t-\sqrt{2t\log\frac{mT}{\delta}}$.
\end{proof}


\subsection{Negative log-Vendi: concentration and entropy continuity}
\label{app:mg_nlv_conc_cont}

\begin{assumption}[Normalized features]\label{ass:nlv_norm_app}
$\|\phi(x)\|_2=1$ for all $x$.
\end{assumption}

For arm $i$, define $S_i=\mathbb{E}[\phi(X)\phi(X)^\top]$ and $S_\alpha=\sum_i\alpha_i S_i$.
Define $F_{\mathrm{NLV}}(\alpha)=\mathrm{Tr}(S_\alpha\log S_\alpha)$ and its plug-in estimator
$\widehat F^{\mathrm{NLV}}_t(\alpha)=\mathrm{Tr}(\widehat S_\alpha(t)\log \widehat S_\alpha(t))$.

\begin{assumption}[Population innovation]\label{ass:nlv_innov_app}
There exist $\nu_0\in(0,1]$ and $\varepsilon_0\in[0,\nu_0/8)$ such that for each $i\in[m]$
there exists a unit vector $v_i$ with
$v_i^\top S_i v_i \ge \nu_0$ and $\sum_{j\neq i} v_i^\top S_j v_i \le \varepsilon_0$.
\end{assumption}

\begin{lemma}[from \citep{sutherland2018efficient}]\label{lem:hilbert_hoeffding_app}
Let $(Y_r)_{r=1}^n$ be i.i.d.\ random elements in a real Hilbert space with $\mathbb{E}[Y_r]=0$ and $\|Y_r\|\le L$ a.s.
Then, for every $\delta>0$, with probability at least $1-\delta$, we have the following
\[
\left\|\frac1n\sum_{r=1}^n Y_r\right\|
\le
\frac{L}{\sqrt{n}}\left(1+\sqrt{2\log\frac1\delta}\right).
\]
\end{lemma}

\begin{lemma}\label{lem:nlv_conc_app}
Fix $T\ge 1$ and $\delta\in(0,1)$.
Under Assumption~\ref{ass:nlv_norm_app}, with probability at least $1-\delta$, simultaneously for all $i\in[m]$ and all
$n\in\{M,\dots,M+T\}$,
\[
\|\widehat S_i(n)-S_i\|_F
\le
\frac{2}{\sqrt{n}}\left(1+\sqrt{2\log\frac{m(T+1)}{\delta}}\right).
\]
\end{lemma}
\begin{proof}
Fix $(i,n)$ and set $Y_{i,r}=\phi(X_{i,r})\phi(X_{i,r})^\top-S_i$. Then $\mathbb{E}Y_{i,r}=0$ and
$\|Y_{i,r}\|_F\le \|\phi\phi^\top\|_F+\|S_i\|_F\le 1+1=2$.
The application of Lemma~\ref{lem:hilbert_hoeffding_app} in Frobenius space with failure probability $\delta/(m(T+1))$ and applying the union bound over $(i,n)$'s proves the result.
\end{proof}

\paragraph{Fannes--Audenaert inequality.}
Let $\|\cdot\|_1$ be the trace norm. We use Fannes--Audenaert inequality \citep{Audenaert2007}, which shows for every two density matrices $S,S'\succeq 0$ with $\mathrm{Tr}(S)=\mathrm{Tr}(S')=1$ and $T=\frac12\|S-S'\|_1\in[0,1-1/d]$, the following holds
\[
|\mathrm{Tr}(S\log S)-\mathrm{Tr}(S'\log S')|\le T\log(d-1)+h(T),
\quad h(T)=-T\log T-(1-T)\log(1-T).
\]

\begin{lemma}\label{lem:nlv_modulus_app}
Let $S,S'\succeq 0$ with $\mathrm{Tr}(S)=\mathrm{Tr}(S')=1$ in dimension $d\ge 2$, and assume $\|S-S'\|_F\le \varepsilon$.
Let $T:=\frac12\|S-S'\|_1$. Then $T\le \frac12\sqrt{d}\,\varepsilon$ and whenever $T\le 1-1/d$,
\[
|\mathrm{Tr}(S\log S)-\mathrm{Tr}(S'\log S')|
\le
T\log(d-1)+h(T).
\]
Moreover, if $\varepsilon\le \frac{2}{e\sqrt{d}}$ (thus $T\le 1/e$), then
\begin{equation}\label{eq:nlv_modulus_simplified}
|\mathrm{Tr}(S\log S)-\mathrm{Tr}(S'\log S')|
\le
\frac12\sqrt{d}\,\varepsilon\Big(\log(d-1)+\log\frac{e\sqrt{d}}{2} + 1 + \log\frac1\varepsilon\Big).
\end{equation}
In particular, in this regime,
\[
|\mathrm{Tr}(S\log S)-\mathrm{Tr}(S'\log S')|
\le
C_d\,\varepsilon\log\frac1\varepsilon
\quad\text{with}\quad
C_d := \frac12\sqrt{d}\Big(2+\log(d-1)+\log\frac{e\sqrt{d}}{2}\Big).
\]
\end{lemma}
\begin{proof}
Fro Frobenius ($\Vert\cdot\Vert_F$) and trace norms ($\Vert\cdot\Vert_1$) of a matrix $A\in\mathbb{R}^{d\times d}$, a well-known inequality is that $\|A\|_1\le \sqrt{d}\|A\|_F$. This implies that $T\le \frac12\sqrt{d}\,\varepsilon$. We then apply Fannes--Audenaert to directly obtain the first bound.
For the simplified second bound, note that if $T\le 1/e$, then $-x\log x$ is increasing on $[0,T]$, and one can see
$h(T)\le T\log(1/T)+T$. Since $T\le \frac12\sqrt{d}\,\varepsilon$, we will have
$\log(1/T)\le \log\frac{2}{\sqrt{d}} + \log\frac{1}{\varepsilon}$, and a    substitution yields \eqref{eq:nlv_modulus_simplified}.
\end{proof}


\subsection{Quantitative interiority for log-Vendi Mixture-Greedy (patched)}
\label{app:mg_nlv_interior}

\begin{lemma}\label{lem:nlv_floor_app}
Fix $T\ge 1$ and $\delta\in(0,1)$.
Assume Assumptions~\ref{ass:nlv_norm_app} and~\ref{ass:nlv_innov_app}, and assume the event of Lemma~\ref{lem:nlv_conc_app} holds.
Define
\[
\eta := \frac{2}{\sqrt{M}}\left(1+\sqrt{2\log\frac{m(T+1)}{\delta}}\right),
\qquad
\nu_{\mathrm{eff}}:=\nu_0-\eta,
\qquad
\varepsilon_{\mathrm{eff}}:=\varepsilon_0+(m-1)\eta,
\]
and assume $\eta\le \nu_0/4$ (thus $\nu_{\mathrm{eff}}\ge 3\nu_0/4$).
Define
\[
\gamma_{\min}
:=
\max\left\{0,\ \exp\left(-1-\frac{1}{\nu_{\mathrm{eff}}}\left(\frac{d}{e}+\log m\right)\right)-\varepsilon_{\mathrm{eff}}\right\}.
\]
Then for every $t\in\{1,\dots,T\}$, every minimizer
$\alpha_t\in\arg\min_{\alpha\in\Delta_m}\widehat F^{\mathrm{NLV}}_{t-1}(\alpha)$ satisfies
\[
\alpha_{t,i}\ge \gamma_{\min}\qquad \forall i\in[m].
\]
\end{lemma}

\begin{proof}
Fix $t$ and let $\alpha$ minimize $\widehat F^{\mathrm{NLV}}_{t-1}$.
Write $\widehat S_i:=\widehat S_i(t-1)$ and $\widehat S_\alpha:=\sum_{i=1}^m\alpha_i\widehat S_i$.

First, we analyze the empirical innovation along each $v_i$. To do this, fix index $i$ and consider the population innovation direction $v_i$ from Assumption~\ref{ass:nlv_innov_app}.
By Lemma~\ref{lem:nlv_conc_app} and $\|A\|_{\mathrm{op}}\le \|A\|_F$,
\[
|v_i^\top(\widehat S_j-S_j)v_i|\le \eta\qquad \forall j.
\]
As a result, we can write
\[
v_i^\top \widehat S_i v_i \ge \nu_0-\eta=\nu_{\mathrm{eff}},
\quad
\sum_{j\neq i} v_i^\top \widehat S_j v_i \le \varepsilon_0+(m-1)\eta=\varepsilon_{\mathrm{eff}}.
\]
Also, $\|\widehat S_j\|_{\mathrm{op}}\le \mathrm{Tr}(\widehat S_j)=1$. Therefore, we have
\begin{equation}\label{eq:vi_upper_Salpha_app}
v_i^\top \widehat S_\alpha v_i
=\alpha_i\, v_i^\top \widehat S_i v_i + \sum_{j\neq i}\alpha_j\, v_i^\top \widehat S_j v_i
\le \alpha_i\cdot 1 + \varepsilon_{\mathrm{eff}}.
\end{equation}

Next, we derive the directional optimality condition. Note that the map $S\mapsto \mathrm{Tr}(S\log S)$ is convex on unit-trace PSD matrices (i.e., density matrices). Therefore,
$\alpha\mapsto \widehat F^{\mathrm{NLV}}_{t-1}(\alpha)$ is convex on $\Delta_m$.
Let $j$ be an index with $\alpha_j\ge 1/m$.
For every $i\in[m]$ and $\tau\in(0,\alpha_j]$, we define $\alpha^{(\tau)}=\alpha+\tau(e_i-e_j)\in\Delta_m$.
Optimality of $\alpha$ implies
\begin{equation}\label{eq:dirderiv_nonneg_app}
\frac{\widehat F^{\mathrm{NLV}}_{t-1}(\alpha^{(\tau)})-\widehat F^{\mathrm{NLV}}_{t-1}(\alpha)}{\tau}\ \ge\ 0.
\end{equation}

Then, we let $\mathcal{G}(S)=\mathrm{Tr}(S\log S)$. By convexity,
\begin{equation}\label{eq:subgrad_upper_fixed_app}
\frac{\widehat F^{\mathrm{NLV}}_{t-1}(\alpha^{(\tau)})-\widehat F^{\mathrm{NLV}}_{t-1}(\alpha)}{\tau}
\le
\Big\langle \log\widehat S_\alpha + I,\ \widehat S_i-\widehat S_j\Big\rangle.
\end{equation}
Since $\mathrm{Tr}(\widehat S_i)=\mathrm{Tr}(\widehat S_j)=1$, the identity terms cancel:
\begin{equation}\label{eq:id_cancel_app}
\Big\langle \log\widehat S_\alpha + I,\ \widehat S_i-\widehat S_j\Big\rangle
=
\mathrm{Tr}(\widehat S_i\log\widehat S_\alpha)-\mathrm{Tr}(\widehat S_j\log\widehat S_\alpha).
\end{equation}

We bound the $\widehat S_i$ term using $v_i$. Since $\widehat S_i\succeq 0$ and $\mathrm{Tr}(\widehat S_i)=1$,
\[
\mathrm{Tr}(\widehat S_i\log\widehat S_\alpha)
\le
(v_i^\top \widehat S_i v_i)\cdot v_i^\top(\log\widehat S_\alpha)v_i.
\]
By Jensen's inequality for the concave $\log$ function (on the spectral measure of $\widehat S_\alpha$), we obtain
$v_i^\top(\log\widehat S_\alpha)v_i \le \log(v_i^\top\widehat S_\alpha v_i)$. Thus, by \eqref{eq:vi_upper_Salpha_app} we attain
\[
\mathrm{Tr}(\widehat S_i\log\widehat S_\alpha)
\le
(v_i^\top \widehat S_i v_i)\log(\alpha_i+\varepsilon_{\mathrm{eff}})
\le
\nu_{\mathrm{eff}}\log(\alpha_i+\varepsilon_{\mathrm{eff}}).
\]
Combining this with \eqref{eq:id_cancel_app} yields the following
\begin{equation}\label{eq:upper_dir_patched_app}
\frac{\widehat F^{\mathrm{NLV}}_{t-1}(\alpha^{(\tau)})-\widehat F^{\mathrm{NLV}}_{t-1}(\alpha)}{\tau}
\le
\nu_{\mathrm{eff}}\log(\alpha_i+\varepsilon_{\mathrm{eff}})-\mathrm{Tr}(\widehat S_j\log\widehat S_\alpha).
\end{equation}

Subsequently, we derive a lower bound for the $\widehat S_j$ term. Since $\alpha_j\ge 1/m$, $\widehat S_\alpha\succeq \alpha_j\widehat S_j\succeq \frac1m\widehat S_j$.
Using an $\epsilon$-regularization argument to apply operator monotonicity of $\log$ on PD matrices and then letting $\epsilon\downarrow 0$, we obtain
\[
\mathrm{Tr}(\widehat S_j\log\widehat S_\alpha)
\ge
\mathrm{Tr}\left(\widehat S_j\log\left(\tfrac1m\widehat S_j\right)\right)
=
\mathrm{Tr}(\widehat S_j\log\widehat S_j)-(\log m)\mathrm{Tr}(\widehat S_j).
\]
Since $\mathrm{Tr}(\widehat S_j)=1$ and $x\log x\ge -1/e$ for $x\ge 0$,
$\mathrm{Tr}(\widehat S_j\log\widehat S_j)\ge -d/e$, and then
\begin{equation}\label{eq:heavy_lower_app}
\mathrm{Tr}(\widehat S_j\log\widehat S_\alpha)\ge -\frac{d}{e}-\log m.
\end{equation}

Combining the above, using \eqref{eq:dirderiv_nonneg_app} and \eqref{eq:upper_dir_patched_app}, we find that
\[
0
\le
\frac{\widehat F(\alpha^{(\tau)})-\widehat F(\alpha)}{\tau}
\le
\nu_{\mathrm{eff}}\log(\alpha_i+\varepsilon_{\mathrm{eff}})-\mathrm{Tr}(\widehat S_j\log\widehat S_\alpha).
\]
Using \eqref{eq:heavy_lower_app} gives
\[
0\le \nu_{\mathrm{eff}}\log(\alpha_i+\varepsilon_{\mathrm{eff}})+\frac{d}{e}+\log m,
\]
which can be written as
\[
\alpha_i+\varepsilon_{\mathrm{eff}}
\ge
\exp\left(-\frac{1}{\nu_{\mathrm{eff}}}\left(\frac{d}{e}+\log m\right)\right),
\]
which yields the stated $\gamma_{\min}$ after subtracting $\varepsilon_{\mathrm{eff}}$ and truncating below at $0$.
\end{proof}

\subsection{Negative log-Vendi regret (dimension-explicit warm-start condition)}
\label{app:mg_nlv_regret}

\begin{theorem}[Negative log-Vendi regret]\label{thm:nlv_regret_app}
Fix $T\ge 1$ and $\delta\in(0,1)$.
Assume Assumptions~\ref{ass:nlv_norm_app} and~\ref{ass:nlv_innov_app}.
Assume $M$ is large enough such that:
(i) $\eta\le \nu_0/4$ and $\gamma_{\min}>0$ in Lemma~\ref{lem:nlv_floor_app}, and
(ii) along the analysis event, the Frobenius deviation $\varepsilon_t:=\sup_{\alpha}\|\widehat S_\alpha(t)-S_\alpha\|_F$
satisfies $\frac12\sqrt{d}\,\varepsilon_t\le 1/e$ for all $t\in\{0,\dots,T-1\}$
(equivalently, it suffices that $\varepsilon_0\le 2/(e\sqrt{d})$ since $\varepsilon_t$ decreases with $t$ under linear sampling).
Then with probability at least $1-\delta$,
\[
\mathrm{Reg}_T
\le
C_{\mathrm{NLV}}\Bigl(1+\sqrt{\log\frac{m(T+1)}{\delta}}\Bigr)\sqrt{T}\,(1+\log T),
\]
where $C_{\mathrm{NLV}}$ depends only on $(d,m,\nu_0,\varepsilon_0)$.
\end{theorem}

\begin{proof}
We consider the intersection of these events:
(A) Lemma~\ref{lem:nlv_conc_app} (failure probability below $ \delta/3$),
(B) Lemma~\ref{lem:counts_floor_app} with $\gamma=\gamma_{\min}$ (failure probability below $\le \delta/3$),
(C) the deterministic interiority conclusion of Lemma~\ref{lem:nlv_floor_app} (holds on event $A$).
On this intersection, $n_{\min}(t)\ge M+\gamma_{\min}t-\sqrt{2t\log\frac{3mT}{\delta}}\ge c t$ for an explicit $c=c(\gamma_{\min})>0$.
Therefore, by Lemma~\ref{lem:nlv_conc_app},
\[
\max_{i}\|\widehat S_i(t)-S_i\|_F
\le
\frac{2}{\sqrt{n_{\min}(t)}}\left(1+\sqrt{2\log\frac{3m(T+1)}{\delta}}\right)
\le
\frac{C_1}{\sqrt{t}}\left(1+\sqrt{\log\frac{m(T+1)}{\delta}}\right).
\]
For every $\alpha\in\Delta_m$, the convexity of norm functions shows that
$\|\widehat S_\alpha(t)-S_\alpha\|_F\le \max_i\|\widehat S_i(t)-S_i\|_F=: \varepsilon_t$.
By assumption (ii), we have $\varepsilon_t\le 2/(e\sqrt{d})$, and therefore Lemma~\ref{lem:nlv_modulus_app} implies
\[
\sup_{\alpha\in\Delta_m}\big|\widehat F^{\mathrm{NLV}}_t(\alpha)-F_{\mathrm{NLV}}(\alpha)\big|
\le
C_d\,\varepsilon_t\log\frac1{\varepsilon_t}
\le
\frac{C_2}{\sqrt{t}}\left(1+\sqrt{\log\frac{m(T+1)}{\delta}}\right)(1+\log t),
\]
where the last step uses $\varepsilon_t\asymp t^{-1/2}$. Hence, we obtain $\log(1/\varepsilon_t)\lesssim 1+\log t$.
Lemma~\ref{lem:erm_compare_app} then yields instantaneous regret bounded by twice the uniform deviation, and summing
$\sum_{t=2}^T t^{-1/2}(1+\log t)\le 2\sqrt{T}(1+\log T)$ gives the stated bound.
The application of union bound over events $A$ and $B$ thereof yields a success probability of at least $1-\delta$.
\end{proof}

\section{Including a bounded linear term: Regret for NLV/FD + linear score}
\label{app:mg_linear}

\subsection{Linear term: definition, estimator, and uniform concentration}
\label{app:linear_defs_conc}

Fix measurable functions $\psi_i:\mathcal{X}\to[0,1]$ for each arm $i\in[m]$ and define
\[
\theta_i:=\mathbb{E}[\psi_i(X)]\in[0,1],\quad X\sim P_{\mathcal{G}_i},
\qquad
G(\alpha):=\sum_{i=1}^m\alpha_i\theta_i.
\]
For a base objective $F(\alpha)$ (negative log-Vendi or Fr\'echet Distance), define
\[
H(\alpha):=F(\alpha)+w\,G(\alpha),\qquad w\ge 0.
\]
At time $t$, define empirical estimates
\[
\widehat\theta_i(t):=\frac{1}{n_i(t)}\sum_{r=1}^{n_i(t)}\psi_i(X_{i,r}),
\qquad
\widehat G_t(\alpha):=\sum_{i=1}^m\alpha_i\widehat\theta_i(t),
\qquad
\widehat H_t(\alpha):=\widehat F_t(\alpha)+w\,\widehat G_t(\alpha).
\]
Mixture-Greedy with the combined objective selects
\begin{equation}\label{eq:mg_update_linear_app}
\alpha_t\in\arg\min_{\alpha\in\Delta_m}\widehat H_{t-1}(\alpha),
\qquad I_t\sim \alpha_t.
\end{equation}
We analyze combined-objective regret
\[
\mathrm{Reg}_T(H):=\sum_{t=1}^T\Big(H(\alpha_t)-\min_{\alpha\in\Delta_m}H(\alpha)\Big).
\]

\begin{lemma}
\label{lem:theta_conc_app}
Fix $T\ge 1$ and $\delta\in(0,1)$.
With probability at least $1-\delta$, simultaneously for all $i\in[m]$ and all $n\in\{M,M+1,\dots,M+T\}$,
\[
|\widehat\theta_i(n)-\theta_i|
\le
\frac{1}{\sqrt{2n}}\sqrt{\log\frac{2m(T+1)}{\delta}}.
\]
\end{lemma}
\begin{proof}
Consider a fixed $(i,n)$. Since $\psi_i(X_{i,r})\in[0,1]$ and samples are i.i.d., Hoeffding's inequality gives
$\mathbb{P}(|\widehat\theta_i(n)-\theta_i|\ge \varepsilon)\le 2e^{-2n\varepsilon^2}$.
If we set $\varepsilon=\frac{1}{\sqrt{2n}}\sqrt{\log\frac{2m(T+1)}{\delta}}$, then the right hand side will be bounded by $ \delta/(m(T+1))$,
and applying union bound over $(i,n)$'s proves the result.
\end{proof}

\begin{lemma}[Uniform deviation of the linear objective]\label{lem:G_uniform_app}
Assuming the event of Lemma~\ref{lem:theta_conc_app} occurs, for every $t\le T$ we have the following where $n_{\min}(t):=\min_i n_i(t)$:
\[
\sup_{\alpha\in\Delta_m}|\widehat G_t(\alpha)-G(\alpha)|
\le
\max_{i\in[m]}|\widehat\theta_i(t)-\theta_i|
\le
\frac{1}{\sqrt{2n_{\min}(t)}}\sqrt{\log\frac{2m(T+1)}{\delta}}
\]
\end{lemma}
\begin{proof}
For every $\alpha\in\Delta_m$, we can write the following
\[
|\widehat G_t(\alpha)-G(\alpha)|
=
\left|\sum_{i=1}^m\alpha_i(\widehat\theta_i(t)-\theta_i)\right|
\le
\sum_{i=1}^m\alpha_i\max_j|\widehat\theta_j(t)-\theta_j|
=
\max_j|\widehat\theta_j(t)-\theta_j|.
\]
Then, we apply Lemma~\ref{lem:theta_conc_app} with $n=n_{\min}(t)$ to prove the result.
\end{proof}

\begin{lemma}\label{lem:ERM_H_app}
For every $t\ge 1$, if $\alpha_t$ minimizes $\widehat H_{t-1}$ over $\Delta_m$, then the following holds
\[
H(\alpha_t)-\min_{\alpha\in\Delta_m}H(\alpha)
\le
2\sup_{\alpha\in\Delta_m}|\widehat H_{t-1}(\alpha)-H(\alpha)|.
\]
Moreover, we will have
\[
\sup_{\alpha\in\Delta_m}|\widehat H_{t-1}(\alpha)-H(\alpha)|
\le
\sup_{\alpha\in\Delta_m}|\widehat F_{t-1}(\alpha)-F(\alpha)|
+
w\sup_{\alpha\in\Delta_m}|\widehat G_{t-1}(\alpha)-G(\alpha)|.
\]
\end{lemma}
\begin{proof}
Let $\alpha^\star\in\arg\min_{\alpha\in\Delta_m}H(\alpha)$.
Since $\alpha_t$ minimizes $\widehat H_{t-1}$, we can write
\begin{align*}
H(\alpha_t)-H(\alpha^\star)
&=
\big(H(\alpha_t)-\widehat H_{t-1}(\alpha_t)\big)
+
\big(\widehat H_{t-1}(\alpha_t)-\widehat H_{t-1}(\alpha^\star)\big)
+
\big(\widehat H_{t-1}(\alpha^\star)-H(\alpha^\star)\big)\\
&\le
2\sup_{\alpha\in\Delta_m}|\widehat H_{t-1}(\alpha)-H(\alpha)|.
\end{align*}
The second inequality follows from the triangle inequality and linearity of the combination as
$H-\widehat H=(F-\widehat F)+w(G-\widehat G)$.
\end{proof}

\subsection{NLV + linear term: corrected interiority floor and regret}
\label{app:nlv_linear}

We use the NLV setup and notation from the base NLV appendix:
normalized features (Assumption~\ref{ass:nlv_norm_app}), population innovation (Assumption~\ref{ass:nlv_innov_app}),
uniform concentration of $\widehat S_i$ (Lemma~\ref{lem:nlv_conc_app}), and the resulting quantities
$\eta,\nu_{\mathrm{eff}},\varepsilon_{\mathrm{eff}}$.

\begin{lemma}[Updated interiority floor for NLV + linear term (tight; no $-1$)]
\label{lem:nlv_floor_linear_app}
Fix $T\ge 1$ and $\delta\in(0,1)$.
Assume Assumptions~\ref{ass:nlv_norm_app} and~\ref{ass:nlv_innov_app}, and assume the event of Lemma~\ref{lem:nlv_conc_app} holds.
Define
\[
\eta := \frac{2}{\sqrt{M}}\left(1+\sqrt{2\log\frac{m(T+1)}{\delta}}\right),
\qquad
\nu_{\mathrm{eff}}:=\nu_0-\eta,
\qquad
\varepsilon_{\mathrm{eff}}:=\varepsilon_0+(m-1)\eta,
\]
and assume $\eta\le \nu_0/4$ to obtain $\nu_{\mathrm{eff}}>0$.
Then, for every $t\in\{1,\dots,T\}$, every minimizer
$\alpha_t\in\arg\min_{\alpha\in\Delta_m}\widehat H_{t-1}(\alpha)$ satisfies
\[
\alpha_{t,i}\ge \gamma_{\min}^{(w)}\qquad\forall i\in[m],
\]
where
\begin{equation}
\gamma_{\min}^{(w)}
:=
\max\left\{0,\ \exp\left(-\frac{1}{\nu_{\mathrm{eff}}}\left(\frac{d}{e}+\log m + w\right)\right)-\varepsilon_{\mathrm{eff}}\right\}.
\label{eq:gamma_min_w_app}
\end{equation}
\end{lemma}

\begin{proof}
Fix $t$ and let $\alpha$ minimize $\widehat H_{t-1}(\cdot)=\widehat F_{t-1}(\cdot)+w\widehat G_{t-1}(\cdot)$.
Let $j$ be an index with $\alpha_j\ge 1/m$ (exists since $\sum_i\alpha_i=1$).
For any $i$ and $\tau\in(0,\alpha_j]$, define the feasible perturbation $\alpha^{(\tau)}=\alpha+\tau(e_i-e_j)$.

First, we analyze directional optimality.
Because $\widehat H_{t-1}$ is convex and $\alpha$ is a minimizer over the convex set $\Delta_m$, we have
\begin{equation}\label{eq:dirderiv_nonneg_linear_app}
\frac{\widehat H_{t-1}(\alpha^{(\tau)})-\widehat H_{t-1}(\alpha)}{\tau}\ge 0.
\end{equation}

Then, we show a convex upper bound for the increment term. By convexity of $\widehat F_{t-1}$ and linearity of $\widehat G_{t-1}$, we have
\begin{align}
\frac{\widehat H_{t-1}(\alpha^{(\tau)})-\widehat H_{t-1}(\alpha)}{\tau}
&=
\frac{\widehat F_{t-1}(\alpha^{(\tau)})-\widehat F_{t-1}(\alpha)}{\tau}
+
w\big(\widehat\theta_i-\widehat\theta_j\big)\nonumber\\
&\le
\Big\langle \log\widehat S_\alpha + I,\ \widehat S_i-\widehat S_j\Big\rangle
+
w\big(\widehat\theta_i-\widehat\theta_j\big),
\label{eq:subgrad_upper_linear_app}
\end{align}
where $\widehat S_\alpha=\sum_k\alpha_k\widehat S_k$.
Since $\mathrm{Tr}(\widehat S_i)=\mathrm{Tr}(\widehat S_j)=1$, the identity components cancel exactly:
\[
\Big\langle \log\widehat S_\alpha + I,\ \widehat S_i-\widehat S_j\Big\rangle
=
\mathrm{Tr}(\widehat S_i\log\widehat S_\alpha)-\mathrm{Tr}(\widehat S_j\log\widehat S_\alpha).
\]

Now, we bound the $\widehat S_i$ term using the innovation direction $v_i$. By Lemma~\ref{lem:nlv_conc_app} and the fact that $\|A\|_{\mathrm{op}}\le \|A\|_F$, we attain
\[
v_i^\top \widehat S_i v_i \ge \nu_0-\eta=\nu_{\mathrm{eff}},
\;\;
\sum_{k\neq i} v_i^\top \widehat S_k v_i \le \varepsilon_0+(m-1)\eta=\varepsilon_{\mathrm{eff}}.
\]
Also, $\|\widehat S_k\|_{\mathrm{op}}\le \mathrm{Tr}(\widehat S_k)=1$.
Hence
\[
v_i^\top \widehat S_\alpha v_i
=
\alpha_i\,v_i^\top \widehat S_i v_i + \sum_{k\neq i}\alpha_k\,v_i^\top \widehat S_k v_i
\le
\alpha_i\cdot 1 + \varepsilon_{\mathrm{eff}}.
\]
Since $\widehat S_i\succeq 0$ and $\mathrm{Tr}(\widehat S_i)=1$, we have
$\mathrm{Tr}(\widehat S_i\log\widehat S_\alpha)\le (v_i^\top \widehat S_i v_i)\, v_i^\top(\log\widehat S_\alpha)v_i$.
By concavity of $\log$ and using Jensen's inequality on the spectral measure of $\widehat S_\alpha$ induced by $v_i$,
$v_i^\top(\log\widehat S_\alpha)v_i \le \log(v_i^\top \widehat S_\alpha v_i)$, hence
\[
\mathrm{Tr}(\widehat S_i\log\widehat S_\alpha)
\le
(v_i^\top \widehat S_i v_i)\,\log(v_i^\top \widehat S_\alpha v_i)
\le
\nu_{\mathrm{eff}}\log(\alpha_i+\varepsilon_{\mathrm{eff}}).
\]

To obtain a lower bound for the $\widehat S_j$ term, note that
since $\alpha_j\ge 1/m$, we have $\widehat S_\alpha\succeq \alpha_j\widehat S_j\succeq \frac{1}{m}\widehat S_j$.
By operator monotonicity of $\log$ on the support,
\[
\mathrm{Tr}(\widehat S_j\log\widehat S_\alpha)
\ge
\mathrm{Tr}\left(\widehat S_j\log\left(\tfrac1m\widehat S_j\right)\right)
=
\mathrm{Tr}(\widehat S_j\log\widehat S_j)-(\log m)\mathrm{Tr}(\widehat S_j).
\]
Let $\lambda_1,\dots,\lambda_d$ be eigenvalues of $\widehat S_j$; then $\sum_k\lambda_k=1$ and $\lambda_k\ge 0$.
Since $x\log x\ge -1/e$ for all $x\ge 0$, we get
$\mathrm{Tr}(\widehat S_j\log\widehat S_j)=\sum_k\lambda_k\log\lambda_k\ge -d/e$.
Also, $\mathrm{Tr}(\widehat S_j)=1$. Therefore,
\[
\mathrm{Tr}(\widehat S_j\log\widehat S_\alpha)\ge -\frac{d}{e}-\log m.
\]
Because $0\le \widehat\theta_i,\widehat\theta_j\le 1$, we have $\widehat\theta_i-\widehat\theta_j\le 1$.
Substituting the bounds from Steps 3--4 into \eqref{eq:subgrad_upper_linear_app} yields
\[
\frac{\widehat H(\alpha^{(\tau)})-\widehat H(\alpha)}{\tau}
\le
\nu_{\mathrm{eff}}\log(\alpha_i+\varepsilon_{\mathrm{eff}})-\left(-\frac{d}{e}-\log m\right)+w.
\]
Combining with \eqref{eq:dirderiv_nonneg_linear_app} gives
\[
0\le \nu_{\mathrm{eff}}\log(\alpha_i+\varepsilon_{\mathrm{eff}})+\frac{d}{e}+\log m+w.
\]
Rearranging yields
\[
\alpha_i+\varepsilon_{\mathrm{eff}}
\ge
\exp\left(-\frac{1}{\nu_{\mathrm{eff}}}\left(\frac{d}{e}+\log m+w\right)\right),
\]
and the claimed floor follows after subtracting $\varepsilon_{\mathrm{eff}}$ and truncating at $0$, giving \eqref{eq:gamma_min_w_app}.
\end{proof}

\begin{theorem}[NLV + linear term: Mixture-Greedy regret]\label{thm:nlv_linear_regret_app}
Fix $T\ge 1$ and $\delta\in(0,1)$.
Assume the NLV conditions (Assumptions~\ref{ass:nlv_norm_app}, \ref{ass:nlv_innov_app}) and $0\le \psi_i\le 1$.
Assume $M$ is sufficiently large to satisfy the following:
(i) $\eta\le \nu_0/4$ and $\gamma_{\min}^{(w)}>0$ in Lemma~\ref{lem:nlv_floor_linear_app}, and
(ii) the dimension-explicit continuity regime holds along the analysis event, namely
$\frac12\sqrt{d}\,\varepsilon_t\le 1/e$ for all $t\in\{0,\dots,T-1\}$, where
$\varepsilon_t:=\sup_{\alpha\in\Delta_m}\|\widehat S_\alpha(t)-S_\alpha\|_F$.
Then with probability at least $1-\delta$,
\[
\mathrm{Reg}_T(H)
\le
C_{\mathrm{NLV},w}\Bigl(1+\sqrt{\log\frac{m(T+1)}{\delta}}\Bigr)\sqrt{T}\,(1+\log T),
\]
where $C_{\mathrm{NLV},w}$ depends only on $(d,m,\nu_0,\varepsilon_0,w)$.
\end{theorem}

\begin{proof}
Work on the intersection of the following events:
(A) Lemma~\ref{lem:nlv_conc_app} (uniform concentration of $\widehat S_i$),
(B) Lemma~\ref{lem:theta_conc_app} (uniform concentration of $\widehat\theta_i$),
(C) Lemma~\ref{lem:counts_floor_app} with $\gamma=\gamma_{\min}^{(w)}$ (linear sampling under the floor),
and allocate failure probabilities so that the intersection holds with probability at least $1-\delta$.

On this intersection, Lemma~\ref{lem:nlv_floor_linear_app} ensures $\alpha_{t,i}\ge \gamma_{\min}^{(w)}$ for all $i,t\le T$.
Thus $n_{\min}(t)=\Omega(t)$ by Lemma~\ref{lem:counts_floor_app}. As in the base NLV proof,
\[
\varepsilon_t:=\sup_{\alpha}\|\widehat S_\alpha(t)-S_\alpha\|_F
\le \max_i\|\widehat S_i(t)-S_i\|_F
\le
\frac{C_1}{\sqrt{t}}\Bigl(1+\sqrt{\log\frac{m(T+1)}{\delta}}\Bigr).
\]
By assumption (ii) we may apply the dimension-explicit continuity modulus (Lemma~\ref{lem:nlv_modulus_app}) to obtain
\[
\sup_{\alpha}\big|\widehat F_t(\alpha)-F(\alpha)\big|
\le
C_2\,\varepsilon_t\log\frac1{\varepsilon_t}
\le
\frac{C_3}{\sqrt{t}}\Bigl(1+\sqrt{\log\frac{m(T+1)}{\delta}}\Bigr)(1+\log t).
\]
For the linear part, Lemma~\ref{lem:G_uniform_app} gives
\[
\sup_{\alpha}|\widehat G_t(\alpha)-G(\alpha)|
\le
\frac{1}{\sqrt{2n_{\min}(t)}}\sqrt{\log\frac{2m(T+1)}{\delta}}
\le
\frac{C_4}{\sqrt{t}}\Bigl(1+\sqrt{\log\frac{m(T+1)}{\delta}}\Bigr).
\]
Therefore,
\[
\sup_{\alpha}|\widehat H_t(\alpha)-H(\alpha)|
\le
\sup_{\alpha}|\widehat F_t(\alpha)-F(\alpha)|
+w\sup_{\alpha}|\widehat G_t(\alpha)-G(\alpha)|
\le
\frac{C_5}{\sqrt{t}}\Bigl(1+\sqrt{\log\frac{m(T+1)}{\delta}}\Bigr)(1+\log t).
\]
Lemma~\ref{lem:ERM_H_app} gives
\[
H(\alpha_{t+1})-\min_{\alpha}H(\alpha)
\le
2\sup_{\alpha}|\widehat H_t(\alpha)-H(\alpha)|,
\]
and summing $\sum_{t=1}^{T-1} t^{-1/2}(1+\log t)\le 2\sqrt{T}(1+\log T)$ completes the proof.
\end{proof}

\subsection{Fr\'echet Distance + linear term: margin transfer and regret}
\label{app:fid_linear}

Let $F_{\mathrm{FD}}$ be the Fr\'echet Distance objective and $\widehat F$ its plug-in estimator from the base FD appendix.
Define $H(\alpha)=F_{\mathrm{FD}}(\alpha)+wG(\alpha)$, where $G(\alpha)\in[0,1]$.

\begin{assumption}[Population interiority margin for $H$]\label{ass:fid_margin_H_app}
There exist $\gamma_0\in(0,1/m]$ and $\Delta_0^{(H)}>0$ such that
\[
\inf_{\alpha\in\Delta_m:\ \min_i\alpha_i\le \gamma_0} H(\alpha)
\ \ge\ \min_{\alpha\in\Delta_m}H(\alpha)+\Delta_0^{(H)}.
\]
\end{assumption}

\begin{lemma}\label{lem:fid_margin_transfer_app}
Assume $F_{\mathrm{FD}}$ satisfies Assumption~\ref{ass:fid_margin_main} with gap $\Delta_0$ and the same $\gamma_0$.
Then $H=F_{\mathrm{FD}}+wG$ satisfies Assumption~\ref{ass:fid_margin_H_app} with gap $\Delta_0^{(H)}=\Delta_0-w$ provided $\Delta_0>w$.
\end{lemma}
\begin{proof}
Since $G(\alpha)\in[0,1]$ for all $\alpha\in\Delta_m$,
\[
\inf_{\min_i\alpha_i\le \gamma_0}H(\alpha)
\ge \inf_{\min_i\alpha_i\le \gamma_0}F_{\mathrm{FD}}(\alpha)
\ge \min_{\alpha}F_{\mathrm{FD}}(\alpha)+\Delta_0
\ge \min_{\alpha}H(\alpha)+(\Delta_0-w),
\]
because $\min_\alpha H(\alpha)\le \min_\alpha F_{\mathrm{FD}}(\alpha)+w$.
\end{proof}

\begin{theorem}[Fr\'echet Distance + linear term: Mixture-Greedy regret]\label{thm:fid_linear_regret_app}
Fix $T\ge 1$ and $\delta\in(0,1)$.
Assume the Fr\'echet Distance conditions from Theorem~\ref{thm:fid_main} (boundedness and PD) and $0\le \psi_i\le 1$.
Assume the population margin Assumption~\ref{ass:fid_margin_H_app}.
Assume $M$ is large enough to provide the following: On an event of probability at least $1-\delta/3$,
\[
\sup_{\alpha\in\Delta_m}|\widehat H_t(\alpha)-H(\alpha)| \le \Delta_0^{(H)}/4
\quad\text{for all } t\in\{0,1,\dots,T-1\}.
\]
Then, with probability at least $1-\delta$, the following regret bound holds
\[
\mathrm{Reg}_T(H)
\le
C_{\mathrm{FD},w}\Bigl(1+\sqrt{\log\frac{m(T+1)}{\delta}}\Bigr)\sqrt{T},
\]
where $C_{\mathrm{FD},w}$ depends only on $(B,\lambda_0,\nu,m,d,\gamma_0,w)$.
\end{theorem}

\begin{proof}
The proof follows the non-linear-term Fr\'echet Distance regret proof with $H$ by replacing $F_{\mathrm{FD}}$.
Under the supposed uniform deviation, we know $\sup_\alpha|\widehat H_t-H|\le \Delta_0^{(H)}/4$ for all $t\le T-1$. Therefore,
the same margin-transfer argument implies that every empirical minimizer satisfies $\min_i\alpha_{t,i}\ge \gamma_0$ for all $t\le T$. Then, Lemma~\ref{lem:counts_floor_app} shows that $n_{\min}(t)=\Omega(t)$.

On this interior region, the Lipschitz analysis for the Fr\'echet term yields
$\sup_\alpha|\widehat F_{t-1}(\alpha)-F_{\mathrm{FD}}(\alpha)|\lesssim t^{-1/2}(1+\sqrt{\log(\cdot)})$,
and Lemma~\ref{lem:G_uniform_app} yields $w\sup_\alpha|\widehat G_{t-1}-G|\lesssim w\,t^{-1/2}(1+\sqrt{\log(\cdot)})$.
Therefore,
\[
\sup_{\alpha}|\widehat H_{t-1}(\alpha)-H(\alpha)|
\le
\frac{C}{\sqrt{t}}\Bigl(1+\sqrt{\log\frac{m(T+1)}{\delta}}\Bigr).
\]
Lemma~\ref{lem:ERM_H_app} converts this into instantaneous regret, and summing
$\sum_{t=1}^T t^{-1/2}\le 2\sqrt{T}$ completes the bound.
A union bound yields an overall probability of at least $1-\delta$.
\end{proof}

\section{Mixture-Greedy regret for the RKE score and RKE + linear fidelity term}
\label{app:rke_mg}

\subsection{Setup: feature covariances and the RKE objective}

Let $\phi:\mathcal{X}\to\mathbb{R}^d$ be a feature map and assume normalized features.

\begin{assumption}[Normalized features]\label{ass:rke_norm}
For all $x\in\mathcal{X}$, $\|\phi(x)\|_2=1$.
\end{assumption}

Each arm $i\in[m]$ produces i.i.d.\ samples $X_{i,r}\sim P_{\mathcal{G}_i}$ when pulled.
Define the (population) feature covariance
\[
S_i:=\mathbb{E}\big[\phi(X_{i})\phi(X_{i})^\top\big]\in\mathbb{R}^{d\times d},
\qquad X_i\sim P_{\mathcal{G}_i}.
\]
Then $S_i\succeq 0$ and $\mathrm{Tr}(S_i)=\mathbb{E}\|\phi(X_i)\|_2^2=1$.
For $\alpha\in\Delta_m$, define the mixture covariance $S_\alpha:=\sum_{i=1}^m\alpha_i S_i$ which by default satisfies the requirements for being  a density matrix, i.e., $\mathrm{Tr}(S_\alpha)=1,\, S_\alpha\succeq 0$.
Consider the following function whose minimization is equivalent to optimizing RKE:
\begin{equation}\label{eq:rke_obj_def}
F_{\mathrm{RKE}}(\alpha):=\mathrm{Tr}(S_\alpha^2).
\end{equation}

\subsection{Empirical estimators}

Given $n_i(t)$ samples observed from arm $i$ by time $t$, define the empirical covariance
\[
\widehat S_i(t):=\frac{1}{n_i(t)}\sum_{r=1}^{n_i(t)}\phi(X_{i,r})\phi(X_{i,r})^\top,
\qquad
\widehat S_\alpha(t):=\sum_{i=1}^m\alpha_i \widehat S_i(t).
\]
Define the plug-in estimator of the RKE-minimization objective:
\begin{equation}\label{eq:rke_hat_def}
\widehat F^{\mathrm{RKE}}_t(\alpha):=\mathrm{Tr}\big(\widehat S_\alpha(t)^2\big).
\end{equation}

Mixture-Greedy for RKE uses the same protocol as in the main paper:
\[
\alpha_t\in\arg\min_{\alpha\in\Delta_m}\widehat F^{\mathrm{RKE}}_{t-1}(\alpha),
\qquad I_t\sim \alpha_t.
\]
The cumulative regret is
\[
\mathrm{Reg}_T(F_{\mathrm{RKE}}):=\sum_{t=1}^T\Big(F_{\mathrm{RKE}}(\alpha_t)-\min_{\alpha\in\Delta_m}F_{\mathrm{RKE}}(\alpha)\Big).
\]

\subsection{Uniform concentration for $\widehat S_i$ and a uniform deviation bound for $\widehat F^{\mathrm{RKE}}$}

We use the following concentration statement, proved elsewhere in the appendix (it matches the uniform
matrix Hoeffding bound used in the NLV section); we restate it here for readability.

\begin{lemma}[from \citep{sutherland2018efficient}]\label{lem:rke_conc_S}
Fix $T\ge 1$ and $\delta\in(0,1)$.
Under Assumption~\ref{ass:rke_norm}, with probability at least $1-\delta$, simultaneously for all $i\in[m]$ and all
$n\in\{M,M+1,\dots,M+T\}$,
\[
\|\widehat S_i(n)-S_i\|_{\mathrm{F}}
\le
\frac{2}{\sqrt{n}}\left(1+\sqrt{2\log\frac{m(T+1)}{\delta}}\right).
\]
\end{lemma}

\begin{lemma}\label{lem:rke_uniform_dev}
Fix $T\ge 1$ and $\delta\in(0,1)$.
On the event of Lemma~\ref{lem:rke_conc_S}, for every $t\in\{0,1,\dots,T\}$,
\[
\sup_{\alpha\in\Delta_m}\Big|\widehat F^{\mathrm{RKE}}_{t}(\alpha)-F_{\mathrm{RKE}}(\alpha)\Big|
\le
2\,\varepsilon_t + \varepsilon_t^2,
\qquad
\text{where }\ \varepsilon_t:=\max_{i\in[m]}\|\widehat S_i(t)-S_i\|_{\mathrm{F}}.
\]
In particular, for $t\ge 1$,
\[
\sup_{\alpha\in\Delta_m}\Big|\widehat F^{\mathrm{RKE}}_{t}(\alpha)-F_{\mathrm{RKE}}(\alpha)\Big|
\le
\frac{C_{\mathrm{rke}}}{\sqrt{n_{\min}(t)}}\left(1+\sqrt{\log\frac{m(T+1)}{\delta}}\right),
\]
where $n_{\min}(t):=\min_i n_i(t)$ and $C_{\mathrm{rke}}$ is a universal numerical constant.
\end{lemma}

\begin{proof}
Fix $t$ and $\alpha\in\Delta_m$. Let $\Delta_\alpha(t):=\widehat S_\alpha(t)-S_\alpha$.
By convexity of the Frobenius norm,
\[
\|\Delta_\alpha(t)\|_F
=
\left\|\sum_{i=1}^m\alpha_i(\widehat S_i(t)-S_i)\right\|_F
\le
\sum_{i=1}^m\alpha_i\|\widehat S_i(t)-S_i\|_F
\le
\varepsilon_t.
\]
Now we can expand this as
\[
\widehat F^{\mathrm{RKE}}_{t}(\alpha)-F_{\mathrm{RKE}}(\alpha)
=
\mathrm{Tr}\big((S_\alpha+\Delta_\alpha)^2 - S_\alpha^2\big)
=
2\,\mathrm{Tr}(S_\alpha\Delta_\alpha)+\mathrm{Tr}(\Delta_\alpha^2).
\]
Using $|\mathrm{Tr}(AB)|\le \|A\|_F\|B\|_F$ and $\mathrm{Tr}(\Delta_\alpha^2)=\|\Delta_\alpha\|_F^2$ gives
\[
\Big|\widehat F^{\mathrm{RKE}}_{t}(\alpha)-F_{\mathrm{RKE}}(\alpha)\Big|
\le
2\|S_\alpha\|_F\|\Delta_\alpha\|_F + \|\Delta_\alpha\|_F^2.
\]
Since $S_\alpha\succeq 0$ and $\mathrm{Tr}(S_\alpha)=1$, we have $\|S_\alpha\|_F\le \mathrm{Tr}(S_\alpha)=1$.
Therefore, we can write
\[
\Big|\widehat F^{\mathrm{RKE}}_{t}(\alpha)-F_{\mathrm{RKE}}(\alpha)\Big|
\le
2\cdot 1\cdot \varepsilon_t + \varepsilon_t^2
=
2\,\varepsilon_t + \varepsilon_t^2.
\]
Taking supremum over $\alpha\in\Delta_m$ yields the first claim. For the second claim, apply Lemma~\ref{lem:rke_conc_S} with $n=n_{\min}(t)$ to obtain
$\varepsilon_t\le \frac{2}{\sqrt{n_{\min}(t)}}\big(1+\sqrt{2\log\frac{m(T+1)}{\delta}}\big)$,
and substitute into $2\varepsilon_t+\varepsilon_t^2$, absorbing the $\varepsilon_t^2$ term into the same rate.
\end{proof}

\paragraph{The need for explicit population margin.} Because $F_{\mathrm{RKE}}(\alpha)=\mathrm{Tr}(S_\alpha^2)$ is smooth on $\Delta_m$ and can admit boundary minimizers,
pure Mixture-Greedy may legitimately assign vanishing weights to some arms unless additional structure is imposed.
We therefore assume a population interiority margin, exactly as in the Fr\'echet Distance analysis.

\begin{assumption}[Population interiority margin for RKE-minimization]\label{ass:rke_margin}
There exist $\gamma_0\in(0,1/m]$ and $\Delta_0>0$ such that
\[
\inf_{\alpha\in\Delta_m:\ \min_i\alpha_i\le \gamma_0} F_{\mathrm{RKE}}(\alpha)
\ \ge\
\min_{\alpha\in\Delta_m}F_{\mathrm{RKE}}(\alpha) + \Delta_0.
\]
\end{assumption}

\begin{lemma}\label{lem:rke_emp_interior}
Fix $T\ge 1$.
Suppose that for all $t\in\{0,1,\dots,T-1\}$,
\[
\sup_{\alpha\in\Delta_m}\Big|\widehat F^{\mathrm{RKE}}_{t}(\alpha)-F_{\mathrm{RKE}}(\alpha)\Big|
\le \frac{\Delta_0}{4}.
\]
Then for every $t\in\{0,1,\dots,T-1\}$, every minimizer
$\alpha_{t+1}\in\arg\min_{\alpha\in\Delta_m}\widehat F^{\mathrm{RKE}}_{t}(\alpha)$
satisfies $\min_i \alpha_{t+1,i}\ge \gamma_0$.
\end{lemma}

\begin{proof}
Fix $t$ and let $\alpha_{t+1}$ minimize $\widehat F^{\mathrm{RKE}}_{t}$.
Let $\alpha^\star$ minimize $F_{\mathrm{RKE}}$.
Then
\[
F_{\mathrm{RKE}}(\alpha_{t+1})
\le
\widehat F^{\mathrm{RKE}}_{t}(\alpha_{t+1})+\frac{\Delta_0}{4}
\le
\widehat F^{\mathrm{RKE}}_{t}(\alpha^\star)+\frac{\Delta_0}{4}
\le
F_{\mathrm{RKE}}(\alpha^\star)+\frac{\Delta_0}{2}.
\]
Thus $F_{\mathrm{RKE}}(\alpha_{t+1})\le \min_\alpha F_{\mathrm{RKE}}(\alpha)+\Delta_0/2$.
By Assumption~\ref{ass:rke_margin}, any $\alpha$ with $\min_i\alpha_i\le \gamma_0$ has
$F_{\mathrm{RKE}}(\alpha)\ge \min_\alpha F_{\mathrm{RKE}}(\alpha)+\Delta_0$,
hence $\alpha_{t+1}$ cannot lie in that boundary layer. Therefore, $\min_i\alpha_{t+1,i}\ge \gamma_0$.
\end{proof}

\subsection{RKE regret bound (pure RKE-minimization)}

We use the same ERM-to-regret comparison lemma as elsewhere in the appendix, but stated here for completeness.

\begin{lemma}\label{lem:erm_rke}
If $\alpha_t\in\arg\min_{\alpha\in\Delta_m}\widehat F^{\mathrm{RKE}}_{t-1}(\alpha)$, then
\[
F_{\mathrm{RKE}}(\alpha_t)-\min_{\alpha\in\Delta_m}F_{\mathrm{RKE}}(\alpha)
\le
2\sup_{\alpha\in\Delta_m}\Big|\widehat F^{\mathrm{RKE}}_{t-1}(\alpha)-F_{\mathrm{RKE}}(\alpha)\Big|.
\]
\end{lemma}
\begin{proof}
This is the standard ERM comparison: let $\alpha^\star\in\arg\min_\alpha F_{\mathrm{RKE}}(\alpha)$ and use
$\widehat F^{\mathrm{RKE}}_{t-1}(\alpha_t)\le \widehat F^{\mathrm{RKE}}_{t-1}(\alpha^\star)$, then add and subtract.
\end{proof}

\begin{theorem}[Mixture-Greedy regret for RKE-minimization]\label{thm:rke_regret}
Fix $T\ge 1$ and $\delta\in(0,1)$.
Assume Assumption~\ref{ass:rke_norm} and the population margin Assumption~\ref{ass:rke_margin}.
Assume the warm start $M$ is adequately large so that the uniform deviation event
\[
\sup_{\alpha\in\Delta_m}\Big|\widehat F^{\mathrm{RKE}}_{t}(\alpha)-F_{\mathrm{RKE}}(\alpha)\Big|\le \frac{\Delta_0}{4}
\quad \text{holds for all } t\in\{0,1,\dots,T-1\}
\]
on an event of probability at least $1-\delta/3$.
Then with probability at least $1-\delta$,
\[
\mathrm{Reg}_T(F_{\mathrm{RKE}})
\le
C_{\mathrm{RKE}}\Bigl(1+\sqrt{\log\frac{m(T+1)}{\delta}}\Bigr)\sqrt{T},
\]
where $C_{\mathrm{RKE}}$ depends only on $(m,\gamma_0)$ (and universal numerical constants).
\end{theorem}

\begin{proof}
On the event that the uniform deviation bound is $\le \Delta_0/4$ for all $t\le T-1$,
Lemma~\ref{lem:rke_emp_interior} implies $\min_i\alpha_{t,i}\ge \gamma_0$ for all $t\le T$.
Applying the standard sampling-count lemma (Lemma~\ref{lem:counts_floor_app}) with $\gamma=\gamma_0$ yields
$n_{\min}(t)=\Omega(t)$ with probability at least $1-\delta/3$ after allocating failure probability and union bounding over arms and times.

On the intersection of this count event and the concentration event of Lemma~\ref{lem:rke_conc_S},
Lemma~\ref{lem:rke_uniform_dev} yields for all $t\ge 2$,
\[
\sup_{\alpha\in\Delta_m}\Big|\widehat F^{\mathrm{RKE}}_{t-1}(\alpha)-F_{\mathrm{RKE}}(\alpha)\Big|
\le
\frac{C}{\sqrt{t-1}}\Bigl(1+\sqrt{\log\frac{m(T+1)}{\delta}}\Bigr),
\]
where we used $n_{\min}(t-1)=\Omega(t-1)$.
Lemma~\ref{lem:erm_rke} converts this into instantaneous regret for rounds $t\ge 2$, and summing
\[
\sum_{t=2}^T (t-1)^{-1/2} \le 2\sqrt{T}
\]
gives the stated $\mathcal{O}(\sqrt{T})$ bound (with the initial $t=1$ warm-start round absorbed into the constant).
A union bound over the deviation and count events yields probability at least $1-\delta$.
\end{proof}

\subsection{RKE + bounded linear term}

We now combine $F_{\mathrm{RKE}}$ with the bounded linear fidelity term exactly as in Appendix~\ref{app:mg_linear}.
We reuse your definitions
\[
G(\alpha)=\sum_{i=1}^m\alpha_i\theta_i,\quad \theta_i=\mathbb{E}[\psi_i(X)]\in[0,1],\quad \psi_i:\mathcal{X}\to[0,1],
\quad H(\alpha)=F_{\mathrm{RKE}}(\alpha)+wG(\alpha),
\]
and the plug-in estimator $\widehat H_t(\alpha)=\widehat F^{\mathrm{RKE}}_t(\alpha)+w\widehat G_t(\alpha)$,
with $\widehat G_t$ and $\widehat\theta_i$ as in Appendix~\ref{app:mg_linear}.

\begin{assumption}[Population interiority margin for $H=F_{\mathrm{RKE}}+wG$]\label{ass:rke_margin_H}
There exist $\gamma_0\in(0,1/m]$ and $\Delta_0^{(H)}>0$ such that
\[
\inf_{\alpha\in\Delta_m:\ \min_i\alpha_i\le \gamma_0} H(\alpha)
\ \ge\
\min_{\alpha\in\Delta_m}H(\alpha) + \Delta_0^{(H)}.
\]
\end{assumption}

\begin{lemma}\label{lem:rke_margin_transfer}
If Assumption~\ref{ass:rke_margin} holds with gap $\Delta_0$ and $\Delta_0>w$, then Assumption~\ref{ass:rke_margin_H} holds
with $\Delta_0^{(H)}=\Delta_0-w$.
\end{lemma}
\begin{proof}
Since $G(\alpha)\in[0,1]$ for all $\alpha\in\Delta_m$,
\[
\inf_{\min_i\alpha_i\le \gamma_0}H(\alpha)
\ge
\inf_{\min_i\alpha_i\le \gamma_0}F_{\mathrm{RKE}}(\alpha)
\ge
\min_\alpha F_{\mathrm{RKE}}(\alpha)+\Delta_0
\ge
\min_\alpha H(\alpha)+(\Delta_0-w),
\]
because $\min_\alpha H(\alpha)\le \min_\alpha F_{\mathrm{RKE}}(\alpha)+w$.
\end{proof}

\begin{theorem}[Mixture-Greedy regret for RKE + linear term]\label{thm:rke_linear_regret}
Fix $T\ge 1$ and $\delta\in(0,1)$.
Assume Assumption~\ref{ass:rke_norm}, $0\le \psi_i\le 1$, and the population margin Assumption~\ref{ass:rke_margin_H}.
Assume the warm start $M$ is sufficiently large  so that on an event of probability at least $1-\delta/3$,
\[
\sup_{\alpha\in\Delta_m}\Big|\widehat H_t(\alpha)-H(\alpha)\Big|
\le
\frac{\Delta_0^{(H)}}{4}
\quad\text{for all } t\in\{0,1,\dots,T-1\}.
\]
Then with probability at least $1-\delta$,
\[
\mathrm{Reg}_T(H)
\le
C_{\mathrm{RKE},w}\Bigl(1+\sqrt{\log\frac{m(T+1)}{\delta}}\Bigr)\sqrt{T},
\]
where $C_{\mathrm{RKE},w}$ depends only on $(m,\gamma_0,w)$ (and universal numerical constants).
\end{theorem}

\begin{proof}
On the event $\sup_\alpha|\widehat H_t-H|\le \Delta_0^{(H)}/4$ for all $t\le T-1$,
the same argument as Lemma~\ref{lem:rke_emp_interior} (with $H$ in place of $F_{\mathrm{RKE}}$) implies
$\min_i\alpha_{t,i}\ge \gamma_0$ for all $t\le T$.
Then Lemma~\ref{lem:counts_floor_app} yields $n_{\min}(t)=\Omega(t)$ with high probability.

By Lemma~\ref{lem:ERM_H_app} (already in Appendix~\ref{app:mg_linear}),
\[
H(\alpha_t)-\min_\alpha H(\alpha)\le 2\sup_\alpha|\widehat H_{t-1}(\alpha)-H(\alpha)|.
\]
Moreover, by triangle inequality and linearity,
\[
\sup_\alpha|\widehat H_{t-1}(\alpha)-H(\alpha)|
\le
\sup_\alpha\Big|\widehat F^{\mathrm{RKE}}_{t-1}(\alpha)-F_{\mathrm{RKE}}(\alpha)\Big|
+
w\sup_\alpha\Big|\widehat G_{t-1}(\alpha)-G(\alpha)\Big|.
\]
The first term is bounded by Lemma~\ref{lem:rke_uniform_dev} and $n_{\min}(t-1)=\Omega(t-1)$.
The second term is bounded by Lemma~\ref{lem:G_uniform_app}.
Therefore, on the intersection event, for all $t\ge 2$,
\[
\sup_\alpha|\widehat H_{t-1}(\alpha)-H(\alpha)|
\le
\frac{C}{\sqrt{t-1}}\Bigl(1+\sqrt{\log\frac{m(T+1)}{\delta}}\Bigr).
\]
Summing the instantaneous regret over $t=2,\dots,T$ yields
\[
\sum_{t=2}^T (t-1)^{-1/2} \le 2\sqrt{T},
\]
which completes the $\mathcal{O}(\sqrt{T})$ bound (with the initial $t=1$ warm-start round absorbed into the constant).
A union bound yields an overall probability of at least $1-\delta$.
\end{proof}

\section{Additional Numerical Results}
\paragraph{Effect of the UCB Coefficient.}
In Figure \ref{fig:FID_delta} For each dataset, we conducted two 
independent experiments corresponding to two distinct optimization objectives as \citep{rezaei2025more} suggests. 
In the first experiment, the objective is to \emph{minimize} the KD  metric. 
In the second experiment, the objective is to \emph{maximize} the RKE metric. As a performance upper bound, the \emph{Mixture Oracle} is introduces as a baseline. 
For a given objective, the oracle provides 
the corresponding optimal mixture weights $\alpha^\star$ in advance, and arms 
are sampled according to this fixed distribution. The optimal mixture 
$\alpha^\star$ is obtained by solving the quadratic program using a large number of samples from each arm.
We analyze the impact of the exploration coefficient $\delta_L$ in the UCB term 
for both objectives. Empirically, decreasing $\delta_L$ leads to faster 
convergence toward the corresponding Mixture Oracle. When $\delta_L = 0$, corresponding to the Mixture Greedy 
strategy, convergence is fastest. 

\subsection{Synthetic Settings}
\label{app:syn}

\paragraph{text-to-image setting.}
In Figure \ref{fig:vendidogs}, we constructed a controlled 
text-to-image generation scenario using Stable Diffusion XL (SDXL) \citep{podell2023sdxl}. We generated samples corresponding to five dog breeds:
\texttt{poodle}, \texttt{bulldog}, \texttt{german shepherd}, 
\texttt{golden retriever}, and \texttt{havanese}. 
Each breed-specific prompt defines one arm in our framework.
The objective in this experiment is to optimize a diversity metric, 
namely the Vendi Score~\citep{friedman2023vendi}. 
Rather than optimizing the Vendi Score directly, we optimize its logarithm, 
which is a convex function of the mixture weights using sing exponentiated gradient descent (EGD). We additionally report results for the \emph{One-Arm Greedy} algorithm, which selects at each round the arm with the highest current Vendi score, as well as an $\epsilon$-greedy baseline with $\epsilon = 0.1$.

In Figure~\ref{fig:birds}, we evaluate three text-to-image generators, Kandinsky~\cite{razzhigaev2023kandinsky}, PixArt-$\alpha$~\cite{chen2023pixartalpha}, and SDXL~\cite{sdxl}, using the prompt: ``Generate a red cartoony bird.'' We then apply the Mixture Greedy algorithm to adaptively select among these models in order to maximize the Vendi Score of the generated samples.

As in our previous experiments, the algorithm rapidly converges toward the optimal mixture, leading to a substantial increase in Vendi Score compared to any individual model. This further supports the benefit of mixture optimization for enhancing generative diversity.

\paragraph{text-to-text setting.}
In Figure~\ref{fig:llm}, we consider three large language models, Qwen~2\citep{qwen3}, Gemma~3 \citep{gemma3}, and Llama~3.2 \citep{llama}, prompted with: ``Write a short sentence about a vibrant city in the U.S.'' We then apply the Mixture Greedy algorithm to adaptively select among the models so as to maximize the Vendi Score of the aggregated outputs.

The resulting mixture achieves a higher Vendi Score than any individual model evaluated in isolation. This demonstrates that adaptively combining multiple LLMs can yield greater diversity, as measured by the Vendi Score, than relying on a single model alone.

\subsection{ablation study}

\subsubsection{Number of Random Fourier Features and Cosine Kernel}
To evaluate the robustness of Vendi Score optimization with respect to kernel approximation, we use the Gaussian kernel with Random Fourier Features (RFF) at $R \in \{256,512,1024\}$, and additionally compare against the cosine kernel. As shown in Figure~\ref{fig:ab_nf}, the results are consistent across all settings. On both datasets, \textit{Mixture Greedy} consistently obtains the highest Vendi Score, outperforming \textit{One Arm Greedy} and \textit{One Arm Epsilon Greedy} for all values of $R$. Increasing $R$ changes the curves only slightly and does not affect the relative ranking of the methods. The same trend also holds for the cosine kernel. These results indicate that our method is robust to both the number of random Fourier features and the choice of kernel. Its advantage is preserved for low- and high-dimensional RFF approximations, as well as when using the cosine kernel instead of the Gaussian kernel.

\subsubsection{Kernel Bandwidth}
We evaluate three kernel bandwidths, $\sigma \in \{20,30,40\}$, to examine the sensitivity of RKE and KID to this hyperparameter. Figure~\ref{fig:ab_sigma} shows that the overall conclusions are consistent across all bandwidths. In all cases, RKE increases with training while KID decreases, indicating stable convergence behavior. Across the three bandwidths, the relative ordering of the methods is unchanged: \textit{Mixture Oracle} performs best, \textit{Mixture Greedy} is consistently close to the oracle, and \textit{Mixture UCB} performs worst. For RKE, all methods improve rapidly in the early stage and then plateau, while for KID they decrease sharply at first and then gradually stabilize. Although the absolute values of RKE and KID vary with $\sigma$, the same qualitative trend is preserved. This ablation suggests that our method is robust to kernel bandwidth choice. In particular, the strong performance of \textit{Mixture Greedy} relative to \textit{Mixture UCB}, and its small gap to \textit{Mixture Oracle}, holds uniformly for $\sigma=20$, $30$, and $40$.

\subsubsection{Feature Extractors}
To test the robustness of our results to the underlying representation, we compute RKE and KID using three different feature extractors: \textit{Inception} \citep{inception}, \textit{CLIP} \citep{clip}, and \textit{DINO-v2} \citep{dinov2}. Figure~\ref{fig:ab_feats} shows that the conclusions are consistent across all three feature spaces. In every case, RKE increases and KID decreases over training, indicating stable convergence. The relative ranking of the methods is also unchanged: \textit{Mixture Oracle} performs best, \textit{Mixture Greedy} remains very close to it, and \textit{Mixture UCB} performs worst. Although the absolute metric values vary across feature extractors, the same qualitative behavior is preserved. This demonstrates that our method is robust to the choice of feature representation, and that the strong performance of \textit{Mixture Greedy} is not tied to a particular feature extractor.

\subsubsection{Effect of warm-start sample size}
We further study the sensitivity of Mixture-Greedy to the initial number of samples per arm, denoted by $M$. Table~\ref{tab:warm_start_ablation} reports results for $M \in \{2,10,20\}$ across two representative settings: FD minimization on FFHQ model generators and Vendi maximization on the generated Dogs dataset. Overall, the method is robust to the choice of $M$, with all settings showing consistent improvement over rounds.

\section{Verification of Theoretical Assumptions}
\label{app:assumption_verification}

Our regret analysis relies on structural assumptions that characterize regimes in which Mixture-Greedy can obtain sufficient exploration from the geometry of the diversity-aware objective itself. In this section, we empirically verify the main assumptions across the datasets and generator collections used in our experiments. For each setting, we repeated the verification procedure over $50$ random subsamples and report the results as mean $\pm$ standard deviation when applicable.

\subsection{Population Innovation Condition}
\label{app:population_innovation}

Assumption~\ref{ass:vendi_innov_main} requires that each generator contributes a feature-space direction that is not fully explained by the remaining generators. This condition is used in our analysis of the log-Vendi objective, where such arm-specific innovation prevents the learned mixture from collapsing to a degenerate solution. We verify this condition using normalized DINOv2 embeddings for image data and SBERT embeddings for text data. For each generator $i$, we compute the innovation gap
\[
\delta_i
=
v_i^\top S_i v_i
-
\sum_{j \neq i} v_i^\top S_j v_i,
\]
where $S_i$ denotes the second-moment matrix of embedded samples from generator $i$, and $v_i$ is the corresponding innovation direction. A positive value of $\delta_i$ indicates that generator $i$ has a distinctive direction in the embedding space relative to the other generators. As shown in Table~\ref{tab:innovation_gap}, the innovation gaps are positive across all evaluated settings.

\begin{table}[t]
\centering
\small
\setlength{\tabcolsep}{10pt}
\renewcommand{\arraystretch}{1.12}
\caption{Empirical verification of the population innovation condition. We report the innovation gap $\delta_i$ for each generator. For image-generation benchmarks, the gaps are computed using normalized DINOv2 embeddings; for text settings, we use SBERT embeddings.}
\label{tab:innovation_gap}
\begin{tabular}{lll|lll}
\toprule
\multicolumn{3}{c|}{\textbf{Synthetic Settings}} &
\multicolumn{3}{c}{\textbf{Image-generation benchmarks}} \\
\midrule
Dataset & Generator & $\delta_i$ &
Dataset & Generator & $\delta_i$ \\
\midrule
\multirow{5}{*}{Dogs}
& Bulldog & $0.35$
& \multirow{5}{*}{FFHQ}
& Efficient-VDVAE & $0.008$ \\
& German Shepherd & $0.29$
& & InsGen & $0.004$ \\
& Poodle & $0.26$
& & StyleGAN-XL & $0.002$ \\
& Havanese & $0.19$
& & LDM & $0.001$ \\
& Golden Retriever & $0.18$
& & StyleNAT & $0.001$ \\
\midrule
\multirow{3}{*}{Red Birds}
& SD3 & $0.21$
& \multirow{4}{*}{ImageNet}
& RQ-Transformer & $0.005$ \\
& Kandinsky & $0.02$
& & StyleGAN-XL & $0.003$ \\
& PixArt & $0.02$
& & LDM & $0.002$ \\
& &
& & DiT-XL-2-guided & $0.001$ \\
\midrule
\multirow{3}{*}{City in the U.S.}
& Gemma & $0.02$
& \multirow{4}{*}{LSUN-Bedroom}
& Projected-GAN & $0.002$ \\
& Llama & $0.01$
& & StyleGAN & $0.002$ \\
& Qwen & $0.01$
& & Unleashing Transformers & $0.001$ \\
& &
& & iDDPM & $0.001$ \\
\bottomrule
\end{tabular}
\end{table}

\subsection{Population Interiority Condition}
\label{app:population_interiority}

Assumptions ~\ref{ass:fid_margin_main}, ~\ref{ass:rke_margin} requires that the population optimum lies in the interior of the simplex, with a positive margin separating the full-mixture optimum from boundary solutions that remove an active generator. Intuitively, this condition captures settings in which multiple generators are genuinely useful for optimizing the objective, so that eliminating any active generator worsens the population objective. We empirically verify this condition by comparing the full-mixture optimum with leave-one-out optima for every active generator. Specifically, for each active generator, we remove it from the mixture optimization problem and compute the objective gap relative to the full mixture. Positive gaps indicate that the full mixture is preferred over the corresponding boundary solution. As shown in Table~\ref{tab:population_interiority}, the full-mixture solutions assign nonzero mass to multiple generators, and removing any active generator leads to a positive objective gap, supporting the population interiority condition for both FD and RKE.

\begin{table*}[t]
\centering
\small
\setlength{\tabcolsep}{5pt}
\renewcommand{\arraystretch}{1.12}
\caption{Empirical verification of the population interiority condition. We report the weight of each active generator in the full mixture and the leave-one-out gap relative to the full-mixture optimum. Left: FD objective. Right: RKE objective.}
\label{tab:population_interiority}
\begin{tabular}{llcc|llcc}
\toprule
\multicolumn{4}{c|}{\textbf{FD objective}} &
\multicolumn{4}{c}{\textbf{RKE objective}} \\
\midrule
Dataset & Generator & Weight & Gap $\widehat{\Delta}_0$ &
Dataset & Generator & Weight & Gap $\widehat{\Delta}_0$ \\
\midrule
\multirow{3}{*}{FFHQ256}
& LDM & $0.40 \pm 0.02$ & $12.54 \pm 0.85$
& \multirow{5}{*}{Dogs}
& G1 & $0.24 \pm 0.01$ & $0.042 \pm 0.002$ \\
& StyleGAN-XL & $0.40 \pm 0.02$ & $14.15 \pm 1.12$
& & G2 & $0.24 \pm 0.01$ & $0.041 \pm 0.002$ \\
& StyleNAT & $0.20 \pm 0.02$ & $8.09 \pm 0.64$
& & G3 & $0.18 \pm 0.01$ & $0.029 \pm 0.002$ \\
& & &
& & G4 & $0.17 \pm 0.01$ & $0.026 \pm 0.001$ \\
& & &
& & G5 & $0.16 \pm 0.01$ & $0.025 \pm 0.001$ \\
\midrule
\multirow{4}{*}{ImageNet256}
& DiT-XL-2-guided & $0.48 \pm 0.02$ & $44.71 \pm 2.45$
& \multirow{3}{*}{Red Birds}
& SD3 & $0.38 \pm 0.02$ & $0.048 \pm 0.003$ \\
& LDM & $0.21 \pm 0.01$ & $19.81 \pm 1.30$
& & Kandinsky & $0.33 \pm 0.01$ & $0.042 \pm 0.002$ \\
& StyleGAN-XL & $0.19 \pm 0.01$ & $8.50 \pm 0.72$
& & PixArt & $0.29 \pm 0.01$ & $0.028 \pm 0.002$ \\
& RQ-Transformer & $0.126 \pm 0.01$ & $4.15 \pm 0.45$
& & & & \\
\midrule
\multirow{3}{*}{LSUN-Bedroom}
& iDDPM & $0.57 \pm 0.03$ & $46.71 \pm 2.80$
& \multirow{3}{*}{City in the U.S.}
& Gemma & $0.59 \pm 0.02$ & $0.058 \pm 0.003$ \\
& StyleGAN & $0.39 \pm 0.02$ & $29.91 \pm 1.75$
& & Llama & $0.26 \pm 0.01$ & $0.021 \pm 0.001$ \\
& Projected-GAN & $0.04 \pm 0.01$ & $0.63 \pm 0.12$
& & Qwen & $0.16 \pm 0.01$ & $0.005 \pm 0.001$ \\
\bottomrule
\end{tabular}
\end{table*}

\section{Implementation Details}
\textbf{Optimization Details.} At each round, Mixture-Greedy requires solving the empirical mixture optimization problem over the simplex. For quadratic objectives (KID and RKE), we solve the optimization exactly using CVXPY. For the FD and log-Vendi objectives, the optimization problem is smooth over the probability simplex but does not admit the same quadratic formulation. In these cases, we optimize the empirical objective using exponentiated-gradient descent (EGD), which naturally preserves the simplex constraints.
To ensure that the reported performance reflects the behavior of Mixture-Greedy rather than the numerical accuracy of the optimization routine, we run EGD for 1000 iterations at every round. This conservative optimization budget was chosen to ensure that the optimization converges numerically before the next sample is collected.

\begin{figure*}
        \centering
        \includegraphics[width=0.94\linewidth]{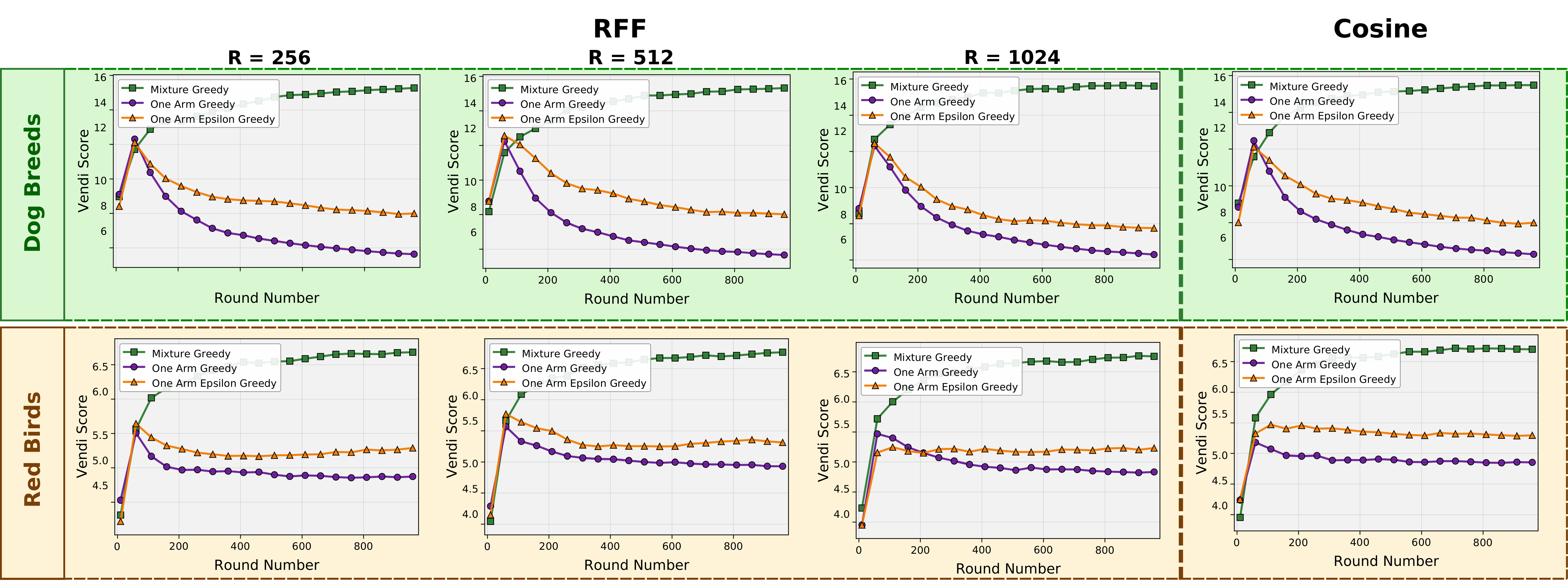}
    \caption{Vendi score over rounds for different using RFF (R=256/512/1024) and cosine kernel on two generated datasets.}
    \label{fig:ab_nf}
\end{figure*}

\begin{figure*}
        \centering
        \includegraphics[width=0.94\linewidth]{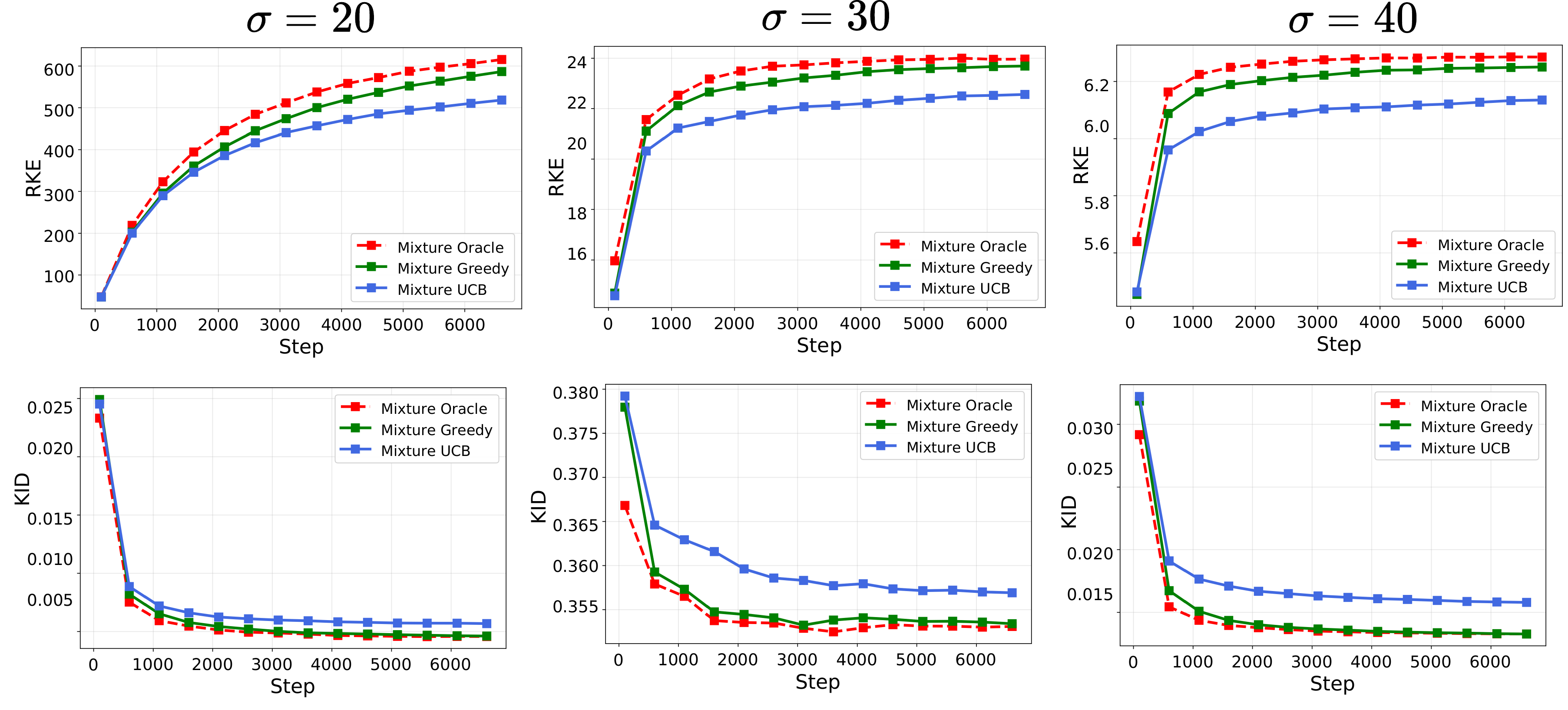}
    \caption{Convergence of KD and RKE for generative models trained on FFHQ dataset using three bandwidths}
    \label{fig:ab_sigma}
\end{figure*}

\begin{figure*}
        \centering
        \includegraphics[width=0.94\linewidth]{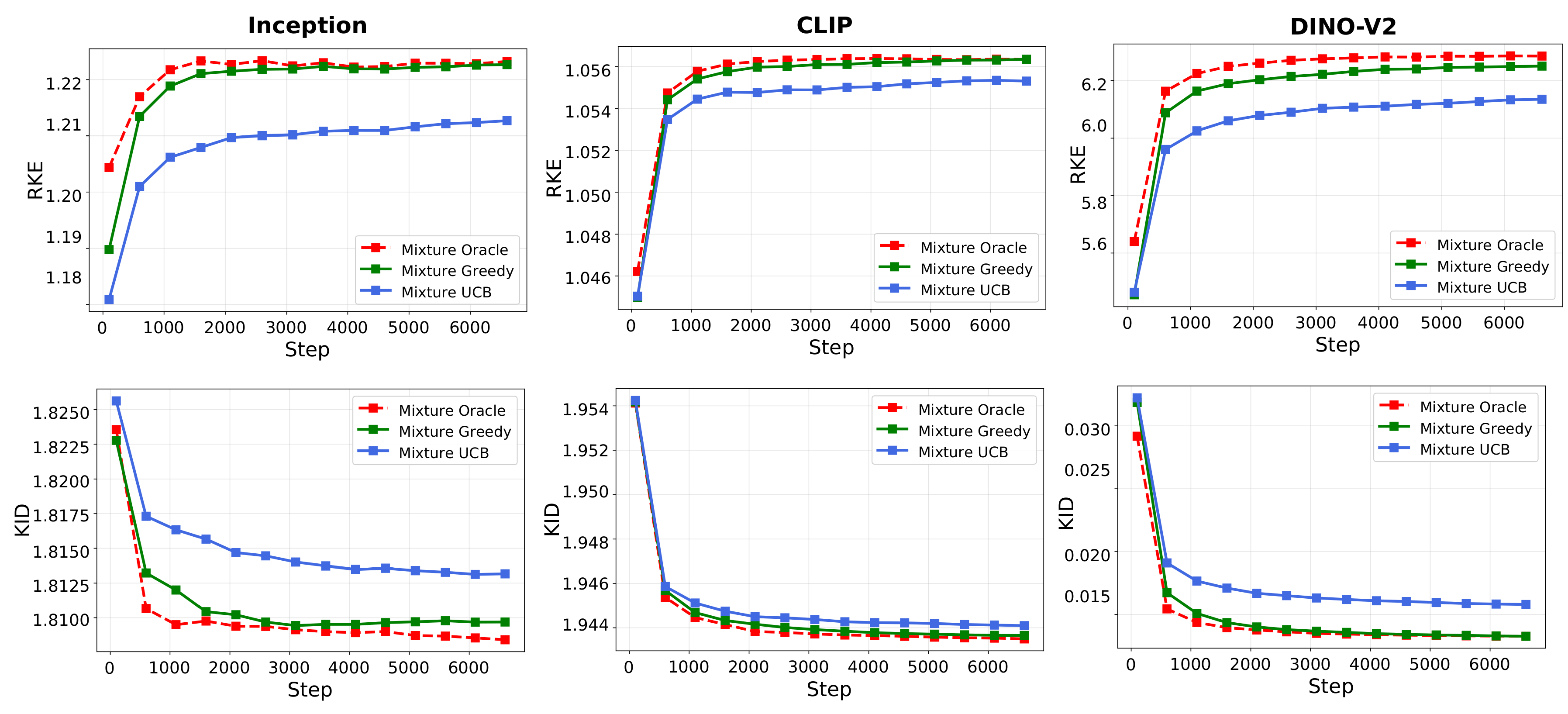}
    \caption{Convergence of RKE and KD on FFHQ generative models, using different feature extractors.}
    \label{fig:ab_feats}
\end{figure*}

\begin{table*}[t]
\centering
\small
\setlength{\tabcolsep}{6pt}
\renewcommand{\arraystretch}{1.15}
\caption{Effect of the initial number of samples per arm, \(M\), on performance across two settings. Panel A reports FD on FFHQ model generators, and Panel B reports Vendi on the generated Dogs dataset. Values are reported as mean \(\pm\) standard deviation.}
\label{tab:warm_start_ablation}
\begin{tabular}{lccccccc}
\toprule
\multicolumn{8}{c}{\textbf{FD on FFHQ model generators}} \\
\midrule
\textbf{\(M\)}
& \textbf{Round 1000}
& \textbf{Round 2000}
& \textbf{Round 3000}
& \textbf{Round 4000}
& \textbf{Round 5000}
& \textbf{Round 6000}
& \textbf{Round 7000} \\
\midrule
2  & \(259.06 \pm 3.1\) & \(195.80 \pm 2.5\) & \(179.75 \pm 2.2\) & \(169.89 \pm 2.0\) & \(165.82 \pm 1.8\) & \(162.50 \pm 1.6\) & \(157.21 \pm 1.4\) \\
10 & \(259.97 \pm 3.0\) & \(196.05 \pm 2.6\) & \(179.50 \pm 2.2\) & \(168.05 \pm 2.0\) & \(162.93 \pm 1.8\) & \(159.99 \pm 1.6\) & \(155.13 \pm 1.4\) \\
20 & \(260.13 \pm 3.0\) & \(196.93 \pm 2.6\) & \(178.97 \pm 2.1\) & \(167.98 \pm 1.9\) & \(162.00 \pm 1.7\) & \(158.45 \pm 1.5\) & \(154.76 \pm 1.3\) \\
\midrule
\multicolumn{8}{c}{\textbf{Vendi on the generated Dogs dataset}} \\
\midrule
\textbf{\(M\)}
& \textbf{Round 100}
& \textbf{Round 200}
& \textbf{Round 300}
& \textbf{Round 400}
& \textbf{Round 500}
& \textbf{Round 600}
& \textbf{Round 700} \\
\midrule
2  & \(12.28 \pm 0.18\) & \(13.18 \pm 0.15\) & \(13.90 \pm 0.13\) & \(14.46 \pm 0.12\) & \(14.70 \pm 0.11\) & \(14.85 \pm 0.10\) & \(14.93 \pm 0.09\) \\
10 & \(12.63 \pm 0.17\) & \(13.19 \pm 0.15\) & \(13.09 \pm 0.14\) & \(13.70 \pm 0.13\) & \(14.00 \pm 0.12\) & \(14.05 \pm 0.11\) & \(14.48 \pm 0.10\) \\
20 & \(12.75 \pm 0.16\) & \(13.22 \pm 0.15\) & \(12.90 \pm 0.14\) & \(13.50 \pm 0.13\) & \(13.85 \pm 0.12\) & \(13.97 \pm 0.11\) & \(14.33 \pm 0.10\) \\
\bottomrule
\end{tabular}
\end{table*}

\begin{figure*}[t]
        \centering
        \includegraphics[width=0.5\linewidth]{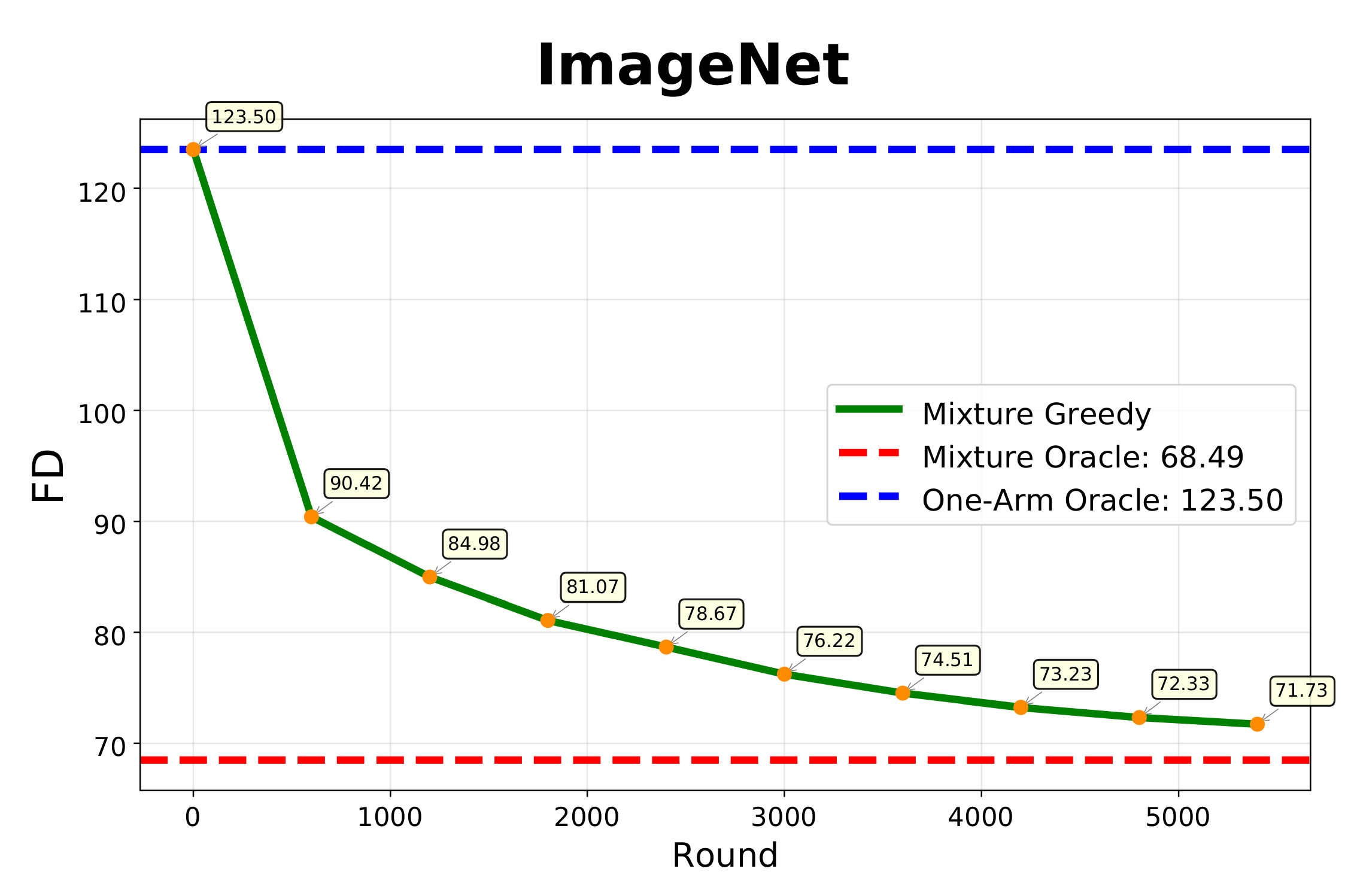}
       
    \caption{Convergence of FD in ImageNet dataset generative models.}
    \label{fig:in_fd}
\end{figure*}

\begin{figure*}[t]
        \centering
        \includegraphics[width=\linewidth]{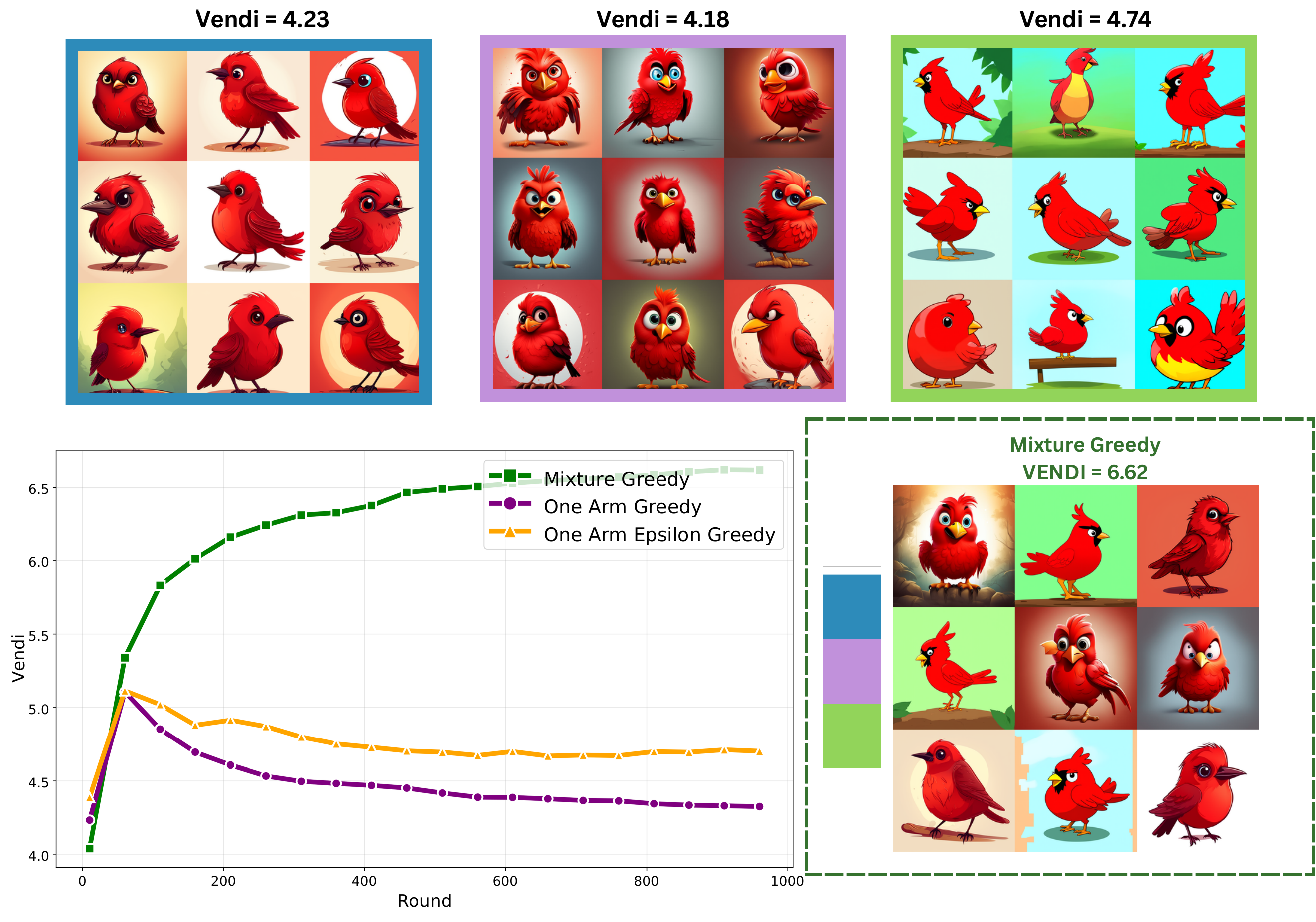}
    \caption{Application of Mixture Greedy on the generated red birds dataset to optimize Vendi Score.}
    \label{fig:birds}
\end{figure*}

\begin{figure*}[t]
        \centering
        \includegraphics[width=\linewidth]{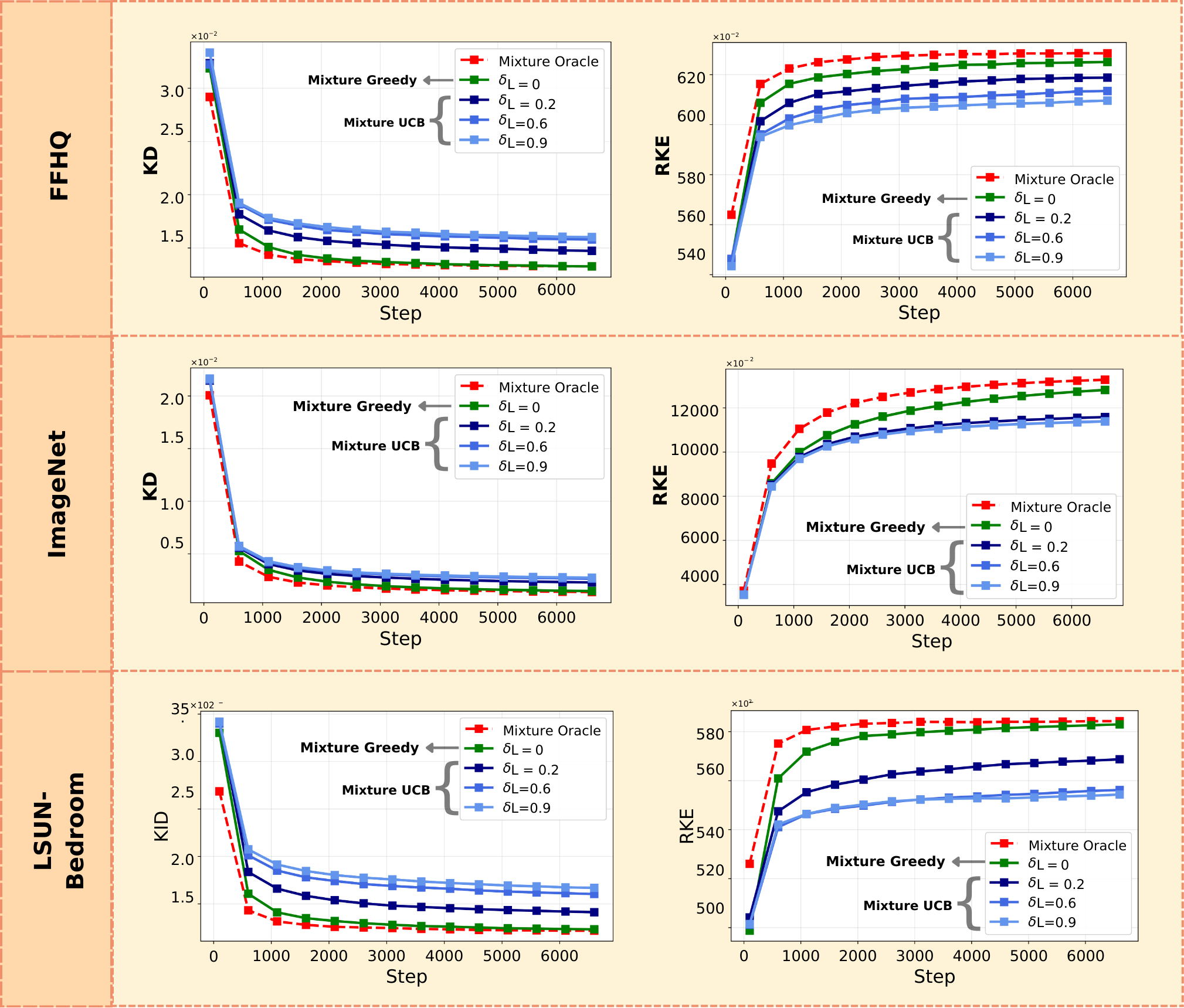}
       
    \caption{Comparison of KD and RKE convergence among different datasets.}
     \label{fig:FID_delta}
\end{figure*}

\begin{figure*}[t]

        \centering
        \includegraphics[width=0.94\linewidth]{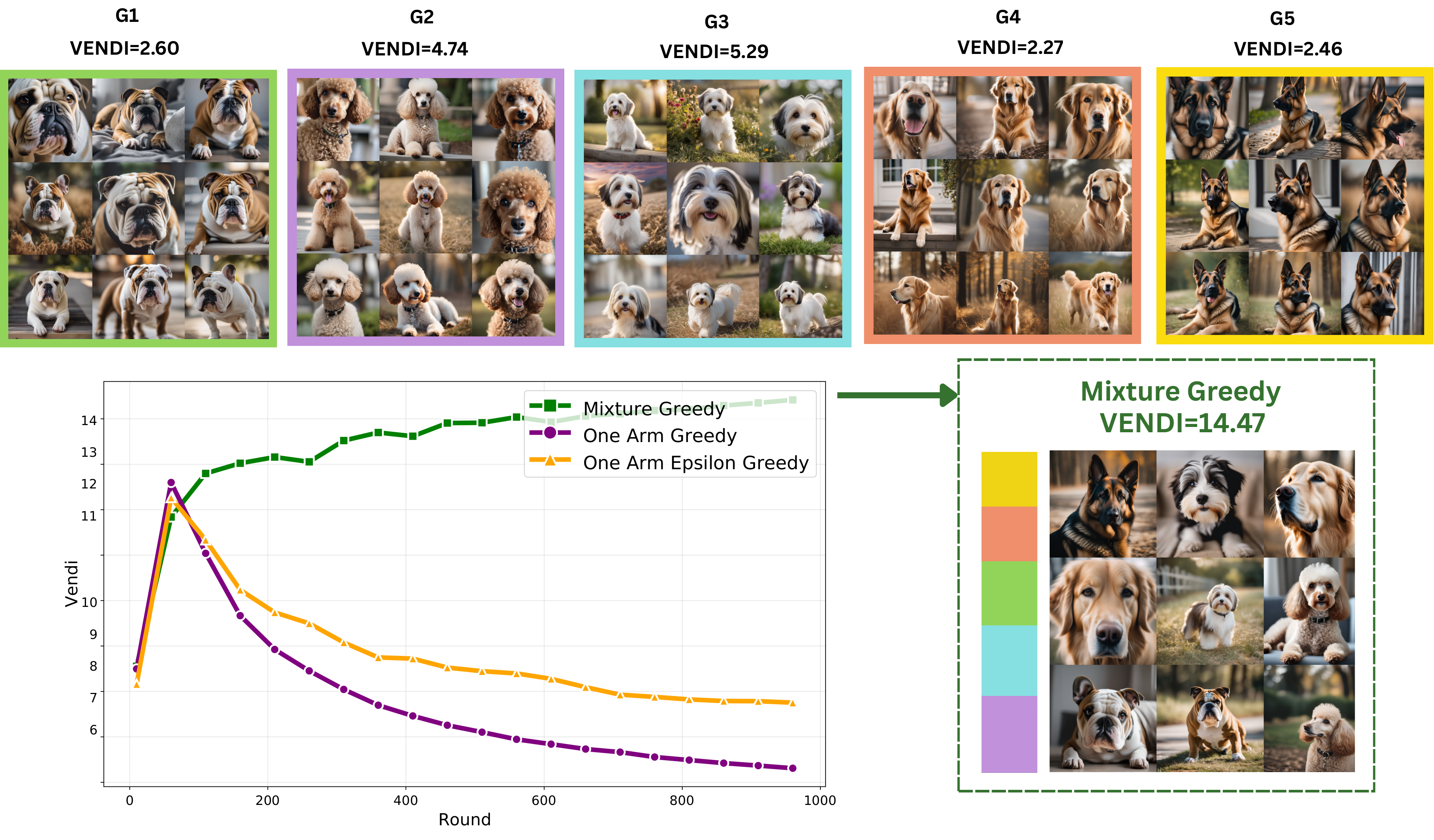}
    \caption{Application of Mixture Greedy on generated dog breeds dataset to optimize Vendi Score.}
     \label{fig:vendidogs}
\end{figure*}

\begin{figure*}[t]
        \centering
        \includegraphics[width=0.94\linewidth]{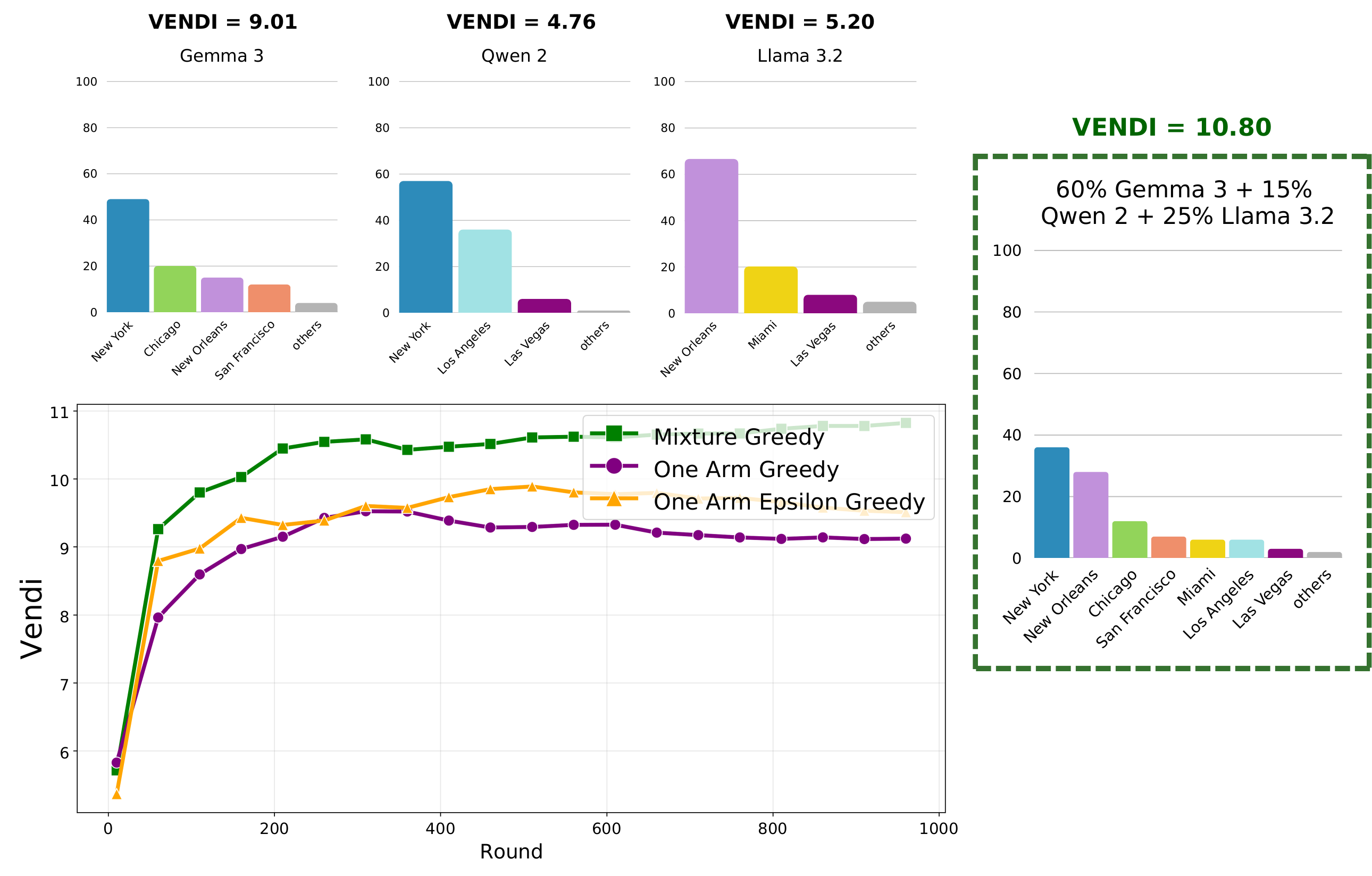}
       
    \caption{Application of Mixture Greedy on LLM generated texts to optimize Vendi Score.}
    \label{fig:llm}
\end{figure*}

\end{document}